\documentclass[letterpaper]{article} 
\usepackage[preprint]{aaai2027}    

\usepackage[hyphens]{url}  
\usepackage{graphicx} 
\usepackage{natbib}  
\usepackage{caption} 

\usepackage[english]{babel}
\usepackage{xcolor}
\usepackage{booktabs}
\usepackage{siunitx}
\usepackage{amsmath,amssymb,amsthm}
\usepackage[page]{appendix}
\newtheorem{assumption}{Assumption}
\newtheorem{corollary}{Corollary} 
\newtheorem{lemma}{Lemma} 
\newtheorem{theorem}{Theorem}
\usepackage{tabularray}
\usepackage{algorithm}
\usepackage{algorithmic}

\newcommand{\Var}{\mathrm{Var}}

\newtheorem{prop}{Proposition}

\newcommand{\Ind}{\mathbb{I}}

\newcolumntype{L}{>{\raggedright\arraybackslash}X}

\usepackage{tikz}
\usetikzlibrary{shapes.geometric, arrows.meta, positioning}
\UseTblrLibrary{booktabs}
\usepackage{subcaption}
\title{Inferential Evaluation of Surrogate-Derived Models under Covariate Shift}

\author{
    Longtian Shi\textsuperscript{\rm 1},
    Molei Liu\textsuperscript{\rm 2, \rm 3}\corresponding,
    Doudou Zhou\textsuperscript{\rm 1}\corresponding
}

\affiliations{
    \textsuperscript{\rm 1}Department of Statistics and Data Science, National University of Singapore\\
    \textsuperscript{\rm 2}Department of Biostatistics, Peking University Health Science Center, Peking University, Beijing, China\\
    \textsuperscript{\rm 3}Beijing International Center for Mathematical Research, Peking University, Beijing, China\\
    longtian.shi\_ilovestat@u.nus.edu, moleiliu@bjmu.edu.cn, ddzhou@nus.edu.sg
}

\begin{document}
\maketitle

\begin{abstract}
In transfer-learning settings, a model derived from abundant surrogate labels may be deployed in a target population where gold-standard outcomes are unobserved. Evaluating its target performance is essential for determining whether decisions based on the model remain reliable, yet it is difficult when gold labels are scarce, and covariate distributions differ across data sources. We study a three-sample setting with a small gold-labeled source, a larger surrogate-labeled source, and an unlabeled target. Under conditional transportability, we evaluate the surrogate-derived model against the latent gold-standard outcome in the target population. We propose cross-fitted estimators that transport information from the two labeled sources through source-specific density ratios. We also combine outcome-regression augmentation with a kernel correction for estimating the model near a threshold,  accounting for uncertainty from all three samples. We establish asymptotically linear inference for TPR and FPR, consistency and pointwise inference for the ROC curve, and asymptotically normal inference for AUC. Simulations assess bias, coverage, and sensitivity to bandwidth and relative sample sizes. A retrospective temporal validation on Chatbot Arena and a semi-synthetic ACS-Income study provide validation in real-world AI applications.
\end{abstract}

\section{Introduction}
\label{Introduction}

Many prediction systems are developed with far more imperfect labels than gold-standard outcomes. Obtaining gold-standard outcomes can require expert adjudication, costly measurements, or long-term follow-up, so such outcomes are often available only in small, selectively assembled samples. A surrogate label, however, may be available at scale from a proxy measurement, a rule-based procedure, weak supervision, or an automated annotator. This larger sample identifies a population prediction score for the surrogate label. Systematic discrepancies between the two labels can carry over to this score and affect its operating characteristics for the gold-standard outcome. We therefore study these operating characteristics in a target population whose covariate distribution may differ from those of both labeled samples. This data structure occurs across a broad range of semi-supervised and surrogate-assisted learning problems \citep{gronsbell2018semi,hou2023surrogate,xia2024prediction}.

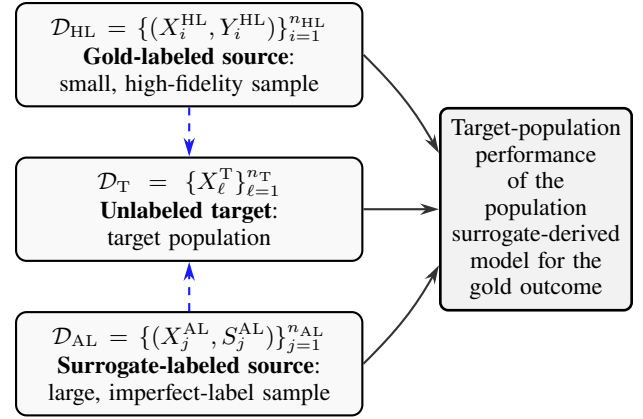
\begin{figure}[!t]
  \centering
  \begin{tikzpicture}[
    every node/.style={
      draw,
      align=center,
      inner xsep=4pt,
      inner ysep=4pt,
    },
    dataset/.style={
  rectangle,
  rounded corners=3pt,
  text width=2.25cm,
  minimum height=1.55cm,
  fill=gray!10,
  line width=1pt,
  font=\footnotesize,
},
subset/.style={
  rectangle,
  rounded corners=4pt,
  text width=4.3cm,
  minimum height=0.95cm,
  fill=gray!05,
  line width=0.8pt,
  font=\footnotesize,
},
    arrow/.style={
      -{Stealth[length=2.5mm,width=1.5mm]},
      thick,
      draw=black!80,
    },
    shift_arrow/.style={
      -{Stealth[length=2.5mm,width=1.5mm]},
      thick,
      draw=blue!90,
      dashed,
    }
  ]

    \node[subset] (HL) {
      $\mathcal{D}_{\mathrm{HL}}
      =\{(X_i^{\mathrm{HL}},Y_i^{\mathrm{HL}})\}_{i=1}^{n_{\mathrm{HL}}}$\\[1pt]
      \textbf{Gold-labeled source}: small, high-fidelity sample
    };

    \node[subset, below=0.65cm of HL] (T) {
      $\mathcal{D}_{\mathrm{T}}
      =\{X_\ell^{\mathrm{T}}\}_{\ell=1}^{n_{\mathrm{T}}}$\\[1pt]
      \textbf{Unlabeled target}: target population
    };

    \node[subset, below=0.65cm of T] (AL) {
      $\mathcal{D}_{\mathrm{AL}}
      =\{(X_j^{\mathrm{AL}},S_j^{\mathrm{AL}})\}_{j=1}^{n_{\mathrm{AL}}}$\\[1pt]
      \textbf{Surrogate-labeled source}: large, imperfect-label sample
    };

    \node[dataset, right=1cm of T] (D) {
      Target-population\\
      performance of the\\
      population surrogate-derived\\
      model for the gold outcome
    };

    \draw[arrow]
      (HL.east) to[bend left=8] (D.150);

    \draw[arrow]
      (T.east) -- (D.west);

    \draw[arrow]
      (AL.east) to[bend right=8] (D.210);

    \draw[shift_arrow]
      (HL.south) -- (T.north);

    \draw[shift_arrow]
      (AL.north) -- (T.south);

  \end{tikzpicture}

  \caption{Three-population setting. Gold and surrogate labels are observed in separate sources; only covariates are observed in the target population. Each source may have its own covariate shift relative to the target (dashed arrows).}
  \label{fig:datastructure}
\end{figure}

Specifically, we evaluate the population surrogate-derived model in the target population with respect to the gold-standard outcome $Y$, using the true-positive rate (TPR), false-positive rate (FPR), receiver operating characteristic (ROC) curve, and area under the ROC curve (AUC).

The way these labels are collected can also induce covariate shift. A gold-labeled sample may come from a specialized study, a different institution or platform, or an earlier period, whereas surrogate labels may be available in a larger operational database. Differences in sampling design, population composition, and data-collection processes can therefore lead to different covariate distributions across the two labeled sources and the target, in which performance is to be evaluated. These differences matter even if the conditional gold-outcome mechanism is stable between the gold-labeled source and the target: TPR, FPR, ROC, and AUC average over the target covariate distribution, so source-population performance need not equal target-population performance. Moreover, because the two labeled samples are assembled through different processes, their shifts relative to the target need not be the same. Each source consequently requires its own transport correction \citep{shimodaira2000improving,sugiyama2012machine,zhou2025doubly}.

These features yield three distinct populations, two labeled sources and an unlabeled target, summarized in Figure~\ref{fig:datastructure}. We index their observations separately as $\mathcal{D}_{\mathrm{HL}}=\{(X_i^{\mathrm{HL}},Y_i^{\mathrm{HL}})\}_{i=1}^{n_{\mathrm{HL}}}$, $\mathcal{D}_{\mathrm{AL}}=\{(X_j^{\mathrm{AL}},S_j^{\mathrm{AL}})\}_{j=1}^{n_{\mathrm{AL}}}$, and $\mathcal{D}_{\mathrm{T}}=\{X_\ell^{\mathrm{T}}\}_{\ell=1}^{n_{\mathrm{T}}}$. The indices $i$, $j$, and $\ell$ are local to their respective samples, and population superscripts distinguish observations in empirical expressions. Three samples contain no shared observational units and are mutually independent; observations are i.i.d. within each sample. The labels $Y$ and $S$ are binary. For $g\in\{\mathrm{HL},\mathrm{AL},\mathrm{T}\}$, let $\mathbb P_g$ and $\mathbb E_g$ denote probability and expectation under population $g$. In population-level expressions, the subscript on $\mathbb P_g$ or $\mathbb E_g$ identifies the population law, so we write the generic variables as $X$, $Y$, and $S$ without population superscripts. Superscripts are retained for sample observations to distinguish the three datasets. We write $\Ind\{A\}$ for the indicator of event $A$. Under $\mathbb P_{\mathrm{T}}$, $Y$ denotes the latent gold-standard outcome; only $X$ is observed in the target sample. Write $O^{\mathrm{HL}}=(X,Y)$, $O^{\mathrm{AL}}=(X,S)$, and $O^{\mathrm{T}}=X$ for generic observed records from three sources. Let $O_r^g$ denote the corresponding $r$th sample record. For any integrable source-specific function $f$, use the operator $\mathbb P_gf=\mathbb E_g\{f(O^g)\}$ and $\mathbb P_{g,n}f=n_g^{-1}\sum_{r=1}^{n_g}f(O_r^g)$ for its population and empirical averages, respectively. 

We assume conditional outcome transportability,
\[
m_Y^\star(x)
:=\mathbb{E}_{\mathrm{HL}}(Y\mid X=x)
=\mathbb{E}_{\mathrm{T}}(Y\mid X=x).
\]
The auxiliary-source law defines the population surrogate-derived model $
m_S^\star(x):=\mathbb{E}_{\mathrm{AL}}(S\mid X=x)$. For a fixed threshold $c$ in the interior of $m_S^\star$'s range, the target TPR is
\begin{align}
\mathrm{TPR}_{\mathrm{T}}(c)
&=\mathbb P_{\mathrm{T}}\{m_S^\star(X)\ge c\mid Y=1\}\notag\\
&=\frac{\mathbb{E}_{\mathrm{T}}\left[m_Y^\star(X)
\Ind\{m_S^\star(X)\ge c\}\right]}
{\mathbb{E}_{\mathrm{T}}\left[m_Y^\star(X)\right]}.
\label{eq:intro_tpr_identification}
\end{align}
The FPR, ROC curve, and AUC are defined with respect to the same target law. The gold-labeled source identifies $m_Y^\star$, the surrogate-labeled source identifies $m_S^\star$, and the target sample supplies the covariate law governing the expectations in~\eqref{eq:intro_tpr_identification}. The identifying formula for TPR also shows why direct plug-in inference can be unreliable with flexible nuisance estimators. Both $m_Y^\star$ and $m_S^\star$ must be estimated from labeled data, and errors in either propagate into the plug-in performance metric. Their convergence rates may be adequate for prediction but too slow for the resulting plug-in bias to be negligible on the scale required for inference. Estimating $m_S^\star(x)$ is especially delicate because it enters the target functional through $\Ind\{m_S^\star(X)\ge c\}$: a small error can move an observation across $c$, and its leading contribution is concentrated where $m_S^\star(X)$ is near $c$. Orthogonal corrections for smooth functionals do not capture this contribution.

Our three-sample setting is related to several lines of work. Semi-supervised methods combine a small gold-labeled sample with unlabeled covariates to estimate prediction performance more efficiently, including under nonuniform label sampling and for ROC analysis \citep{gronsbell2018semi,gronsbell2022efficient,gao2024ssroc}. Proxy- and surrogate-assisted methods instead use inexpensive measurements or machine predictions for model training or statistical inference \citep{hou2023surrogate,xia2024prediction,angelopoulos2023prediction}. They do not address target-population performance when the gold- and surrogate-labeled samples undergo distinct shifts relative to the target.

\noindent\textbf{Technical challenge.}
The inferential difficulty lies in the interaction between multi-source transport and a learned decision boundary: the gold-outcome regression and surrogate score are identified from differently shifted sources, while estimation error in the latter propagates through a hard threshold into TPR, FPR, and ROC. A separate literature develops doubly robust or automatic debiasing methods for statistical functionals under covariate shift \citep{reddi2015doubly,chernozhukov2023automatic,qiu2024risk}. Work specific to model evaluation transports prediction error, classification accuracy, or AUC from a labeled source to a target population \citep{steingrimsson2023transporting,li2023auc,zhou2025doubly,liu2025auc}. These methods treat the prediction rule as fixed, or as learned using gold outcomes, and involve a single labeled-source-to-target shift; here, the population surrogate-derived model is identified from a second shifted source. \citet{ying2025towards} incorporates automated computational phenotypes under covariate shift, with those phenotypes observed in the unlabeled population rather than in a separate source. \citet{park2024debiased} handles estimated nuisance functions inside indicators by smoothing the indicator, whereas our estimator retains the hard-threshold estimand and localizes the correction for estimating $m_S^\star$ near the threshold. Concurrent work considers target-risk evaluation under covariate shift and selective labels and multi-source prediction-powered inference under heterogeneous distributions \citep{ulichney2026beyond,li2026multisource}; neither targets the operating characteristics of a population surrogate-derived model identified from a separate surrogate-label population. The present problem, therefore, requires transporting gold-outcome residuals from the gold-labeled source and residuals associated with estimating $m_S^\star$ from the auxiliary source, with the latter weighted near the evaluation threshold.

\noindent\textbf{Technical and practical novelty.}
Our central contribution is a source-aware debiasing principle that corrects each first-order nuisance effect using the population in which it is identified, enabling valid target-population inference without gold labels in the target. We develop three-sample estimators with two source-specific corrections. First, transported outcome residuals from the gold-labeled source correct the first-order error in target-sample imputation based on $\hat m_Y$. Second, since $\hat m_S$ enters a hard threshold, its estimation error matters primarily for observations near that threshold. We therefore use a kernel to localize surrogate residuals near the threshold and transport this correction from the auxiliary source to the target population. Together with cross-fitting, the two corrections remove the leading nuisance-estimation effects and yield an asymptotically linear expansion with separate contributions from the gold-labeled, auxiliary, and target samples. We establish pointwise inference for TPR and FPR at fixed thresholds and for ROC values at fixed false-positive rates, as well as consistency of the projected ROC curve and asymptotically normal inference for AUC. Simulations verify theoretical coverage and characterize sensitivity to bandwidth and sample sizes. The two applications evaluate the method under a natural temporal shift in Chatbot Arena and a semi-synthetic ACS-Income label-masking design, with Target outcomes withheld during estimation.

\section{Method}
\label{Method}

Suppress the target subscript and write
$\tau_1(c)=\operatorname{TPR}(c)$ and
$\tau_0(c)=\operatorname{FPR}(c)$. Let $\theta_1(c)$ and $p_1\in(0,1)$
denote the numerator and denominator in
\eqref{eq:intro_tpr_identification}, so that
$\tau_1(c)=\theta_1(c)/p_1$. For
$g\in\{\mathrm{HL},\mathrm{AL},\mathrm{T}\}$, let $p_g$ denote the density of
$X$ under population $g$ with respect to a common dominating measure. Under
the source-support and bounded-density-ratio conditions in part (a) of
Assumption~\ref{ass:smooth}, define
$w_{\mathrm{HL}}(x)=p_{\mathrm{T}}(x)/p_{\mathrm{HL}}(x)$ and
$w_{\mathrm{AL}}(x)=p_{\mathrm{T}}(x)/p_{\mathrm{AL}}(x)$.
For a fixed threshold $c$, write
$\Ind_c(x)=\Ind\{m_S^\star(x)\ge c\}$.

\paragraph{Population Identities.}
Motivated by~\eqref{eq:intro_tpr_identification}, a direct plug-in estimator substitutes fitted estimators $\hat m_Y$ and $\hat m_S$ for $m_Y^\star$ and $m_S^\star$, respectively, replacing target expectations with empirical averages over covariates. Below, arguments $X$ are omitted. For generic approximations $m$ and $w$ to $m_Y^\star$ and $w_{\mathrm{HL}}$, and a given function $q$, define
\begin{align*}
\Psi(m,w;q)
&=\mathbb E_{\mathrm{T}}(mq)+\mathbb E_{\mathrm{HL}}\{w(Y-m)q\},\\
\Psi(m,w;q)-\mathbb E_{\mathrm{T}}(m_Y^\star q)
&=\mathbb E_{\mathrm{HL}}\{(w_{\mathrm{HL}}-w)(m-m_Y^\star)q\}.
\end{align*}
The first term averages the imputed quantity $mq$ over the target covariates,
while the second uses transported gold-source residuals. Their sum differs
from $\mathbb E_{\mathrm{T}}(m_Y^\star q)$ only by the product of the
outcome-regression and density-ratio errors. $q\equiv1$ gives the
corresponding identity for $p_1=\mathbb E_{\mathrm{T}}(m_Y^\star)$. For the numerator, condition on the data used to fit $\hat m_S$ and define
$\hat I_c(x)=\Ind\{\hat m_S(x)\ge c\}$. With
$\delta_S=\hat m_S-m_S^\star$ and $q=\hat I_c$, the target functional in
the display is $\mathbb E_{\mathrm{T}}(m_Y^\star\hat I_c)$, whose difference
from $\theta_1(c)$ is $B_S(c)=\mathbb E_{\mathrm{T}}\{m_Y^\star(\hat I_c-\Ind_c)\}$. 

We now construct a correction for $B_S(c)$. The derivative of the map
$t\mapsto\Ind\{t\ge c\}$, in the distributional sense, is the Dirac measure
at $c$. An observation can contribute to $B_S(c)$ only if
$|m_S^\star-c|\le|\delta_S|$; under uniform consistency, such observations
lie in a shrinking neighborhood of the threshold. Replacing the point mass
by $K_h(c-t)$, where
$K_h(v)=h^{-1}K(v/h)$, suggests estimating $B_S(c)$ with $\mathbb E_{\mathrm{T}}\!\left[m_Y^\star(X)K_h\{c-\hat m_S(X)\}\delta_S(X)\right]$. Appendix~\ref{methodintuition} justifies this approximation.
Conditional on the data used to fit $\hat m_S$,
$\mathbb E_{\mathrm{AL}}\{S-\hat m_S(X)\mid X\}=-\delta_S(X)$;
hence,
\begin{align*}
&\mathbb E_{\mathrm{AL}}\!\left[
w_{\mathrm{AL}}(X)m_Y^\star(X)
K_h\{c-\hat m_S(X)\}
\{S-\hat m_S(X)\}\right]\\
&\qquad=-\mathbb E_{\mathrm{T}}\!\left[
m_Y^\star(X)K_h\{c-\hat m_S(X)\}\delta_S(X)
\right].
\end{align*}
The auxiliary term cancels, to first order, the error caused by
replacing $m_S^\star$ with $\hat m_S$ in the threshold indicator.

\paragraph{Cross-fitted Implementation.}
The preceding population identities treat the fitted nuisance functions as fixed. Cross-fitting recreates this separation in the sample by evaluating each observation with nuisance functions trained without that observation. Independently split $\mathcal{D}_{\mathrm{HL}}$, $\mathcal{D}_{\mathrm{AL}}$, and $\mathcal{D}_{\mathrm{T}}$ into $V$ folds $\{\mathcal{I}_k\}_{k=1}^V$, $\{\mathcal{J}_k\}_{k=1}^V$, and $\{\mathcal{L}_k\}_{k=1}^V$. For each nuisance estimator, the superscript $(-k)$ means fold $k$ is removed from every sample entering that particular fit. Empirical averages over the three validation folds are denoted by $\mathbb P_{\mathrm{HL},\mathcal I_k}$, $\mathbb P_{\mathrm{AL},\mathcal J_k}$, and $\mathbb P_{\mathrm{T},\mathcal L_k}$, respectively. Fit $\hat m_Y^{(-k)}$ for $m_Y^\star$ on $\mathcal{D}_{\mathrm{HL}}\setminus\mathcal{I}_k$ and $\hat m_S^{(-k)}$ for $m_S^\star$ on $\mathcal{D}_{\mathrm{AL}}\setminus\mathcal{J}_k$. Fit $\hat w_{\mathrm{HL}}^{(-k)}$ for $w_{\mathrm{HL}}$ by distinguishing $\mathcal{D}_{\mathrm{HL}}\setminus\mathcal{I}_k$ from $\mathcal{D}_{\mathrm{T}}\setminus\mathcal{L}_k$, and fit $\hat w_{\mathrm{AL}}^{(-k)}$ for $w_{\mathrm{AL}}$ analogously using $\mathcal{D}_{\mathrm{AL}}\setminus\mathcal{J}_k$ and $\mathcal{D}_{\mathrm{T}}\setminus\mathcal{L}_k$. Appendix~\ref{densityratioestimation} gives the classifier-to-density-ratio conversion. 

Therefore, the cross-fitted sample analogue of $\Psi(m,w;1)$ is
\begin{align*}
\widehat p_1
&=\frac{1}{n_{\mathrm{T}}}\sum_{k=1}^V\sum_{\ell\in\mathcal{L}_k}
\hat m_Y^{(-k)}(X_\ell^{\mathrm{T}})\\
&\quad+\frac{1}{n_{\mathrm{HL}}}\sum_{k=1}^V\sum_{i\in\mathcal{I}_k}
\hat w_{\mathrm{HL}}^{(-k)}(X_i^{\mathrm{HL}})
\{Y_i^{\mathrm{HL}}-\hat m_Y^{(-k)}(X_i^{\mathrm{HL}})\}.
\end{align*}
The target plug-in term, gold-source correction, and auxiliary correction for the numerator are
\begin{align*}
\widehat\theta_{\mathrm{T}}(c)
&=\frac{1}{n_{\mathrm{T}}}\sum_{k=1}^V\sum_{\ell\in\mathcal{L}_k}
\hat m_Y^{(-k)}(X_\ell^{\mathrm{T}})
\Ind\{\hat m_S^{(-k)}(X_\ell^{\mathrm{T}})\ge c\},\\
\widehat\theta_{\mathrm{HL}}(c)
&=\frac{1}{n_{\mathrm{HL}}}\sum_{k=1}^V\sum_{i\in\mathcal{I}_k}
\hat w_{\mathrm{HL}}^{(-k)}(X_i^{\mathrm{HL}})\\
&\quad\times\{Y_i^{\mathrm{HL}}-\hat m_Y^{(-k)}(X_i^{\mathrm{HL}})\}
\Ind\{\hat m_S^{(-k)}(X_i^{\mathrm{HL}})\ge c\},\\
\widehat\theta_{\mathrm{AL}}(c)
&=\frac{1}{n_{\mathrm{AL}}}\sum_{k=1}^V\sum_{j\in\mathcal{J}_k}
\hat w_{\mathrm{AL}}^{(-k)}(X_j^{\mathrm{AL}})
\hat m_Y^{(-k)}(X_j^{\mathrm{AL}})\\
&\times K_h\{c-\hat m_S^{(-k)}(X_j^{\mathrm{AL}})\}\!
\{S_j^{\mathrm{AL}}-\hat m_S^{(-k)}(X_j^{\mathrm{AL}})\}.
\end{align*}
The sum $\widehat\theta_{\mathrm{T}}(c)+\widehat\theta_{\mathrm{HL}}(c)$ is the cross-fitted sample analog of $\Psi(m,w;\hat I_c)$: it removes the first-order outcome-regression error while treating the fitted decision rule as fixed. $\widehat\theta_{\mathrm{AL}}(c)$ is the sample analog of the auxiliary correction above and addresses the remaining error from replacing $m_S^\star$ by $\hat m_S$ inside the indicator. Thus, the proposed estimator is
\[
\hat\tau_1(c)
=\frac{\widehat\theta_{\mathrm{T}}(c)
+\widehat\theta_{\mathrm{HL}}(c)
+\widehat\theta_{\mathrm{AL}}(c)}{\widehat p_1},
\]

\paragraph{FPR, ROC, and AUC}
The corresponding FPR estimator is denoted by $\hat\tau_0(c)$. For numerical
ROC and AUC construction, $\hat\tau_1$ and $\hat\tau_0$ are evaluated on a
prespecified threshold grid. The deterministic endpoint pairs
$(\tau_0,\tau_1)=(1,1)$ and $(0,0)$ are then added, and equal-weight
least-squares isotonic projection produces $[0,1]$-valued sequences that
are nonincreasing in $c$. Denote the resulting projected curves by
$\tilde\tau_1$ and $\tilde\tau_0$. For a target FPR level $u\in[0,1]$, the estimated threshold and ROC value are $\hat c_u=\inf\{c:\tilde\tau_0(c)\le u\}$ and $\widehat{\operatorname{ROC}}^{cf}(u)
=\tilde\tau_1(\hat c_u)$.
The projected TPR and FPR sequences are linearly interpolated as functions of
$c$. The trapezoidal area under the estimated ROC curve defines
$\widehat{\operatorname{AUC}}^{cf}$, the AUC point estimate associated with
that curve. Standard errors and confidence intervals are based on the scalar one-step estimator $\widehat{\operatorname{AUC}}_{\mathrm{os}}^{cf}$, which targets AUC directly
and avoids requiring a distributional approximation for the entire estimated
ROC curve. Both estimators target the same population AUC, and
Appendix~\ref{FPRROCAUC} gives their constructions.

\section{Theoretical Results}
\label{Theory}

This section establishes asymptotic linear expansions, Gaussian limits, and
consistent standard errors for TPR and FPR at fixed thresholds and for ROC at
fixed FPR levels. It also establishes uniform consistency of the projected ROC
estimator and its trapezoidal AUC, together with asymptotic normality and
consistent standard-error estimation for the scalar one-step AUC estimator.

\subsection{TPR and FPR at a Fixed Threshold}

Throughout the threshold-specific results, $c$ is a fixed interior threshold and $h\to0$ along the joint sample-size sequence indexed by $n$. Define the reference rate $a_n=\{n_{\mathrm{HL}}^{-1}+(n_{\mathrm{AL}}h)^{-1} +n_{\mathrm{T}}^{-1}\}^{1/2}$. The effective auxiliary sample size is $n_{\mathrm{AL}}h$ (only observations near $c$ contribute to the auxiliary correction). The three source-specific influence terms are
\begin{align*}
\psi_{\mathrm{HL},1}(c)
=&p_1^{-1}w_{\mathrm{HL}}(X)\{Y-m_Y^\star(X)\}\{\Ind_c(X)-\tau_1(c)\},\\
\psi_{\mathrm{AL},1,h}(c)
&=p_1^{-1}w_{\mathrm{AL}}(X)m_Y^\star(X)
K_h\{c-m_S^\star(X)\}\\
&\quad\times\{S-m_S^\star(X)\},\\
\psi_{\mathrm{T},1}(c)
&=p_1^{-1}m_Y^\star(X)
\{\Ind_c(X)-\tau_1(c)\}.
\end{align*}
They capture variation from gold-outcome residuals,
surrogate model residuals near $c$, and target-population averaging,
respectively. For compact notation, write
$\psi_{\mathrm{AL},1}(c)=\psi_{\mathrm{AL},1,h}(c)$ and let
$\mathcal G=\{\mathrm{HL},\mathrm{AL},\mathrm{T}\}$. Their exact aggregate
variance is
$s_{1,n}^2(c)=\sum_{g\in\mathcal G}n_g^{-1}
\Var_g\{\psi_{g,1}(c)\}$.
Thus $a_n$ gives the stochastic fluctuation order, while $s_{1,n}(c)$
retains the source-specific variance constants. The condition
$s_{1,n}(c)\asymp a_n$ below is a nondegeneracy condition ensuring that the
reference rate matches the actual variance.

\begin{theorem}[Three-Sample Asymptotic Normality]
\label{thm:asymptotic_normality_tpr}
Suppose Assumptions~\ref{ass:smooth}--\ref{ass:higher_smoothness} in
Appendix~\ref{app:formal_assumptions} hold. Then $\hat\tau_1(c)-\tau_1(c)=\sum_{g\in\mathcal G}\mathbb P_{g,n}\psi_{g,1}(c)+o_p(a_n)$. If additionally $s_{1,n}(c)\asymp a_n$, the remainder is
$o_p\{s_{1,n}(c)\}$ and
$\{\hat\tau_1(c)-\tau_1(c)\}/s_{1,n}(c)\Rightarrow\mathcal N(0,1)$.
\end{theorem}

The expansion separates uncertainty from the three independent samples. Its
nonstandard component is the auxiliary-source term, which uses score residuals
near $c$ to cancel $B_S(c)$, the error induced by replacing $m_S^\star$ with
$\hat m_S$ inside the indicator. Under the two parts of
Assumption~\ref{ass:rate},
the cancellation remainder is
$O_p(r_{S,2}^2+h^2r_{S,2})+o_p(a_n)=o_p(a_n)$, where
$r_{S,2}$ is the cross-fitted $L_2$ error of $\hat m_S$.
Appendix~\ref{methodintuition} gives the proof. The centered factor
$\Ind_c-\tau_1(c)$ in the HL and target terms follows from the ratio expansion
with the estimated prevalence; Appendix~\ref{sqrtp1appendix} gives the details.

The relative sizes of $n_{\mathrm{HL}}^{-1}$, $(n_{\mathrm{AL}}h)^{-1}$, and $n_{\mathrm{T}}^{-1}$ determine the normalization and the source of first-order uncertainty. If $(n_{\mathrm{AL}}h)^{-1}=o(n_{\mathrm{HL}}^{-1}+n_{\mathrm{T}}^{-1})$, the auxiliary contribution from estimating $m_S^\star$ is negligible and the limit is the same as if $m_S^\star$ were known. If $(n_{\mathrm{AL}}h)^{-1}$ dominates the other two variance orders, the convergence rate is instead $(n_{\mathrm{AL}}h)^{-1/2}$. When all three orders are comparable, all three samples contribute to the standard error. A $\sqrt{n_{\mathrm{HL}}}$ normalization arises when
$(n_{\mathrm{AL}}h)^{-1}=o(n_{\mathrm{HL}}^{-1})$ and
$n_{\mathrm{T}}^{-1}=o(n_{\mathrm{HL}}^{-1})$. More generally, the theorem quantifies the first-order uncertainty from estimating $m_S^\star$ when it enters a hard threshold.

\begin{theorem}[TPR Standard Error and Confidence Interval]
\label{thm:variance_and_ci}
Let $\hat\psi_{g,1,r}(c)$ be the cross-fitted plug-in version of
$\psi_{g,1}(c)$,
defined explicitly in Appendix~\ref{ROCAUCCorollary}. For each
$g\in\mathcal G$, let
$\bar\psi_{g,1}(c)=n_g^{-1}\sum_{r=1}^{n_g}\hat\psi_{g,1,r}(c)$.
Under Assumptions~\ref{ass:smooth}--\ref{ass:higher_smoothness} and the nondegeneracy
condition $s_{1,n}(c)\asymp a_n$,
\[
\widehat{\mathrm{SE}}_{\mathrm{TPR}}^2(c)
=\sum_{g\in\mathcal G}
\frac{1}{n_g^2}\sum_{r=1}^{n_g}
\{\hat\psi_{g,1,r}(c)-\bar\psi_{g,1}(c)\}^2
\]
satisfies
$\widehat{\mathrm{SE}}_{\mathrm{TPR}}(c)/s_{1,n}(c)\to_p1$. Writing $z_q$ for
the $q$th quantile of $\mathcal N(0,1)$, a pointwise $(1-\alpha)$ confidence
interval is $\hat\tau_1(c)\pm
z_{1-\alpha/2}\widehat{\mathrm{SE}}_{\mathrm{TPR}}(c)$.
\end{theorem}

The FPR estimator admits an analogous three-source expansion and Gaussian
limit; Corollary~\ref{cor:pointwise_fpr} in
Appendix~\ref{ROCAUCCorollary} gives the formal statement and the corresponding
influence-function substitutions and plug-in standard error.

\subsection{ROC and AUC Inference}
Let $Z=m_S^\star(X)$, let $f_Z$ be its density under the target law, and set
$p_0=1-p_1$. The weighted target score densities are
$q_1(t)=\mathbb E_{\mathrm{T}}\{m_Y^\star(X)\mid Z=t\}f_Z(t)$ and
$q_0(t)=\mathbb E_{\mathrm{T}}\{1-m_Y^\star(X)\mid Z=t\}f_Z(t)$.
For $g\in\mathcal G$, the FPR influence terms $\psi_{g,0}(c)$ are obtained
from $\psi_{g,1}(c)$ by replacing
$m_Y^\star$ with $1-m_Y^\star$,
$Y-m_Y^\star$ with $m_Y^\star-Y$ in the HL term, $p_1$ with $p_0$, and
$\tau_1(c)$ with $\tau_0(c)$. For a fixed $u\in(0,1)$, let
$c_u$ be the unique solution to $\tau_0(c_u)=u$. Because
$\tau_1'(c)=-q_1(c)/p_1$ and $\tau_0'(c)=-q_0(c)/p_0$, the slope ratio is
$\lambda_u=\tau_1'(c_u)/\tau_0'(c_u)
=p_0q_1(c_u)/\{p_1q_0(c_u)\}$. The source-specific ROC influence terms are
$\phi_{g,u}=\psi_{g,1}(c_u)-\lambda_u\psi_{g,0}(c_u)$, with aggregate
variance $v_n^2(u)=\sum_{g\in\mathcal G}n_g^{-1}
\Var_g(\phi_{g,u})$.
The first term is the source-$g$ contribution to TPR estimation at
$c_u$; the second accounts for estimating $c_u$ by inverting the FPR
curve.

\begin{theorem}[ROC Inference at a Fixed FPR]
\label{thm:pointwise_roc}
Suppose Assumptions~\ref{ass:smooth}--\ref{ass:roc_local} hold, and
$v_n(u)\asymp a_n$. Then
\[
\widehat{\operatorname{ROC}}^{cf}(u)-\operatorname{ROC}(u)
=\sum_{g\in\mathcal G}\mathbb P_{g,n}\phi_{g,u}
+o_p\{v_n(u)\}.
\]
Moreover, $\{\widehat{\operatorname{ROC}}^{cf}(u)-
\operatorname{ROC}(u)\}/v_n(u)\Rightarrow\mathcal N(0,1)$.
\end{theorem}

To estimate $v_n(u)$, the slope ratio is estimated by
$\hat\lambda_u=(1-\hat p_1)\hat q_1(\hat c_u)/
\{\hat p_1\hat q_0(\hat c_u)\}$, using the target score-density estimates in
Appendix~\ref{ROCAUCCorollary}. The plug-in ROC influence values are
$\hat\phi_{g,r,u}=\hat\psi_{g,1,r}(\hat c_u)-
\hat\lambda_u\hat\psi_{g,0,r}(\hat c_u)$, with source-specific means
$\bar\phi_{g,u}=n_g^{-1}\sum_{r=1}^{n_g}\hat\phi_{g,r,u}$.
Appendix~\ref{ROCAUCCorollary} states the target score-density estimators and
sufficient conditions for $\hat\lambda_u\to_p\lambda_u$.
Here $\psi_{g,1,r}(c)$ and $\psi_{g,0,r}(c)$ are the observation-$r$
contributions from source $g$ to the TPR and FPR influence functions,
respectively. Their cross-fitted plug-in versions are given under
\emph{Plug-in influence values} in Appendix~\ref{ROCAUCCorollary}.

\begin{corollary}[ROC Standard Error and Confidence Interval at a Fixed FPR]
\label{cor:roc_se}
Under the conditions of Theorem~\ref{thm:pointwise_roc}, if additionally $\hat\lambda_u\to_p\lambda_u$ and, for every $g\in\mathcal G$ and
$y\in\{0,1\}$,
\[
\frac{1}{n_g}\sum_{r=1}^{n_g}
\left\{\hat\psi_{g,y,r}(\hat c_u)-\psi_{g,y,r}(c_u)\right\}^2
=o_p(d_g),
\]
where $d_{\mathrm{HL}}=d_{\mathrm{T}}=1$ and
$d_{\mathrm{AL}}=h^{-1}$. Then
\[
\widehat{\operatorname{SE}}_{\mathrm{ROC}}^2(u)
=\sum_{g\in\mathcal G}\frac{1}{n_g^2}
\sum_{r=1}^{n_g}\{\hat\phi_{g,r,u}-\bar\phi_{g,u}\}^2
\]
satisfies
$\widehat{\operatorname{SE}}_{\mathrm{ROC}}(u)/v_n(u)\to_p1$.
Hence a pointwise $(1-\alpha)$ confidence interval is
$\widehat{\operatorname{ROC}}^{cf}(u)
\pm z_{1-\alpha/2}\widehat{\operatorname{SE}}_{\mathrm{ROC}}(u)$.
\end{corollary}

Under local strict monotonicity and uniform accuracy of the unprojected curves, interpolation
and isotonic projection do not change the estimator to first order at $c_u$;
part (b) of Assumption~\ref{ass:roc_local} states the required monotonicity and local stochastic-process
conditions. Simultaneous ROC bands need a uniform Gaussian approximation.

AUC has a different asymptotic structure. For the scalar one-step estimator introduced
in Section~\ref{Method}, the AUC functional has a bounded derivative with respect to score
perturbations, so its auxiliary-source influence term has variance of order
one rather than $h^{-1}$. Its contribution to the estimator variance is
therefore of order $n_{\mathrm{AL}}^{-1}$, and AUC inference does not
require a Gaussian-process limit for the full ROC curve. Write
$A=\operatorname{AUC}$. Appendix~\ref{FPRROCAUC} gives the construction of
$\widehat{\operatorname{AUC}}_{\mathrm{os}}^{cf}$, and
Appendix~\ref{ROCAUCCorollary} gives its source-specific influence functions
$\varphi_g^A$, $g\in\mathcal G$, and plug-in variance estimator. The three
influence functions correspond to gold-outcome estimation, surrogate-derived model
estimation, and target pairwise averaging. Let
$\rho_n=(n_{\mathrm{HL}}^{-1}+n_{\mathrm{AL}}^{-1}+n_{\mathrm{T}}^{-1})^{1/2}$
and $\sigma_{A,n}^2=\sum_{g\in\mathcal G}n_g^{-1}
\Var_g(\varphi_g^A)$.

\begin{theorem}[Asymptotic Normality of the One-Step AUC Estimator]
\label{thm:auc_clt}
Suppose part (a) of Assumption~\ref{ass:smooth} and
Assumption~\ref{ass:auc_path} hold and
$\sigma_{A,n}\asymp\rho_n$. Then
\[
\widehat{\operatorname{AUC}}_{\mathrm{os}}^{cf}-A
=\sum_{g\in\mathcal G}\mathbb P_{g,n}\varphi_g^A+o_p(\sigma_{A,n}).
\]
Moreover, $\{\widehat{\operatorname{AUC}}_{\mathrm{os}}^{cf}-A\}/
\sigma_{A,n}\Rightarrow\mathcal N(0,1)$.
\end{theorem}

\paragraph{Bandwidth Choice.} $h$ appears in the AL correction
$K_h\{c-\hat m_S(X)\}\{S-\hat m_S(X)\}$. Estimation error in $\hat m_S$
can change a hard-threshold decision only when the score is close to $c$, so the kernel uses AL residuals near $c$ to correct the resulting
error. A smaller $h$ tightens correction around $c$ but reduces effective AL observations, raising variability. A larger $h$ uses more observations, is more stable, but increases localization bias by including scores farther from $c$. The pointwise theory, therefore, requires $h\to0$, $n_{\mathrm{AL}}h\to\infty$, and $\bar r_S/h=o_p(1)$, where $\bar r_S$ is the maximum uniform error of $\hat m_S$ across the cross-fitting folds. The neighborhood must shrink, contain increasing local information,
and remain wider than the score-estimation error. The remaining joint-rate
conditions are given in Assumption~\ref{ass:rate}; fixed-$u$ ROC inference
additionally requires $a_n=o(h)$ under
Assumption~\ref{ass:roc_local}. Proposition~\ref{prop:oracle_bandwidth} formalizes this trade-off: the AL
variance is of order $(n_{\mathrm{AL}}h)^{-1}$, but the leading
localization bias $\propto h^2$. When the AL term dominates,
$h$ is chosen so that this bias is negligible relative to the standard error
while $n_{\mathrm{AL}}h$ continues to diverge; the proposition gives the
corresponding conditional benchmark.

Because $m_S^\star(X)\in[0,1]$, one reference family satisfying these rate requirements is $h_{\mathrm{ref}}=C n_{\mathrm{AL}}^{-1/4}$, $C>0$. When the three sample sizes are of the same order, and the nuisance errors are
$o_p(n_{\mathrm{AL}}^{-3/8})$, this choice satisfies the preceding rate
requirements. A prespecified $C$ controls the finite-sample width of the local
neighborhood. Sensitivity can be assessed over a range of values by reporting
both $h$ and $n_{\mathrm{AL}}h$; Simulation studies illustrate the
effect of these two quantities. 

\section{Simulation Studies}
\label{simulation}

Studies~1, 2, and 4 examine the theory under controlled conditions using correctly specified nonlinear nuisance bases and analytically known, untruncated density ratios. They compare Full with two fixed-score transport baselines: importance-weighted gold-label validation (IW-HL) and transported score calibration (TSC). Study~3 instead ablates two correction terms under estimated weights for density ratios. Each experiment uses 500 replications; proposed estimators use five-fold cross-fitting. 

DGP-A has correlated Gaussian covariates and smooth interactions; DGP-B has a symmetric bimodal target law with distinct saturating and quadratic interaction effects. Set \(h=Cn_{\mathrm{AL}}^{-1/4}\), with \(C=1\) unless varied. Complete data-generating processes (DGPs), comparator definitions, diagnostics, and supplementary results are in Appendix~\ref{sec:sim_setup_details}.


\subsection{Study 1: Sample-Size Scaling and Standard-Error Accuracy.}
\label{sim:scaling}

\begin{figure*}[t]
  \centering
  \includegraphics[width=\textwidth]{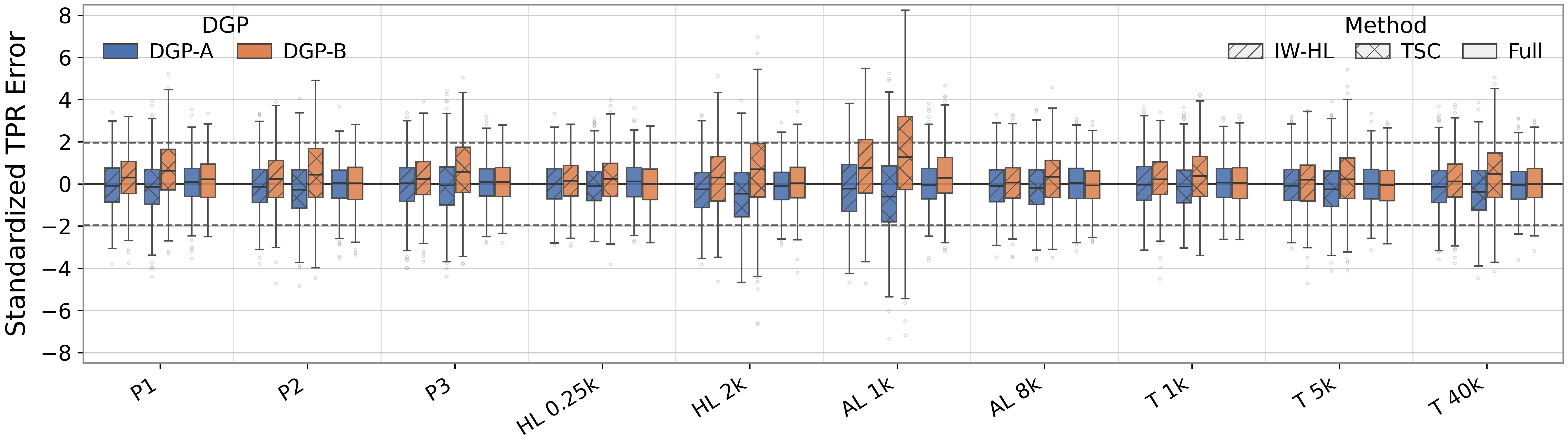}
  \caption{Study 1 Standardized TPR errors at \(c=0.5\). Dashed lines mark \(\pm1.96\), and outliers are shown.}
  \label{fig:sim_scaling_tpr}
\end{figure*}

We evaluate the three-source expansion in
Theorem~\ref{thm:asymptotic_normality_tpr}. The proportional designs are
\(\mathrm{P1}=(500,2000,10000)\),
\(\mathrm{P2}=(1000,4000,20000)\), and
\(\mathrm{P3}=(2000,8000,40000)\), where each triple is
\((n_{\mathrm{HL}},n_{\mathrm{AL}},n_{\mathrm T})\). Around P2, we also vary
one finite sample at a time:
\(n_{\mathrm{HL}}\in\{250,2000\}\),
\(n_{\mathrm{AL}}\in\{1000,8000\}\), and
\(n_{\mathrm T}\in\{1000,5000,40000\}\).

In Figure~\ref{fig:sim_scaling_tpr}, at \(c=0.5\), Full TPR coverage over P1--P3 is \(0.926\)--\(0.952\); from P1 to P3, empirical SD falls \(0.0295\to0.0155\) in DGP-A and \(0.0323\to0.0161\) in DGP-B, with mean SE tracking it (\(0.0281\to0.0151\) and \(0.0292\to0.0161\)). Across the six P1--P3-by-DGP cells, Full is closest to 0.95 coverage in every cell, has lower MSE than IW-HL in every cell, and has mean SE/SD \(0.957\), versus \(0.834\) for IW-HL and \(0.695\) for TSC. Across all 20 cells, mean coverage/SE--SD is \(0.934/0.950\) for Full, \(0.879/0.821\) for IW-HL, and \(0.805/0.701\) for TSC. TSC sometimes lowers RMSE through smoothing but has optimistic conditional SEs; the evidence here is for calibrated unconditional inference, not uniform RMSE dominance. Including the appendix FPR results, Full is closest to both 0.95 coverage and unit SE/SD in 39 of the 40 Study~1 cells at \(c=0.5\) (Figure~\ref{fig:sim_scaling_fpr}). Varying each sample exposes all three variance contributions (Table~\ref{tab:sim_scaling_sdse}); weaker cells remain in Table~\ref{tab:sim_calibration_ranges}.

\subsection{Study 2: Bandwidth Sensitivity.}
\label{sim:bandwidth}

At P2, we vary \(C\in\{0.5,0.75,1,1.5,2\}\), which gives \(h\in\{0.063,0.094,0.126,0.189,0.251\}\). IW-HL and TSC do not use \(h\) and are shown once in the Appendix Figure~\ref{fig:sim_bandwidth_tpr}, rather than being duplicated over five bandwidths. 

Larger bandwidths reduce variability but can make the analytical SE optimistic. At \(c=0.5\), \(C\leq1\), Full coverage is \(0.934\)--\(0.960\). From \(C=0.5\) to \(2\), empirical SD falls \(21.5\%\)/\(24.5\%\) in DGP-A/B, while mean SE falls faster, \(28.9\%\)/\(36.1\%\), reducing coverage to \(0.922/0.862\). Averaged over the ten Full bandwidth-by-DGP cells, coverage is \(0.930\); the \(h\)-free IW-HL and TSC references have coverage \(0.879\) and \(0.804\). Point-error rankings differ, and lower-threshold FPR contains an appendix counterexample. Thus this supports prespecified, threshold-specific variance--calibration sensitivity, not a universal \(C\) or uniform dominance.

\subsection{Study 3: Unified Correction Ablation.}
\label{ablat}

We now use \(\mathrm{P4}=(4000,16000,80000)\), simultaneous moderate two-source shift, estimated density ratios, \(h=0.5n_{\mathrm{AL}}^{-1/4}\), and the same scalar one-step AUC implementation as in the theory. We compare: (i) the plug-in estimator; (ii) HL-only, the fixed-score debiased baseline; (iii) AL-only; (iv) the full estimator; and (v) the full estimating equation with oracle nuisance functions. Omitting the HL correction also omits its prevalence correction. The main display uses TPR at \(c=0.5\), ROC at the low-FPR point \(u=0.1\), and scalar one-step AUC. Smaller-sample and bandwidth sensitivity results are reported in Appendix~\ref{app:ablattpr}.

\begin{figure}[t]
  \centering
  \includegraphics[width=\linewidth]{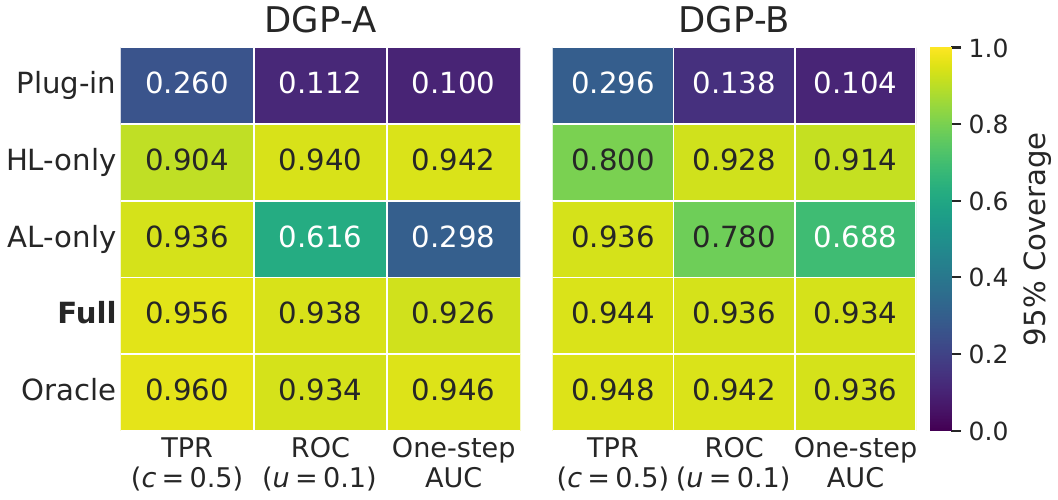}
  \caption{Study 3 empirical 95\% coverage over 500 replications. Coverage, rather than the magnitude of the performance metric, is the target of this comparison.}
  \label{fig:sim_ablation_coverage}
\end{figure}

Figure~\ref{fig:sim_ablation_coverage} shows that the plug-in standard error severely understates nuisance-estimation uncertainty: across the six DGP-by-target cells, its coverage is \(0.100\)--\(0.296\). Full coverage is \(0.926\)--\(0.956\), close to the oracle-nuisance range \(0.934\)--\(0.960\), and its mean-SE/empirical-SD ratios are \(0.927\)--\(1.036\). The partial corrections are target dependent: HL-only misses the threshold-specific AL contribution, whereas AL-only misses substantial global-score uncertainty. The complete bias, SD, SE, and coverage comparison in Appendix Table~\ref{tab:sim_ablation_full} shows why both corrections are needed when both nuisance directions are first order.

\subsection{Study 4: Fixed-\(u\) ROC and Scalar AUC Inference.}
\label{sim:roc_auc}
Using P1--P3 and \(C=1\), we evaluate
\(\operatorname{ROC}(u)\) at \(u\in\{0.1,0.2\}\) and the scalar one-step AUC.
The ROC estimator inverts the separately projected TPR/FPR curves using a
common grid-based interpolation convention and the slope-adjusted influence
function in Theorem~\ref{thm:pointwise_roc}; AUC uses
Theorem~\ref{thm:auc_clt} and the separate pilot
\(b=n_{\mathrm T}^{-1/5}\). In Appendix Figure~\ref{fig:sim_roc_auc}, Full ROC coverage is \(0.882\)--\(0.954\) and AUC coverage is
\(0.932\)--\(0.950\); the worst case is DGP-B/P2 at \(u=0.1\), while AUC
SE/SD is \(0.974\)--\(1.006\). Across all 18 cells, the mean coverage is
\(0.927\) for Full, \(0.911\) for IW-HL, and \(0.919\) for TSC. The AUC results are
comparable, where Full has a more stable average ROC coverage despite sometimes higher RMSE. The DGP-A/P2 Full Q--Q diagnostic has well-aligned centers
but visible tail deviations (Figure~\ref{fig:sim_qq_appendix}). We also provide a regularized
dictionary diagnostic regarding flexible learners in Appendix~\ref{app:flexible_ml_diagnostic}.

\section{Real Data Studies}
\label{sec:real_data}

We treat the two applications as complementary validation tests. Chatbot Arena asks whether inference remains informative and the source corrections remain consequential under a natural temporal shift. ACS-Income supplies a stricter gold-label validation: after a label-masking analysis is frozen, do the corrections move the estimates toward performance computed from revealed Target outcomes? In both studies, the HL, AL, and Target cohorts are disjoint, and Target outcomes are inaccessible during estimation. The revealed outcomes evaluate frozen finite-sample scores rather than the population model \(m_S^\star\); agreement is therefore corroborating evidence, not a bias or coverage calculation.

\paragraph{Chatbot Arena: validation under natural temporal shift.}
\label{sec:arena_temporal_study}
Chatbot Arena provides pairwise human preferences \citep{chiang2024chatbot}. We evaluate a downstream score learned from LLM-judge labels, rather than raw judge--human agreement \citep{zheng2023judging}. Exact-ID linkage to timestamped Arena-33K \citep{lmsys2023arena33k,agieai2023arena33kmirror} yields user-disjoint early HL (\(X,Y\); \(n=2{,}816\)), middle AL (\(X,S\); \(n=3{,}590\)), and later Target (\(X\); \(n=1{,}499\)) cohorts, without artificial covariate tilting. The score uses multilingual prompt/response embeddings, named metadata, and A/B-swap augmentation; the existing six-model judge panel supplies \(S\). Figure~\ref{fig:arena_temporal_main} displays the changing candidate-model mix. Appendix~\ref{app:arena_temporal_data}, Algorithm~\ref{alg:arena_temporal_workflow}, and Figure~\ref{fig:arena_temporal_shift} document linkage, masking, overlap, and label construction.

\begin{figure}[t]
\centering
\includegraphics[width=0.95\columnwidth]{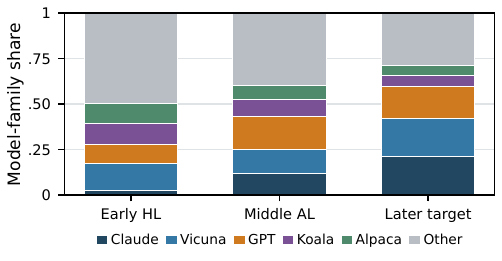}
\caption{Natural temporal shift in Chatbot Arena. Each stacked bar is the distribution of candidate-model family appearances within a fixed cohort; quantitative overlap diagnostics appear in Appendix Figure~\ref{fig:arena_temporal_shift}.}
\label{fig:arena_temporal_main}
\end{figure}

The result is informative in three respects. First, Full estimates target AUC as 0.848 (95\% CI [0.814, 0.881]): even its lower endpoint exceeds 0.80, and the interval width is only 0.067 despite two source-to-target shifts. At FPR 0.10, 0.15, and 0.20, target TPR is 0.562, 0.640, and 0.717. Second, the corrections are consequential rather than cosmetic. Full exceeds uncorrected outcome regression by 0.051 [0.020, 0.083] and the human-correction-only estimate by 0.023 [0.014, 0.031]; thus the complete construction cannot be reduced either to plug-in evaluation or to a fixed-score gold-label correction. Third, after all analyses were frozen, revealed human preferences give AUCs 0.810 and 0.812 for the cross-fitted and all-AL frozen scores. Although these are different estimands, their agreement and their point estimates above 0.80 corroborate the substantive conclusion that judge-derived discrimination persists in the later human-preference cohort; no Target \(S\) is used. Appendix Tables~\ref{tab:arena_temporal_all_auc}--\ref{tab:arena_temporal_paired} and Figures~\ref{fig:arena_temporal_roc}--\ref{fig:arena_temporal_sensitivity} report complete results and diagnostics.

\paragraph{ACS-Income.}
\label{ACS}
For eligible 2019 California ACS records, \(Y\) indicates income above \(\$50{,}000\), while \(S\) thresholds Gemini 3.5 Flash predictions at the same value. Prespecified opposing age-dependent sampling creates disjoint HL/AL/Target cohorts of size \(447/552/186{,}193\). We use five-fold cross-fitted cubic-spline logistic regression, with tuning confined to AL training folds; it ranks first by AL out-of-fold log loss in the reported exploratory six-learner sensitivity.

Here, the masked gold outcomes allow the correction direction to be empirically verified. Full estimates AUC as 0.8663 (95\% CI [0.8216, 0.9110]), versus 0.8383 for the human-correction-only ablation and 0.7074 for plug-in; despite only 447/552 labeled cases, its interval is 0.089 wide. After the reveal, the cross-fitted finite-score Target AUC is 0.8549. Full's descriptive cross-estimand discrepancy is 0.0114, 92.3\% smaller than plug-in's 0.1475 and 31.5\% smaller than the human-correction-only discrepancy of 0.0166; Full is also descriptively closest at FPR 0.10, 0.15, and 0.20. Moreover, the paired auxiliary-label correction is 0.0280 [0.0044, 0.0515], directly showing that this correction is non-negligible. This contribution remains positive across all six exploratory learners, with every learner-specific interval excluding zero, and Full is descriptively closer in five of six cases. Thus, the evidence supports a persistent auxiliary-source contribution without claiming uniform dominance. Figure~\ref{fig:acs_semisynthetic_validation} summarizes the scalar validation; Appendix Table~\ref{tab:acs_main_metrics}, Figure~\ref{fig:acs_spline_full_validation}, and Sections~\ref{supprealmain}--\ref{app:acs_learner_sensitivity} report complete metrics, diagnostics, and sensitivities.

\begin{figure}[t]
\centering
\includegraphics[width=0.95\columnwidth]{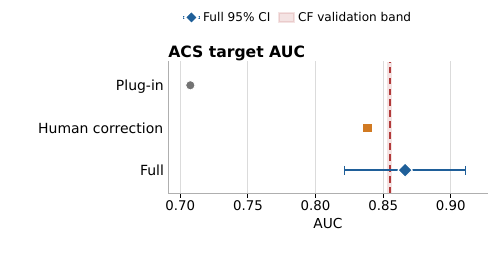}
\caption{ACS gold-label validation for the spline-3 score. Only Full carries the proposed three-source interval; the red line and band evaluate a different pooled cross-fitted finite-score object after the Target-outcome reveal. The paired Full-minus-human contrast is 0.028 [0.004, 0.051].}
\label{fig:acs_semisynthetic_validation}
\end{figure}

\section{Discussion}
\label{Discussion}

We develop a three-sample debiased framework to evaluate the predictive performance of the surrogate-derived model against the gold-standard label in the target population. Our main contribution is a source-aware debiasing principle that mitigates each first-order nuisance effect using sources, allowing for valid inference. In Chatbot Arena and ACS-Income, the real data study results corroborate our theoretical findings and illustrate our practical potential, beyond extensive simulation studies. Scope, limitations, and extensions appear in Supplementary Appendix~\ref{disclimappendix}.

\bibliography{main} 
\clearpage

\begin{appendices}

\counterwithin{assumption}{section}
\counterwithin{theorem}{section}
\counterwithin{corollary}{section}
\counterwithin{lemma}{section}
\counterwithin{prop}{section}
\counterwithin{equation}{section}
\counterwithin{figure}{section}
\counterwithin{table}{section}
\counterwithin{algorithm}{section}

The appendices are organized as follows. Appendix A states the formal assumptions and explains the regularity condition along the score-estimation path. Appendix B develops the methodological intuition and provides implementation details for density-ratio estimation and the FPR, ROC, and AUC constructions. Appendix C gives additional asymptotic results. Appendix D contains the proofs. Appendix E gives the complete simulation designs and additional numerical results. Appendix F documents the application designs and additional empirical results. Appendix G discusses practical considerations and extensions, including bandwidth sensitivity, nuisance-estimation rates, uniform inference, and computational cost.

\section{Formal Assumptions}
\label{app:formal_assumptions}

Unless otherwise stated, all limits are taken as $n\to\infty$ along a joint
asymptotic sequence with
$n_g=n_g(n)\to\infty$ for each
$g\in\{\mathrm{HL},\mathrm{AL},\mathrm{T}\}$, while the number $V$ of
cross-fitting folds remains fixed and, within each source, validation-fold
sizes differ by at most one. Thus, $n$ is a sequence index rather than one of
the three sample sizes.
Throughout the appendix, $\mathcal G=\{\mathrm{HL},\mathrm{AL},\mathrm{T}\}$.
We write $X_n\to_p X$ for convergence in probability and $X_n\Rightarrow X$
for convergence in distribution. For a positive, possibly random sequence
$r_n$, $X_n=O_p(r_n)$ means that $X_n/r_n$ is bounded in probability, whereas
$X_n=o_p(r_n)$ means $X_n/r_n\to_p0$. For positive deterministic
sequences $a_n$ and $b_n$, $a_n\asymp b_n$ means that $a_n/b_n$ is bounded
above and away from zero for all sufficiently large $n$; for positive random
sequences $A_n$ and $B_n$, $A_n\asymp_p B_n$ means
$A_n/B_n=O_p(1)$ and $B_n/A_n=O_p(1)$. For deterministic sequences $x_n$ and
$r_n>0$, $x_n=O(r_n)$ means that $|x_n|/r_n$ is eventually bounded, whereas
$x_n=o(r_n)$ means that $x_n/r_n\to0$. Arrows without a probability subscript
denote ordinary convergence of deterministic sequences. When a rate statement
is said to hold uniformly, the corresponding notation applies to the
displayed supremum or maximum.
For a probability law $\mathbb P$, write
$\|f\|_{\mathbb P,2}=\{\int f^2\,d\mathbb P\}^{1/2}$, and write
$\|f\|_r=\{\int |f(v)|^r\,dv\}^{1/r}$ for the Lebesgue $L_r$ norm,
$r\in\{1,2\}$. For a function on a set $\mathcal A$, let
$\|f\|_{\infty,\mathcal A}=\sup_{x\in\mathcal A}|f(x)|$; when the domain is
clear, write $\|f\|_\infty$. For a finite vector, $\|\cdot\|_\infty$ denotes
the maximum norm.

The assumptions are grouped by inferential target; no theorem requires every
condition below.

\subsection{Conditions for TPR and FPR at a Fixed Threshold}

\begin{assumption}[Basic Identification and Regularity Conditions]
\label{ass:smooth}
\textit{(a) Sampling, transportability, and source-support conditions.}
The three samples contain no shared observational units and are mutually independent; observations are i.i.d. within each sample. 
The target-outcome transportability equality defining $m_Y^\star$ holds, and the target covariate law is absolutely continuous with respect to both source covariate laws. The covariates $X$ have compact support. For constants $C<\infty$ and $\eta>0$, $w_{\mathrm{HL}}(X)\le C$ and $w_{\mathrm{AL}}(X)\le C$ almost surely, $\eta\le p_1\le1-\eta$, and $Y$, $S$, and the fitted regressions are uniformly bounded.

\medskip
\noindent\textit{(b) Kernel and bandwidth.}
\label{ass:bandwidth}
The kernel $K$ is a symmetric probability density with zero first moment,
finite second moment $\mu_2(K)=\int v^2K(v)\,dv$, and $\int K^2<\infty$.
It is twice continuously differentiable, and either has compact support or
$K$, $K'$, and $K''$ have exponentially decaying tails. The bandwidth $h=h_n$ is
deterministic, is fixed before evaluating the analysis folds, and satisfies
\[
h\to0,\qquad n_{\mathrm{AL}}h\to\infty.
\]

\medskip
\noindent\textit{(c) Local score-density regularity.}
\label{ass:joint_density}
Let $Z=m_S^\star(X)$. Under the target law, let $f_Z$ denote its density;
under the auxiliary law, let $f_{Z,\mathrm{AL}}$ denote its density. On a
compact neighborhood $\mathcal U$ of $c$, $f_Z$ is bounded and continuously
differentiable and $f_{Z,\mathrm{AL}}$ is bounded. The weighted
target-score densities
\begin{align*}
q_1(t)
&=\mathbb E_{\mathrm{T}}\{m_Y^\star(X)\mid Z=t\}f_Z(t),\\
q_0(t)
&=\mathbb E_{\mathrm{T}}\{1-m_Y^\star(X)\mid Z=t\}f_Z(t).
\end{align*}
have bounded continuous first derivatives on $\mathcal U$.
\end{assumption}

For each fold, write $\delta_Y^{(-k)}=\hat m_Y^{(-k)}-m_Y^\star$,
$\delta_S^{(-k)}=\hat m_S^{(-k)}-m_S^\star$, and, for
$d\in\{\mathrm{HL},\mathrm{AL}\}$,
$\delta_{w,d}^{(-k)}=\hat w_d^{(-k)}-w_d$. Define
\begin{align*}
r_{Y,2}&=\max_{1\le k\le V}\max_{g\in\mathcal G}
\|\delta_Y^{(-k)}\|_{\mathbb P_g,2},\\
r_{S,2}&=\max_{1\le k\le V}\max_{g\in\mathcal G}
\|\delta_S^{(-k)}\|_{\mathbb P_g,2},\\
r_{w,d,2}&=\max_{1\le k\le V}
\|\delta_{w,d}^{(-k)}\|_{\mathbb P_d,2},\\
\bar r_Y&=\max_{1\le k\le V}\|\delta_Y^{(-k)}\|_\infty,\\
\bar r_S&=\max_{1\le k\le V}\|\delta_S^{(-k)}\|_\infty,\\
\bar r_{w,d}&=\max_{1\le k\le V}\|\delta_{w,d}^{(-k)}\|_\infty.
\end{align*}
Recall that
$a_n=\{n_{\mathrm{HL}}^{-1}+(n_{\mathrm{AL}}h)^{-1}
+n_{\mathrm{T}}^{-1}\}^{1/2}$.

\begin{assumption}[Nuisance-Estimation Rates and Local Threshold Regularity]
\label{ass:rate}
\textit{(a) Cross-fitted nuisance consistency and remainder rates.}
The quantities defined above satisfy
\[
r_{Y,2}+r_{S,2}+\bar r_Y+\bar r_S
+\sum_{d\in\{\mathrm{HL},\mathrm{AL}\}}
\{r_{w,d,2}+\bar r_{w,d}\}=o_p(1).
\]
In addition, the following remainder-rate conditions hold:
\begin{align*}
r_{w,\mathrm{HL},2}r_{Y,2}&=o_p(a_n),\\
(\bar r_{w,\mathrm{AL}}+\bar r_Y)\bar r_S&=o_p(a_n),\\
\frac{\bar r_S}{h^2\sqrt{n_{\mathrm{AL}}}}&=o_p(a_n),\\
r_{S,2}^2+h^2r_{S,2}&=o_p(a_n),\\
\bar r_S/h&=o_p(1).
\end{align*}

\medskip
\noindent\textit{(b) Regularity along the score-estimation path.}
\label{ass:higher_smoothness}
Conditional on the fold-$k$ training data, set
$Z_{k,u}=m_S^\star(X)+u\delta_S^{(-k)}(X)$
for $u\in[0,1]$. For $a\in\{m_Y^\star,1-m_Y^\star\}$, suppose the signed
measure
\[
\nu_{k,u}^{(a)}(A)=\mathbb E_{\mathrm{T}}\!\left[
a(X)\delta_S^{(-k)}(X)
\Ind\{Z_{k,u}\in A\}\right]
\]
has a density $\mu_{k,u}^{(a)}$ on a fixed neighborhood $\mathcal U$ of
$c$. The map
$F_{k,a}(u)=\mathbb E_{\mathrm{T}}[a(X)\Ind\{Z_{k,u}\ge c\}]$ is
absolutely continuous with $F_{k,a}'(u)=\mu_{k,u}^{(a)}(c)$.

Uniformly over folds and both choices of $a$,
$u\mapsto\mu_{k,u}^{(a)}(c)$ is continuously differentiable and
\[
\sup_{0\le u\le1}|\partial_u\mu_{k,u}^{(a)}(c)|
=O_p(r_{S,2}^2).
\]
Moreover, $\mu_{k,1}^{(a)}$ is twice continuously differentiable on
$\mathcal U$ and
\[
\sup_{t\in\mathcal U}|\partial_t^2\mu_{k,1}^{(a)}(t)|
=O_p(r_{S,2}),
\]
with
\[
\sup_{\substack{s,t\in\mathcal U\\|s-t|\le h}}
|\partial_t^2\mu_{k,1}^{(a)}(s)-
\partial_t^2\mu_{k,1}^{(a)}(t)|=o_p(r_{S,2}).
\]
\end{assumption}

The consistency condition in part (a) also implies
\[
\frac{\bar r_{w,\mathrm{AL}}+\bar r_Y+\bar r_S}
{\sqrt{n_{\mathrm{AL}}h}}=o_p(a_n),
\]
because $a_n\ge(n_{\mathrm{AL}}h)^{-1/2}$. Thus this empirical-remainder
rate is a consequence of uniform consistency, not an additional assumption.

Part (b) of Assumption~\ref{ass:rate} is local and matches the two quantities
used in the proof:
variation of the path derivative is second order in the score error, while
the kernel-localization bias is $h^2$ times a quantity of order $r_{S,2}$.
A sufficient condition stated in terms of the conditional density of
the true score given the fitted score error is proved in
Lemma~\ref{lem:primitive_score_path}.

Thus Assumption~\ref{ass:smooth} collects the standard identification,
kernel, and local density conditions. The two parts of
Assumption~\ref{ass:rate} isolate what is specific to an estimated
hard-threshold rule: nuisance-rate control and regularity along the score-estimation path.

\subsection{Additional Conditions for ROC and Trapezoidal AUC}

\begin{assumption}[Additional ROC Conditions]
\label{ass:uniform_curve}
\textit{(a) Uniform curve and threshold-grid regularity.}
Let $\mathcal C=[\underline c,\overline c]$ be the compact population
threshold interval used for inversion. The numerical anchor points introduced
below are not elements of $\mathcal C$. The local score-density condition in
Assumption~\ref{ass:joint_density} and the regularity condition in part (b) of
Assumption~\ref{ass:higher_smoothness} hold uniformly for
$c\in\mathcal C$. The functions $\tau_1$ and $\tau_0$ are
continuous on $\mathcal C$, $\tau_0$ is strictly decreasing, and
$\tau_1(\underline c)=\tau_0(\underline c)=1$ and
$\tau_1(\overline c)=\tau_0(\overline c)=0$.

Let
$\mathcal C_n=\{\underline c=c_{1,n}<\cdots<
c_{M_n-1,n}=\overline c\}$ be the prespecified grid on $\mathcal C$, and let
$\Delta_{c,n}=\max_{2\le m\le M_n-1}
(c_{m,n}-c_{m-1,n})\to0$. In addition,
the grid is quasi-uniform: for some fixed $C_{\mathrm{grid}}<\infty$,
\[
\frac{\max_{2\le m\le M_n-1}(c_{m,n}-c_{m-1,n})}
{\min_{2\le m\le M_n-1}(c_{m,n}-c_{m-1,n})}
\le C_{\mathrm{grid}}.
\]
For the numerical projection only, augment this grid by artificial thresholds $c_{0,n}<\underline c$ and $c_{M_n,n}>\overline c$ (spaced to preserve the aforementioned quasi-uniformity, e.g., $c_{0,n} = \underline c - \Delta_{c,n}$), assigned the deterministic TPR/FPR pairs $(1,1)$ and $(0,0)$, respectively. Write $\mathcal C_n^+=\{c_{0,n}\}\cup\mathcal C_n\cup\{c_{M_n,n}\}$. No strict-monotonicity or smoothness condition is imposed outside $\mathcal C$.
Moreover,
\[
\frac{\log(M_n+1)}{\min\{n_{\mathrm{HL}},n_{\mathrm{T}}\}}\to0,
\qquad
\frac{\log(M_n+1)}{n_{\mathrm{AL}}h}\to0.
\]
For $y\in\{0,1\}$, let $\hat\tau_y^{\mathrm{dir}}(c)$ denote the raw
cross-fitted TPR or FPR estimator before isotonic projection, and define its
remainder after removing the three oracle empirical terms by
\[
R_{y,n}(c)=\hat\tau_y^{\mathrm{dir}}(c)-\tau_y(c)
-\sum_{g\in\mathcal G}\mathbb P_{g,n}\psi_{g,y}(c).
\]
The required uniform nuisance condition is
\[
\max_{c\in\mathcal C_n}|R_{y,n}(c)|=o_p(1),
\qquad y\in\{0,1\}.
\]

\medskip
\noindent\textit{(b) Local conditions for pointwise ROC inference.}
\label{ass:roc_local}
Fix $u\in(0,1)$ and let $c_u$ satisfy $\tau_0(c_u)=u$. On a neighborhood
$\mathcal U_u$ of $c_u$, $\tau_1$ and $\tau_0$ are continuously differentiable, their
derivatives are locally Lipschitz, and, for some $\kappa_u>0$,
\[
\inf_{c\in\mathcal U_u}
\min\{|\tau_1'(c)|,|\tau_0'(c)|\}\ge \kappa_u.
\]
For the local ROC analysis, let $\hat\tau_1$ and $\hat\tau_0$ be the
linearly interpolated raw estimators on the threshold grid, and let
$e_{y,n}=\hat\tau_y-\tau_y$.
There is a deterministic sequence
$b_n\downarrow0$ such that
\begin{align*}
b_n&=o(h),\qquad a_n=O(b_n),\\
\Delta_{c,n}&=o(a_n),\qquad \bar r_S=o_p(a_n).
\end{align*}
For $y\in\{0,1\}$,
\begin{align*}
\sup_{t\in\mathcal U_u}
|\hat\tau_y(t)-\tau_y(t)|&=O_p(b_n),\\
\sup_{\substack{s,t\in\mathcal U_u\\|s-t|\le M b_n}}
\left|e_{y,n}(s)-e_{y,n}(t)\right|&=o_p(a_n)
\end{align*}
for every fixed $M<\infty$.
\end{assumption}

Write $\tilde\tau_1$ and $\tilde\tau_0$ for the interpolated isotonic
projections. One sufficient set of conditions for local equicontinuity in
Assumption~\ref{ass:roc_local} consists of a locally uniform three-source expansion
with $o_p(a_n)$ remainder, together with the usual VC maximal inequality for
the HL and target threshold processes and the translation-class maximal
inequality for the AL kernel process. For a finite threshold grid, a
conservative sufficient set of entropy-rate conditions is
\[
b_n\log(M_n+1)\to0,\qquad
\frac{b_n}{h}\sqrt{\log(M_n+1)}\to0,
\]
together with a locally uniform $o_p(a_n)$ nuisance remainder.

Part (b) of Assumption~\ref{ass:roc_local} is stated for the unprojected curves.
Lemma~\ref{lem:isotonic_local}
below proves that isotonic projection is locally equivalent to the unprojected
estimator to first order under these conditions. The stronger scale-separation condition
$b_n=o(h)$ ensures that any isotonic block intersecting the local neighborhood is small relative to
the localization bandwidth. Up to logarithmic grid factors, it reduces to
$a_n=o(h)$; when the AL term determines the order of $a_n$, this is
$n_{\mathrm{AL}}h^3\to\infty$. The additional condition
$\bar r_S=o_p(a_n)$ is used in Lemma~\ref{lem:grid_interpolation} to make
the interpolated and directly evaluated estimators equivalent to first order.

Part (a) of Assumption~\ref{ass:uniform_curve} is used for uniform
consistency of the projected ROC curve; part (b) is additionally required
only for pointwise inference after inversion and isotonic projection.

\subsection{Conditions for One-Step AUC Inference}

For the AUC analysis, using the nuisance-error notation defined above, let
$Z=m_S^\star(X)$, $M=m_Y^\star(X)$, and
$\kappa(z,z')=\Ind\{z>z'\}+\tfrac12\Ind\{z=z'\}$. Define
$X'$ to be an independent copy of $X$ under $\mathbb P_{\mathrm{T}}$, and
write $\mathbb E_{\mathrm{T}\times\mathrm{T}}$ for expectation over the
product law of $(X,X')$. Define
\[
\eta(m,z)=\mathbb E_{\mathrm{T}\times\mathrm{T}}
[m(X)\{1-m(X')\}\kappa\{z(X),z(X')\}].
\]
Define $G_0(z)=\mathbb E_{\mathrm{T}}[(1-M)\kappa(z,Z)]$ and
$G_1(z)=\mathbb E_{\mathrm{T}}[M\kappa(Z,z)]$, and set
$\Gamma_Y(X)=G_0(Z)-G_1(Z)$ and
$\Gamma_S(X)=Mq_0(Z)-(1-M)q_1(Z)$.

For fold $k$, using the pilot functions constructed in
Appendix~\ref{FPRROCAUC}, let
\begin{align*}
\hat\Gamma_{Y,k}(x)
&=\hat G_{0,k}\{\hat m_S^{(-k)}(x)\}
-\hat G_{1,k}\{\hat m_S^{(-k)}(x)\},\\
\hat\Gamma_{S,k}(x)
&=\hat m_Y^{(-k)}(x)
\hat q_{0,k}\{\hat m_S^{(-k)}(x)\}\\
&\quad-\{1-\hat m_Y^{(-k)}(x)\}
\hat q_{1,k}\{\hat m_S^{(-k)}(x)\}.
\end{align*}
Let $r_{\Gamma_Y,2}$ and $r_{\Gamma_S,2}$ be the maximum, over folds, of
$\|\hat\Gamma_{Y,k}-\Gamma_Y\|_{\mathbb P_{\mathrm{T}},2}$ and
$\|\hat\Gamma_{S,k}-\Gamma_S\|_{\mathbb P_{\mathrm{T}},2}$, respectively.

\begin{assumption}[Additional One-Step AUC Conditions]
\label{ass:auc_path}
\textit{(a) AUC pathwise regularity.}
The target score is continuously distributed. The weighted score densities
$q_0$ and $q_1$ are bounded and continuously differentiable on the score
range. Uniformly over folds, the map
$(s,t)\mapsto\eta(M+s\delta_Y^{(-k)},Z+t\delta_S^{(-k)})$ is twice pathwise
differentiable on $[0,1]^2$, with first derivatives at $(0,0)$ given by
\begin{align*}
D_Y\eta[\delta_Y^{(-k)}]
&=\mathbb E_{\mathrm{T}}\{\delta_Y^{(-k)}(X)\Gamma_Y(X)\},\\
D_S\eta[\delta_S^{(-k)}]
&=\mathbb E_{\mathrm{T}}\{\delta_S^{(-k)}(X)\Gamma_S(X)\},
\end{align*}
and with second-order remainder
$O_p(r_{Y,2}^2+r_{Y,2}r_{S,2}+r_{S,2}^2)$.

\medskip
\noindent\textit{(b) AUC pilot consistency and product rates.}
\label{ass:auc_rates}
Let $\rho_n=(n_{\mathrm{HL}}^{-1}+n_{\mathrm{AL}}^{-1}
+n_{\mathrm{T}}^{-1})^{1/2}$. The pilot bandwidth $b$ satisfies
$b\to0$ and $n_{\mathrm{T}}b\to\infty$, and the pilot-function errors obey
$r_{\Gamma_Y,2}=o_p(1)$ and $r_{\Gamma_S,2}=o_p(1)$. In addition,
the two density-ratio estimators are $L_2$ consistent:
\[
r_{w,\mathrm{HL},2}+r_{w,\mathrm{AL},2}=o_p(1).
\]
The following product rates hold:
\begin{align*}
r_{Y,2}^2+r_{S,2}^2+r_{Y,2}r_{S,2}
&=o_p(\rho_n),\\
\{r_{w,\mathrm{HL},2}+r_{\Gamma_Y,2}\}r_{Y,2}
&=o_p(\rho_n),\\
\{r_{w,\mathrm{AL},2}+r_{\Gamma_S,2}\}r_{S,2}
&=o_p(\rho_n).
\end{align*}
The fitted density ratios and the four pilot functions
$\hat G_{0,k},\hat G_{1,k},\hat q_{0,k},\hat q_{1,k}$ are uniformly bounded
with probability tending to one.
\end{assumption}

The errors $r_{\Gamma_Y,2}$ and $r_{\Gamma_S,2}$ are measured after the pilot
functions are composed with the fitted nuisance functions, so they include both pilot
estimation error and replacement of $m_S^\star$ by $\hat m_S^{(-k)}$. They
are conditions on the nuisance functions entering the scalar AUC expansion rather than uniform conditions on the full
ROC process.

\paragraph{A sufficient condition for the AUC path expansion.}
The expansion in part (a) of Assumption~\ref{ass:auc_path} can be verified by
a pairwise version of the regularity condition in part (b) of
Assumption~\ref{ass:higher_smoothness}. Conditional on fold-$k$ training, let
\begin{align*}
D&=Z(X)-Z(X'),\\
\Delta_{S,k}&=\delta_S^{(-k)}(X)-\delta_S^{(-k)}(X'),\\
M_{k,s}(x)&=M(x)+s\delta_Y^{(-k)}(x),
\end{align*}
and consider the signed measure
\begin{align*}
\nu_{k,s,t}(A)
&=\mathbb E_{\mathrm{T}\times\mathrm{T}}\left[
M_{k,s}(X)\{1-M_{k,s}(X')\}\right.\\
&\qquad\left.\times\Delta_{S,k}
\Ind\{D+t\Delta_{S,k}\in A\}\right].
\end{align*}
It is sufficient that this measure have a density
$\mu_{k,s,t}$ near zero, uniformly in $(s,t)\in[0,1]^2$, with
\begin{align*}
\sup_{s,t}|\partial_t\mu_{k,s,t}(0)|&=O_p(r_{S,2}^2),\\
\sup_{s,t}|\mu_{k,s,t}(0)-\mu_{k,0,t}(0)|
&=O_p(r_{Y,2}r_{S,2}).
\end{align*}
Indeed,
$\partial_t\eta(M_{k,s},Z+t\delta_S^{(-k)})=\mu_{k,s,t}(0)$;
the fundamental theorem of calculus and the two bounds above control the
remainder from estimating $m_S^\star$ and the interaction remainder from
estimating both $m_Y^\star$ and $m_S^\star$, respectively.
Together with the elementary $O_p(r_{Y,2}^2)$ quadratic remainder from the
outcome weights, these bounds imply the stated remainder. They follow,
for example, from a differentiable joint weighted density of
$(D,\Delta_{S,k})$ whose first two $\Delta_{S,k}$-moment derivatives have the
displayed orders. Smoothness of the marginal density of $D$ alone is not
sufficient.

\section{Additional Methodological Details}
\label{app:method_details}

\subsection{Remainder for the Auxiliary Correction}
\label{methodintuition}

Fix one cross-fitting fold and condition on its training data. To quantify the
approximation in Section~\ref{Method}, write
$\delta_S=\hat m_S-m_S^\star$, set
$Z_u=m_S^\star(X)+u\delta_S(X)$ and let $\mu_u^{(1)}(t)$ denote the density
at $t$ of the signed measure
\[
A\mapsto\mathbb E_{\mathrm{T}}\!\left[
m_Y^\star(X)\delta_S(X)\Ind\{Z_u\in A\}\right].
\]
The path expansion in part (b) of Assumption~\ref{ass:higher_smoothness} gives
\[
B_S(c)=\int_0^1\mu_u^{(1)}(c)\,du,
\]
while the conditional mean of the auxiliary correction in
Section~\ref{Method} is
$-\int K_h(c-t)\,d\nu_1^{(1)}(t)$. Because the density
$\mu_1^{(1)}$ is assumed only on the neighborhood $\mathcal U$, split this
integral as
\[
-\int_{\mathcal U}K_h(c-t)\mu_1^{(1)}(t)\,dt
-\int_{\mathcal U^c}K_h(c-t)\,d\nu_1^{(1)}(t).
\]
The second term is zero for all sufficiently small $h$ when $K$ has compact
support and is $o_p(h^2r_{S,2})$ when $K$ has exponentially decaying tails.
Indeed, $c$ is an interior point of $\mathcal U$ and
$\|\nu_1^{(1)}\|_{\mathrm{TV}}\le
\mathbb E_{\mathrm{T}}|m_Y^\star\delta_S|=O_p(r_{S,2})$.
Thus the sum of the indicator drift and the auxiliary mean is
\begin{align*}
&\int_0^1\{\mu_u^{(1)}(c)-\mu_1^{(1)}(c)\}\,du\\
&\quad+\left[\mu_1^{(1)}(c)
-\int_{\mathcal U} K_h(c-t)\mu_1^{(1)}(t)\,dt\right]
+o_p(h^2r_{S,2}).
\end{align*}
The first line is the error from replacing $\mu_u^{(1)}(c)$ by
$\mu_1^{(1)}(c)$, and the second is the kernel-approximation error. They are
respectively $O_p(r_{S,2}^2)$ and
\[
-\frac{h^2\mu_2(K)}{2}\partial_t^2\mu_1^{(1)}(c)
+o_p(h^2r_{S,2}).
\]
The order-$h$ kernel term vanishes because $K$ is symmetric. Hence the sum
is $O_p(r_{S,2}^2+h^2r_{S,2})=o_p(a_n)$ under the two parts of
Assumption~\ref{ass:rate}.

\subsection{Density Ratio Estimation Procedure}
\label{densityratioestimation}

Set $\mathcal F_{\mathrm{HL},k}=\mathcal I_k$ and
$\mathcal F_{\mathrm{AL},k}=\mathcal J_k$. For
$g\in\{\mathrm{HL},\mathrm{AL}\}$ and fold $k$, train a source-versus-target
classifier using the source covariates whose indices are outside
$\mathcal F_{g,k}$ and the target covariates whose indices are outside
$\mathcal L_k$. Let $G^g=0$ denote membership in source $g$ and $G^g=1$
denote target membership. Define the target-membership probability
\[
\pi_g(x)=\mathbb P(G^g=1\mid X=x),
\]
and fit its estimator $\hat\pi_g^{(-k)}$ using only the fold-$k$ training observations above. Bayes' rule gives
\[
w_g(x)=\frac{p_{\mathrm{T}}(x)}{p_g(x)}
=\frac{\mathbb P(G^g=0)}{\mathbb P(G^g=1)}
\frac{\pi_g(x)}{1-\pi_g(x)}.
\]
Consequently, if $n_g^{(-k)}$ and $n_{\mathrm{T}}^{(-k)}$ are the two class sizes used to train the fold-$k$ classifier,
\[
\hat w_g^{(-k)}(x)=
\frac{n_g^{(-k)}}{n_{\mathrm{T}}^{(-k)}}
\frac{\hat\pi_g^{(-k)}(x)}{1-\hat\pi_g^{(-k)}(x)}.
\]
For example, logistic regression may parameterize
$\hat\pi_g^{(-k)}(x)=\operatorname{expit}(x^\top\hat\gamma_g^{(-k)})$,
where $\operatorname{expit}(t)=(1+e^{-t})^{-1}$; other calibrated
probabilistic classifiers may be used.

The two ratios are fitted separately because the HL-to-target and AL-to-target shifts need not agree. Their classifiers also need not belong to the same model class as either outcome regression. The theoretical estimator uses the untruncated ratios under the bounded-density-ratio condition. In computation, fitted class probabilities must be kept away from zero and one; any probability clipping or weight truncation is prespecified and reported as part of the implementation. Evaluation of $\hat w_g^{(-k)}$ is restricted to the held-out source fold $\mathcal{F}_{g,k}$.

\subsection{FPR, ROC, and AUC Construction}
\label{FPRROCAUC}
The target FPR is
\[
\tau_0(c)=\frac{1}{1-p_1}
\mathbb{E}_{\mathrm{T}}\!\left[
\{1-m_Y^\star(X)\}
\Ind\{m_S^\star(X)\ge c\}\right].
\]
Using the same three held-out folds as the TPR estimator, define
{\small
\begin{align*}
\widehat\nu_{\mathrm{T}}(c)
&=
\frac{1}{n_{\mathrm{T}}}\sum_{k=1}^V\sum_{\ell\in\mathcal{L}_k}
\{1-\hat m_Y^{(-k)}(X_\ell^{\mathrm{T}})\}\\
&\quad\times \Ind\{\hat m_S^{(-k)}(X_\ell^{\mathrm{T}})\ge c\},\\
\widehat\nu_{\mathrm{HL}}(c)
&=\frac{1}{n_{\mathrm{HL}}}\sum_{k=1}^V\sum_{i\in\mathcal{I}_k}
\hat w_{\mathrm{HL}}^{(-k)}(X_i^{\mathrm{HL}})
\{\hat m_Y^{(-k)}(X_i^{\mathrm{HL}})-Y_i^{\mathrm{HL}}\}\\
&\quad\times \Ind\{\hat m_S^{(-k)}(X_i^{\mathrm{HL}})\ge c\},\\
\widehat\nu_{\mathrm{AL}}(c)
&=\frac{1}{n_{\mathrm{AL}}}\sum_{k=1}^V\sum_{j\in\mathcal{J}_k}
\hat w_{\mathrm{AL}}^{(-k)}(X_j^{\mathrm{AL}})
\{1-\hat m_Y^{(-k)}(X_j^{\mathrm{AL}})\}\\
&\qquad\times K_h\{c-\hat m_S^{(-k)}(X_j^{\mathrm{AL}})\}
\{S_j^{\mathrm{AL}}-\hat m_S^{(-k)}(X_j^{\mathrm{AL}})\}.
\end{align*}
}
Then
\[
\hat\tau_0(c)
=\frac{\widehat\nu_{\mathrm{T}}(c)+\widehat\nu_{\mathrm{HL}}(c)
+\widehat\nu_{\mathrm{AL}}(c)}{1-\widehat p_1}.
\]

For numerical construction, evaluate the raw TPR and FPR estimators on the
grid $\mathcal C_n=\{c_1,\ldots,c_{M-1}\}\subset\mathcal C$. Add two
artificial thresholds $c_0<\underline c$ and $c_M>\overline c$ solely as
anchors, and assign them the deterministic endpoint pairs
$(\operatorname{FPR},\operatorname{TPR})=(1,1)$ and $(0,0)$, respectively.
These artificial thresholds lie outside the population interval on which
strict monotonicity is assumed. Separately project the two augmented sequences
by least-squares isotonic regression onto
\[
1\ge v_0\ge v_1\ge\cdots\ge v_M\ge0,
\]
because both rates are nonincreasing functions of the threshold. Denote the resulting piecewise-linear curves by $\tilde\tau_0(c)$ and $\tilde\tau_1(c)$. The theoretical results use equal grid weights.

For $u\in[0,1]$, define the generalized inverse and ROC estimator by
\begin{align*}
\hat c_u&=\inf\{c:\tilde\tau_0(c)\le u\},\\
\widehat{\operatorname{ROC}}^{cf}(u)
&=\tilde\tau_1(\hat c_u),
\end{align*}
using linear interpolation between adjacent threshold-grid points. On a grid $0=u_0<\cdots<u_R=1$, the AUC associated with this curve is estimated by
\begin{align*}
\widehat{\operatorname{AUC}}^{cf}
&=\sum_{r=1}^R\frac{u_r-u_{r-1}}{2}\\
&\quad\times\left\{\widehat{\operatorname{ROC}}^{cf}(u_r)
+\widehat{\operatorname{ROC}}^{cf}(u_{r-1})\right\}.
\end{align*}
Under part (a) of Assumption~\ref{ass:uniform_curve},
Theorem~\ref{thm:roc_auc_consistency} establishes uniform consistency of the
projected ROC estimator and consistency of the trapezoidal AUC when the
AUC-grid mesh tends to zero. Theorem~\ref{thm:pointwise_roc}
further gives pointwise inference for the ROC value at a fixed interior
false-positive rate. Uniform confidence bands require stronger stochastic
equicontinuity and uniform-in-threshold remainder bounds; scalar AUC inference
is developed next without such a process condition.

The trapezoidal estimator summarizes the projected ROC curve. For formal AUC
inference, the scalar one-step construction avoids a process approximation. Let
$\kappa(z,z')=\Ind\{z>z'\}+\tfrac12\Ind\{z=z'\}$ and, for
$k=1,\ldots,V$, write $n_{\mathrm{T},k}=|\mathcal L_k|$,
$\hat M_{k\ell}=\hat m_Y^{(-k)}(X_\ell^{\mathrm{T}})$, and
$\hat Z_{k\ell}=\hat m_S^{(-k)}(X_\ell^{\mathrm{T}})$. Define
\begin{align*}
\hat\eta_{\mathrm{T},k}
&=\frac{1}{n_{\mathrm{T},k}(n_{\mathrm{T},k}-1)}
\sum_{\substack{\ell,r\in\mathcal L_k\\ \ell\ne r}}
\hat M_{k\ell}(1-\hat M_{kr})
\kappa(\hat Z_{k\ell},\hat Z_{kr}),\\
\hat\eta_{\mathrm{T}}
&=\sum_{k=1}^V\frac{n_{\mathrm{T},k}}{n_{\mathrm{T}}}
\hat\eta_{\mathrm{T},k}.
\end{align*}
For a pilot bandwidth $b$, set $K_b(v)=b^{-1}K(v/b)$ and define
$\hat a_{1,k\ell}=\hat M_{k\ell}$ and
$\hat a_{0,k\ell}=1-\hat M_{k\ell}$. The held-out target summaries are
\begin{align*}
\hat G_{0,k}(z)
&=\frac{1}{n_{\mathrm{T},k}}\sum_{\ell\in\mathcal L_k}
(1-\hat M_{k\ell})\kappa(z,\hat Z_{k\ell}),\\
\hat G_{1,k}(z)
&=\frac{1}{n_{\mathrm{T},k}}\sum_{\ell\in\mathcal L_k}
\hat M_{k\ell}\kappa(\hat Z_{k\ell},z),\\
\hat q_{y,k}(z)
&=\frac{1}{n_{\mathrm{T},k}}\sum_{\ell\in\mathcal L_k}
\hat a_{y,k\ell}K_b(z-\hat Z_{k\ell}).
\end{align*}
The two source corrections are
\begin{align*}
\hat\eta_{\mathrm{HL}}
&=\frac{1}{n_{\mathrm{HL}}}\sum_{k=1}^V\sum_{i\in\mathcal I_k}
\hat w_{\mathrm{HL}}^{(-k)}(X_i^{\mathrm{HL}})
\{Y_i^{\mathrm{HL}}-\hat m_Y^{(-k)}(X_i^{\mathrm{HL}})\}\\
&\quad\times[\hat G_{0,k}\{\hat m_S^{(-k)}(X_i^{\mathrm{HL}})\}
-\hat G_{1,k}\{\hat m_S^{(-k)}(X_i^{\mathrm{HL}})\}],\\
\hat\eta_{\mathrm{AL}}
&=\frac{1}{n_{\mathrm{AL}}}\sum_{k=1}^V\sum_{j\in\mathcal J_k}
\hat w_{\mathrm{AL}}^{(-k)}(X_j^{\mathrm{AL}})
\{S_j^{\mathrm{AL}}-\hat m_S^{(-k)}(X_j^{\mathrm{AL}})\}\\
&\quad\times\bigl[\hat m_Y^{(-k)}(X_j^{\mathrm{AL}})
\hat q_{0,k}\{\hat m_S^{(-k)}(X_j^{\mathrm{AL}})\}\\
&\qquad-\{1-\hat m_Y^{(-k)}(X_j^{\mathrm{AL}})\}
\hat q_{1,k}\{\hat m_S^{(-k)}(X_j^{\mathrm{AL}})\}\bigr].
\end{align*}
The scalar one-step estimator is
\[
\widehat{\operatorname{AUC}}_{\mathrm{os}}^{cf}
=\frac{\hat\eta_{\mathrm{T}}+\hat\eta_{\mathrm{HL}}
+\hat\eta_{\mathrm{AL}}}{\hat p_1(1-\hat p_1)}.
\]
It targets the same population AUC as the area under the ROC curve. Its AL
correction uses the derivative of AUC with respect to score perturbations,
expressed through the weighted score densities $q_0$ and $q_1$, rather than a
separate correction at every threshold.

\section{Additional Theoretical Results}
\label{app:theory_details}

\subsection{Asymptotic Expansion for the Target Prevalence Estimator}
\label{sqrtp1appendix}

To formalize the asymptotic behavior of $\hat p_1$, define
\begin{align}
    \phi_{p_1}^{\mathrm{HL}} &= w_{\mathrm{HL}}(X)
    \{Y - m_Y^\star(X)\} \label{eq:phi_hl} \\
    \phi_{p_1}^{\mathrm{T}} &= m_Y^\star(X) - p_1. \label{eq:phi_t}
\end{align}

\begin{theorem}[Target Prevalence Estimator]
\label{thm:p1_consistency}
Under part (a) of Assumption~\ref{ass:smooth}, suppose that, uniformly over
folds, $\hat m_Y^{(-k)}$ converges to $m_Y^\star$ in both
$L_2(\mathbb P_{\mathrm{T}})$ and $L_2(\mathbb P_{\mathrm{HL}})$, and
$\hat w_{\mathrm{HL}}^{(-k)}$ converges to $w_{\mathrm{HL}}$ in
$L_2(\mathbb P_{\mathrm{HL}})$. If
$r_{w,\mathrm{HL},2}r_{Y,2}=o_p(b_{p,n})$, where
$b_{p,n}=(n_{\mathrm{HL}}^{-1}+n_{\mathrm{T}}^{-1})^{1/2}$, then
\[
\hat{p}_1 - p_1 = \frac{1}{n_{\mathrm{HL}}} \sum_{i=1}^{n_{\mathrm{HL}}} \phi_{p_1,i}^{\mathrm{HL}}
+ \frac{1}{n_{\mathrm{T}}} \sum_{\ell=1}^{n_{\mathrm{T}}} \phi_{p_1,\ell}^{\mathrm{T}} + o_p(b_{p,n}).
\]
\end{theorem}

\subsection{Pointwise FPR and ROC/AUC Extensions}
\label{ROCAUCCorollary}

\paragraph{Plug-in influence values.}
For $i\in\mathcal I_k$, $j\in\mathcal J_k$, and $\ell\in\mathcal L_k$, define
$\hat I_{g,r}^{(-k)}(c)=
\Ind\{\hat m_S^{(-k)}(X_r^g)\ge c\}$ and
\begin{align*}
\hat R_{Y,i}^{(-k)}
&=Y_i^{\mathrm{HL}}-\hat m_Y^{(-k)}(X_i^{\mathrm{HL}}),\\
\hat R_{S,j}^{(-k)}
&=S_j^{\mathrm{AL}}-\hat m_S^{(-k)}(X_j^{\mathrm{AL}}).
\end{align*}
The cross-fitted TPR influence values in
Theorem~\ref{thm:variance_and_ci} are
\begin{align*}
\hat\psi_{\mathrm{HL},1,i}(c)
&=\frac{\hat w_{\mathrm{HL}}^{(-k)}(X_i^{\mathrm{HL}})}{\hat p_1}
\hat R_{Y,i}^{(-k)}\\
&\quad\times
\{\hat I_{\mathrm{HL},i}^{(-k)}(c)-\hat\tau_1(c)\},\\
\hat\psi_{\mathrm{AL},1,j}(c)
&=\frac{\hat w_{\mathrm{AL}}^{(-k)}(X_j^{\mathrm{AL}})
\hat m_Y^{(-k)}(X_j^{\mathrm{AL}})}{\hat p_1}\\
&\quad\times K_h\{c-\hat m_S^{(-k)}(X_j^{\mathrm{AL}})\}
\hat R_{S,j}^{(-k)},\\
\hat\psi_{\mathrm{T},1,\ell}(c)
&=\frac{\hat m_Y^{(-k)}(X_\ell^{\mathrm{T}})}{\hat p_1}
\{\hat I_{\mathrm{T},\ell}^{(-k)}(c)-\hat\tau_1(c)\}.
\end{align*}

\begin{corollary}[Pointwise FPR Inference]
\label{cor:pointwise_fpr}
Let $\psi_{g,0}(c)$ be obtained from the TPR influence terms by replacing
$m_Y^\star$ by $1-m_Y^\star$, $Y-m_Y^\star$ by $m_Y^\star-Y$ in the
gold-source term, $p_1$ by $p_0=1-p_1$, and $\tau_1(c)$ by
$\tau_0(c)$. Set
$s_{0,n}^2(c)=\sum_{g\in\mathcal G}n_g^{-1}
\Var_g\{\psi_{g,0}(c)\}$. Under Assumptions~\ref{ass:smooth}--\ref{ass:higher_smoothness},
with $p_0$ bounded away from zero, the FPR estimator in
Appendix~\ref{FPRROCAUC} satisfies
\[
\hat\tau_0(c)-\tau_0(c)
=\sum_{g\in\mathcal G}\mathbb P_{g,n}\psi_{g,0}(c)+o_p(a_n).
\]
If additionally $s_{0,n}(c)\asymp a_n$, then the left-hand side divided by
$s_{0,n}(c)$ converges to $\mathcal N(0,1)$. Let
$\hat\psi_{g,0,r}(c)$ be the cross-fitted plug-in values obtained by the same
substitutions in the displayed TPR plug-in values, and let
$\bar\psi_{g,0}(c)=n_g^{-1}\sum_{r=1}^{n_g}\hat\psi_{g,0,r}(c)$. Under the same
nondegeneracy condition,
\[
\widehat{\operatorname{SE}}_{\mathrm{FPR}}^2(c)
=\sum_{g\in\mathcal G}\frac{1}{n_g^2}\sum_{r=1}^{n_g}
\{\hat\psi_{g,0,r}(c)-\bar\psi_{g,0}(c)\}^2
\]
satisfies
$\widehat{\operatorname{SE}}_{\mathrm{FPR}}(c)/s_{0,n}(c)\to_p1$ and yields
the corresponding pointwise Wald interval.
\end{corollary}

We write $\psi_{g,1}$ for the TPR influence terms in
Theorem~\ref{thm:asymptotic_normality_tpr} and $\psi_{g,0}$ for the FPR
terms defined by these substitutions. Their cross-fitted plug-in versions are
denoted by $\hat\psi_{g,1,r}$ and $\hat\psi_{g,0,r}$, respectively.

The generalized-inverse construction in Appendix~\ref{FPRROCAUC} follows the
standard representation of an ROC curve as a composition of class-conditional
survival functions \citep{hsieh1996roc}. The following theorem establishes
consistency of the projected ROC and trapezoidal AUC estimators.

\begin{theorem}[Consistency of the Projected ROC and AUC Estimators]
\label{thm:roc_auc_consistency}
Suppose Assumptions~\ref{ass:smooth}--\ref{ass:higher_smoothness} and
Assumption~\ref{ass:uniform_curve} hold. Construct
$\widehat{\operatorname{ROC}}^{cf}$ and
$\widehat{\operatorname{AUC}}^{cf}$ as in
Appendix~\ref{FPRROCAUC}, using equal-weight isotonic projection. Then
\[
\sup_{u\in[0,1]}\left|
\widehat{\operatorname{ROC}}^{cf}(u)-\operatorname{ROC}(u)
\right|\to_p0.
\]
If, in addition, the AUC grid
$0=u_{0,n}<\cdots<u_{R_n,n}=1$ has mesh
$\Delta_{u,n}=\max_{1\le r\le R_n}(u_{r,n}-u_{r-1,n})\to0$, then
\[
\left|\widehat{\operatorname{AUC}}^{cf}
-\operatorname{AUC}\right|\to_p0.
\]
\end{theorem}

Theorem~\ref{thm:pointwise_roc} strengthens this consistency result at a fixed
interior false-positive rate. Its slope ratio can be written as
\[
\lambda_u=\frac{(1-p_1)q_1(c_u)}{p_1q_0(c_u)},
\]
where $q_1$ and $q_0$ are defined in part (c) of
Assumption~\ref{ass:joint_density}. One
consistent implementation uses a pilot bandwidth $b$. Define
$\hat a_{1,\ell}^{(-k)}=\hat m_Y^{(-k)}(X_\ell^{\mathrm{T}})$ and
$\hat a_{0,\ell}^{(-k)}=1-\hat m_Y^{(-k)}(X_\ell^{\mathrm{T}})$. The
target-sample kernel estimates are
\[
\hat q_y(t)=\frac{1}{n_{\mathrm{T}}}
\sum_{k=1}^V\sum_{\ell\in\mathcal L_k}
\hat a_{y,\ell}^{(-k)}K_b\{t-\hat m_S^{(-k)}(X_\ell^{\mathrm{T}})\},
\]
and we set
\[
\hat\lambda_u=
\frac{(1-\hat p_1)\hat q_1(\hat c_u)}
{\hat p_1\hat q_0(\hat c_u)}.
\]
Under the score-density condition, this estimator is consistent when
$b\to0$, $n_{\mathrm{T}}b\to\infty$,
$(a_n+\bar r_S)/b=o_p(1)$, and $q_0(c_u)$ is bounded away from zero. These
conditions concern only estimation of the slope ratio and do not change the
ROC point estimator. A Gaussian-process limit is still required for a
simultaneous confidence band, but not for the scalar AUC result in
Theorem~\ref{thm:auc_clt}.

\paragraph{AUC influence functions.}
Let $Z=m_S^\star(X)$, $M=m_Y^\star(X)$, $p_0=1-p_1$,
$A=\operatorname{AUC}$, and $\eta=p_1p_0A$. With continuous $Z$, define
$G_0(z)=\mathbb E_{\mathrm{T}}[(1-M)\Ind\{Z<z\}]$ and
$G_1(z)=\mathbb E_{\mathrm{T}}[M\Ind\{Z>z\}]$. Using the weighted score
densities in part (c) of Assumption~\ref{ass:joint_density}, set
$H_Y(X)=G_0(Z)-G_1(Z)-A(p_0-p_1)$. Recall that
$\Gamma_S(X)=Mq_0(Z)-(1-M)q_1(Z)$. The source-specific influence functions are
\begin{align*}
\varphi_{\mathrm{HL}}^A
&=\frac{w_{\mathrm{HL}}(X)(Y-M)H_Y(X)}{p_1p_0},\\
\varphi_{\mathrm{AL}}^A
&=\frac{w_{\mathrm{AL}}(X)(S-Z)\Gamma_S(X)}{p_1p_0},\\
\varphi_{\mathrm{T}}^A
&=\frac{MG_0(Z)+(1-M)G_1(Z)-2\eta}{p_1p_0}\\
&\quad-\frac{A(p_0-p_1)(M-p_1)}{p_1p_0}.
\end{align*}

To estimate the AUC variance, form the cross-fitted plug-in versions
$\hat\varphi_{g,r}^A$ of the three terms above using the fold-specific
$\hat G_{y,k}$ and $\hat q_{y,k}$ from Appendix~\ref{FPRROCAUC}; in the target
term estimate $\eta$ by
$\hat p_1(1-\hat p_1)\widehat{\operatorname{AUC}}_{\mathrm{os}}^{cf}$.
Let $\tilde\varphi_{g,r}^A=\hat\varphi_{g,r}^A-
n_g^{-1}\sum_{s=1}^{n_g}\hat\varphi_{g,s}^A$.

\begin{corollary}[AUC Standard Error and Confidence Interval]
\label{cor:auc_se}
Under the conditions of Theorem~\ref{thm:auc_clt}, if
$n_g^{-1}\sum_{r=1}^{n_g}(\hat\varphi_{g,r}^A-\varphi_{g,r}^A)^2=o_p(1)$ for each
$g\in\mathcal G$, then
\[
\widehat{\operatorname{SE}}_{\mathrm{AUC}}^2
=\sum_{g\in\mathcal G}\frac{1}{n_g^2}
\sum_{r=1}^{n_g}(\tilde\varphi_{g,r}^A)^2
\]
satisfies
$\widehat{\operatorname{SE}}_{\mathrm{AUC}}/\sigma_{A,n}\to_p1$.
Thus
$\widehat{\operatorname{AUC}}_{\mathrm{os}}^{cf}
\pm z_{1-\alpha/2}\widehat{\operatorname{SE}}_{\mathrm{AUC}}$
is an asymptotically valid confidence interval.
\end{corollary}

\subsection{Oracle Bandwidth Details}
\label{app:oracle_bandwidth}

This subsection gives the exact $h^2$ bias and AL variance coefficients
summarized in Section~\ref{Theory}. Let
$a_1(X)=m_Y^\star(X)$, $a_0(X)=1-m_Y^\star(X)$,
$p_y=\mathbb E_{\mathrm{T}}\{a_y(X)\}$, $Z=m_S^\star(X)$, and
$\sigma_S(X)=\operatorname{Var}_{\mathrm{AL}}(S\mid X)^{1/2}$.
Recall that $\mu_2(K)=\int v^2K(v)\,dv$; write
$R(K)=\int K^2(v)\,dv$, and let
$f_{Z,\mathrm{AL}}$ be the AL density of $Z$. For compactness, define
\[
\xi_y(X)=w_{\mathrm{AL}}(X)a_y(X)\sigma_S(X).
\]
Conditional on the training folds, define
\begin{align*}
B_{2,y,n}(c)
&=-\frac{\mu_2(K)}{2p_yV}\sum_{k=1}^V
\partial_t^2\mu_{k,1}^{(a_y)}(c),\\
V_{\mathrm{AL},y}(c)
&=\frac{R(K)f_{Z,\mathrm{AL}}(c)}{p_y^2}
\mathbb E_{\mathrm{AL}}\!\left\{\xi_y^2(X)\mid Z=c\right\}.
\end{align*}

\begin{prop}[Oracle Bandwidth Benchmark]
\label{prop:oracle_bandwidth}
Under the conditions of Theorem~\ref{thm:asymptotic_normality_tpr} for $y=1$
or Corollary~\ref{cor:pointwise_fpr} for $y=0$, let $R$
denote TPR or FPR and use the matching $B_{2,R,n}=B_{2,y,n}$ and
$V_{\mathrm{AL},R}=V_{\mathrm{AL},y}$. If the variance profile
\[
\Gamma_y(t)=f_{Z,\mathrm{AL}}(t)
\mathbb E_{\mathrm{AL}}\{\xi_y^2(X)\mid Z=t\}
\]
is continuous and positive at the evaluation threshold $c$, then the
kernel-localization component of the conditional bias is
$B_{2,R,n}h^2+o_p(h^2r_{S,2})$ and the AL variance contribution is
$V_{\mathrm{AL},R}/(n_{\mathrm{AL}}h)\{1+o(1)\}$. Retaining only these two
leading localization terms defines the conditional benchmark
\[
\operatorname{AMSE}_{R}^{\mathrm{loc}}(h)
=\frac{V_{\mathrm{AL},R}}{n_{\mathrm{AL}}h}
+B_{2,R,n}^2h^4.
\]
By construction, this benchmark excludes the $h$-independent
$O_p(r_{S,2}^2)$ bias from variation along the score-estimation path,
the HL and target variance components, and all other nuisance remainders; no
equality with the full conditional mean squared error is asserted.
If $B_{2,R,n}\ne0$ and the resulting bandwidth tends to zero, its interior
minimizer is
\[
h_{\mathrm{AMSE},R}
=\left\{\frac{V_{\mathrm{AL},R}}
{4n_{\mathrm{AL}}B_{2,R,n}^2}\right\}^{1/5}.
\]
When the AL variance term dominates the HL and target variance terms, a
deterministic bandwidth satisfying the pointwise theorem makes the leading
localization bias negligible relative to the AL standard deviation when
$n_{\mathrm{AL}}B_{2,R,n}^2h^5\to_p0$, equivalently
$h/h_{\mathrm{AMSE},R}\to_p0$. Together with the remaining conditions of
that theorem or corollary, this controls the localization bias
in centered inference. Because $B_{2,R,n}$ and hence
$h_{\mathrm{AMSE},R}$ depend on the training folds, the latter is a conditional
oracle comparison, not a random bandwidth choice covered by
part (b) of Assumption~\ref{ass:bandwidth}.
\end{prop}

When $|B_{2,R,n}|\asymp_p r_{S,2}$, the oracle rate is
$(n_{\mathrm{AL}}r_{S,2}^2)^{-1/5}$; the familiar
$n_{\mathrm{AL}}^{-1/5}$ rate is recovered only when the bias coefficient
does not shrink. If $n_{\mathrm{AL}}B_{2,R,n}^2$ fails to diverge, the
localization bias is already too small for the AMSE criterion to have an interior minimizer,
and the remaining rate conditions determine how large $h$ may be.

\section{Proofs}
\label{apendsec:proof_results}

For source $g$, write $n_{g,k}$ for the size of its $k$th validation fold
and $\alpha_{g,k}=n_{g,k}/n_g$. The fold-balance convention above gives
\[
\max_{1\le k\le V}|\alpha_{g,k}-V^{-1}|\le n_g^{-1}.
\]
Consequently, for any fixed-$V$ collection $Q_{g,k}=O_p(1)$ uniformly in
$k$,
\[
\sum_{k=1}^V\alpha_{g,k}Q_{g,k}
-\frac1V\sum_{k=1}^VQ_{g,k}=O_p(n_g^{-1}).
\]
All foldwise population quantities and empirical averages for which this
replacement is made below are uniformly $O_p(1)$. For an AL kernel average,
this follows from its $O_p(1)$ conditional mean and its
$O_p\{(n_{\mathrm{AL},k}h)^{-1/2}\}$ centered part. Hence replacing exact
fold weights by $1/V$ contributes $o_p(b_{p,n})$, $o_p(a_n)$, or
$o_p(\rho_n)$ in the corresponding proof. This fold-weight approximation permits
first-order nuisance directions from different sources to be compared
fold by fold even when their fold sizes are not identical.

\subsection{Elementary Probability Tools Used Below}
\label{app:elementary_tools}

We record three elementary calculations so that the later proofs can focus on
the estimator-specific algebra. They are also useful for keeping track of what
cross-fitting does and does not change. In this subsection only, write
$\mathcal F_{\mathrm{HL},k}=\mathcal I_k$,
$\mathcal F_{\mathrm{AL},k}=\mathcal J_k$, and
$\mathcal F_{\mathrm{T},k}=\mathcal L_k$.

\begin{lemma}[Conditional empirical-mean bound]
\label{lem:conditional_empirical_mean}
Let $f_k$ be a possibly random function determined by the training data for
fold $k$, and suppose the observations in the corresponding validation fold
are independent of that training data. Conditional on the training data,
\[
(\mathbb P_{g,\mathcal F_{g,k}}-\mathbb P_g)f_k
=O_p\!\left(
\frac{\|f_k\|_{\mathbb P_g,2}}{\sqrt{n_{g,k}}}
\right).
\]
Consequently, when $V$ is fixed and
$\max_k\|f_k\|_{\mathbb P_g,2}=O_p(r_n)$,
\[
\sum_{k=1}^V\alpha_{g,k}
(\mathbb P_{g,\mathcal F_{g,k}}-\mathbb P_g)f_k
=O_p(r_n/\sqrt{n_g}).
\]
\end{lemma}

\begin{proof}
Conditional on the training data, the validation observations are i.i.d. and
the displayed empirical average has conditional mean zero and conditional
variance at most
$\|f_k\|_{\mathbb P_g,2}^2/n_{g,k}$. Conditional Chebyshev's inequality gives
the first assertion. The second follows by summing the fixed number of
foldwise bounds and using $n_{g,k}\asymp n_g$.
\end{proof}

\begin{lemma}[Ratio expansion]
\label{lem:ratio_expansion}
Let $p$ be bounded away from zero, $\theta=p\tau$, and suppose
$\hat p-p=O_p(r_n)$ and $\hat\theta-\theta=O_p(r_n)$ for $r_n\to0$. Then
\[
\frac{\hat\theta}{\hat p}-\tau
=\frac{\hat\theta-\theta}{p}
-\frac{\tau}{p}(\hat p-p)+o_p(r_n).
\]
\end{lemma}

\begin{proof}
On an event whose probability tends to one, $\hat p$ is bounded away from
zero. The exact difference between the left-hand side and the two displayed
linear terms is
\[
-\frac{(\hat\theta-\theta)(\hat p-p)}{p\hat p}
+\frac{\tau(\hat p-p)^2}{p\hat p}.
\]
Its absolute value is $O_p(r_n^2)=o_p(r_n)$.
\end{proof}

\begin{lemma}[Three-sample triangular-array CLT]
\label{lem:three_sample_clt}
Suppose three independent empirical averages have mean-zero summands. Assume
the HL and target summands have fixed bounded envelopes, while the AL summand
has envelope $O(h^{-1})$ and variance $O(h^{-1})$. Let the variance of their
sum be $s_n^2$ and suppose
\[
s_n\asymp
\{n_{\mathrm{HL}}^{-1}+(n_{\mathrm{AL}}h)^{-1}
+n_{\mathrm{T}}^{-1}\}^{1/2},
\qquad n_{\mathrm{AL}}h\to\infty.
\]
Then the sum divided by $s_n$ converges in distribution to
$\mathcal N(0,1)$.
\end{lemma}

\begin{proof}
After division by $s_n$, the largest possible contribution of one HL or
target observation is bounded by $C/(n_gs_n)\le C/\sqrt{n_g}\to0$. The
corresponding AL bound is
\[
\frac{C}{n_{\mathrm{AL}}hs_n}
\le \frac{C}{\sqrt{n_{\mathrm{AL}}h}}\to0.
\]
Thus no single observation can contribute a nonnegligible amount. Because the
normalized variances sum to one, the Lindeberg condition follows, and the
Lindeberg--Feller theorem gives the result.
\end{proof}

\begin{lemma}[A Conditional-Density Condition for Regularity along the Score-Estimation Path]
\label{lem:primitive_score_path}
Suppose Assumption~\ref{ass:smooth} holds and $\bar r_S=o_p(1)$. For
$a\in\{m_Y^\star,1-m_Y^\star\}$, let
$p_a=\mathbb E_{\mathrm{T}}\{a(X)\}$ and define the tilted target law
$d\mathbb P_a=a\,d\mathbb P_{\mathrm{T}}/p_a$, with expectation
$\mathbb E_a$. Conditional on the fold-$k$
training data, put $V_k=\delta_S^{(-k)}(X)$ and let
$f_{a,k}(z\mid v)$ be the conditional density of
$Z=m_S^\star(X)$ given $V_k=v$ under $\mathbb P_a$.

Let $\mathcal U^+$ be a fixed open neighborhood with
$\overline{\mathcal U}\subset\mathcal U^+$. Suppose that, with probability tending to one,
$f_{a,k}(z\mid v)$ is three times continuously differentiable in $z$ on
$\mathcal U^+$ and, for a constant $C<\infty$,
\[
\max_k\max_a\operatorname*{ess\,sup}_{v}
\sup_{z\in\mathcal U^+}
\sum_{j=0}^3|\partial_z^j f_{a,k}(z\mid v)|\le C.
\]
Then part (b) of Assumption~\ref{ass:higher_smoothness} holds.
\end{lemma}

\begin{proof}
Because $\bar r_S=o_p(1)$, with probability tending to one
$t-uV_k\in\mathcal U^+$ for every $t\in\mathcal U$, $u\in[0,1]$, and fold
$k$. Conditioning on $V_k$ and changing variables from $Z$ to
$Z+uV_k$ give
\[
\mu_{k,u}^{(a)}(t)
=p_a\mathbb E_a\!\left[
V_k f_{a,k}(t-uV_k\mid V_k)\right].
\]
The same conditioning argument gives
\[
F_{k,a}(u)
=p_a\mathbb E_a\!\left[
1-F_{a,k}(c-uV_k\mid V_k)\right],
\]
where $F_{a,k}(\cdot\mid v)$ is the conditional distribution function
corresponding to $f_{a,k}(\cdot\mid v)$. Dominated differentiation therefore
yields $F_{k,a}'(u)=\mu_{k,u}^{(a)}(c)$ and
\begin{align*}
\partial_u\mu_{k,u}^{(a)}(c)
&=-p_a\mathbb E_a\!\left[
V_k^2\partial_z f_{a,k}(c-uV_k\mid V_k)\right],\\
\partial_t^2\mu_{k,1}^{(a)}(t)
&=p_a\mathbb E_a\!\left[
V_k\partial_z^2 f_{a,k}(t-V_k\mid V_k)\right].
\end{align*}
Since $0\le a\le1$ and $\|V_k\|_{\mathbb P_{\mathrm{T}},2}\le r_{S,2}$,
\[
p_a\mathbb E_a|V_k|\le r_{S,2},
\qquad
p_a\mathbb E_a(V_k^2)\le r_{S,2}^2.
\]
The derivative bound on $f_{a,k}$ now gives the two rate bounds in
Assumption~\ref{ass:higher_smoothness}. Finally, the mean-value theorem and
the bound on $\partial_z^3f_{a,k}$ give
\begin{align*}
&\sup_{\substack{s,t\in\mathcal U\\|s-t|\le h}}
|\partial_t^2\mu_{k,1}^{(a)}(s)
-\partial_t^2\mu_{k,1}^{(a)}(t)|\\
&\qquad\le Cp_a h\mathbb E_a|V_k|
=O_p(hr_{S,2})=o_p(r_{S,2}),
\end{align*}
which proves the remaining condition.
\end{proof}

\subsection{Proof of Theorem \ref{thm:p1_consistency} (Target Prevalence Estimator)}
\label{apendsec:proof_p1}

\begin{proof}
For fold $k$, abbreviate
$m_k=\hat m_Y^{(-k)}$ and
$w_k=\hat w_{\mathrm{HL}}^{(-k)}$. When the full-sample averages defining
$\hat p_1$ are decomposed by fold, fold $k$ receives its exact
source-specific sample proportion. To display these weights explicitly, set
\[
\alpha_{\mathrm{T},k}=\frac{|\mathcal L_k|}{n_{\mathrm{T}}},
\qquad
\alpha_{\mathrm{HL},k}=\frac{|\mathcal I_k|}{n_{\mathrm{HL}}},
\]
and abbreviate
\begin{align*}
Q_{\mathrm{T},k}&=\mathbb P_{\mathrm{T},\mathcal L_k}m_k,\\
Q_{\mathrm{HL},k}&=\mathbb P_{\mathrm{HL},\mathcal I_k}
\bigl[w_k\{Y-m_k(X)\}\bigr].
\end{align*}
The estimator in Section~\ref{Method} is therefore exactly
\[
\hat p_1=\sum_{k=1}^V
\{\alpha_{\mathrm{T},k}Q_{\mathrm{T},k}
+\alpha_{\mathrm{HL},k}Q_{\mathrm{HL},k}\}.
\]
For the foldwise cancellation below, define its equal-fold-weight counterpart
\[
\widetilde p_1=\frac1V\sum_{k=1}^V
\{Q_{\mathrm{T},k}+Q_{\mathrm{HL},k}\}.
\]
Balanced folds satisfy
$|\alpha_{g,k}-V^{-1}|\le n_g^{-1}$. Moreover,
$\max_k|Q_{g,k}|=O_p(1)$ under Assumption~\ref{ass:smooth}, the stated
$L_2$ convergence, and fixed $V$. Consequently,
\begin{align*}
\hat p_1-\widetilde p_1
&=\sum_{k=1}^V
\left(\alpha_{\mathrm{T},k}-\frac1V\right)Q_{\mathrm{T},k}\\
&\quad+\sum_{k=1}^V
\left(\alpha_{\mathrm{HL},k}-\frac1V\right)Q_{\mathrm{HL},k}\\
&=O_p(n_{\mathrm{T}}^{-1}+n_{\mathrm{HL}}^{-1})
=o_p(b_{p,n}),
\end{align*}
because $b_{p,n}=(n_{\mathrm{HL}}^{-1}+n_{\mathrm{T}}^{-1})^{1/2}$.
Thus replacing the exact fold proportions by $1/V$ does not affect the
first-order expansion, and it is enough to analyze $\widetilde p_1$.

\paragraph{Step 1: Establish the population identity.}
For deterministic functions $m$ and $w$, recall
\[
\Psi(m,w;1)=\mathbb P_{\mathrm{T}}m
+\mathbb P_{\mathrm{HL}}\bigl[w\{Y-m(X)\}\bigr].
\]
Assumption~\ref{ass:smooth} gives
$\mathbb E_{\mathrm{HL}}(Y\mid X)=m_Y^\star(X)$. Therefore, conditioning on
$X$ in the second term yields
\begin{align*}
\Psi(m,w;1)-p_1
&=\mathbb P_{\mathrm{T}}(m-m_Y^\star)
+\mathbb P_{\mathrm{HL}}\{w(m_Y^\star-m)\}.
\end{align*}
The same assumption gives the change-of-measure identity
$\mathbb P_{\mathrm{T}}f=\mathbb P_{\mathrm{HL}}(w_{\mathrm{HL}}f)$.
Applying it to $f=m-m_Y^\star$ gives
\begin{equation}
\Psi(m,w;1)-p_1
=\mathbb P_{\mathrm{HL}}
\bigl[\{w_{\mathrm{HL}}-w\}(m-m_Y^\star)\bigr].
\label{eq:p1_product_identity_proof}
\end{equation}
The identity expresses the population bias as the product of the
outcome-regression and density-ratio errors. This product form is used below
to control the remainder when both nuisance estimators are consistent.

\paragraph{Step 2: Derive the first-order decomposition.}
Write
\[
\delta_{Y,k}=m_k-m_Y^\star,
\qquad
\delta_{w,k}=w_k-w_{\mathrm{HL}}.
\]
For compactness, define the centered validation-fold operators
\[
\mathbb U_{\mathrm{T},k}=\mathbb P_{\mathrm{T},\mathcal L_k}
-\mathbb P_{\mathrm{T}},
\qquad
\mathbb U_{\mathrm{HL},k}=\mathbb P_{\mathrm{HL},\mathcal I_k}
-\mathbb P_{\mathrm{HL}}.
\]
Substituting $(m_k,w_k)$ into
\eqref{eq:p1_product_identity_proof} shows that the fold-$k$ population bias is
$-\mathbb P_{\mathrm{HL}}(\delta_{w,k}\delta_{Y,k})$.
Expanding $w_k=w_{\mathrm{HL}}+\delta_{w,k}$ and
$m_k=m_Y^\star+\delta_{Y,k}$ gives
\begin{align*}
H_k
&:=w_k(Y-m_k)-w_{\mathrm{HL}}(Y-m_Y^\star)\\
&=\delta_{w,k}(Y-m_Y^\star)\\
&\quad-w_{\mathrm{HL}}\delta_{Y,k}
-\delta_{w,k}\delta_{Y,k}.
\end{align*}
Hence direct addition and subtraction of the two true summands give the exact
decomposition
\begin{align}
\widetilde p_1-p_1
&=\frac1V\sum_{k=1}^V
\left[\mathbb U_{\mathrm{T},k}m_Y^\star
+\mathbb U_{\mathrm{HL},k}
\{w_{\mathrm{HL}}(Y-m_Y^\star)\}\right]
\notag\\
&\quad+R_{p,\mathrm{emp}}+R_{p,\mathrm{pop}},
\label{eq:p1_exact_decomposition}
\end{align}
where
\begin{align*}
R_{p,\mathrm{emp}}
&=\frac1V\sum_{k=1}^V
\{\mathbb U_{\mathrm{T},k}\delta_{Y,k}
+\mathbb U_{\mathrm{HL},k}H_k\}\\
&=\frac1V\sum_{k=1}^V\Bigl[
\mathbb U_{\mathrm{T},k}\delta_{Y,k}
+\mathbb U_{\mathrm{HL},k}
\{\delta_{w,k}(Y-m_Y^\star)\}\Bigr]\\
&\quad-\frac1V\sum_{k=1}^V\mathbb U_{\mathrm{HL},k}
\{w_{\mathrm{HL}}\delta_{Y,k}+\delta_{w,k}\delta_{Y,k}\},\\
R_{p,\mathrm{pop}}
&=-\frac1V\sum_{k=1}^V
\mathbb P_{\mathrm{HL}}(\delta_{w,k}\delta_{Y,k}).
\end{align*}
Thus $R_{p,\mathrm{emp}}$ contains centered validation-sample fluctuations,
whereas $R_{p,\mathrm{pop}}$ is the remaining product bias.

\paragraph{Step 3: Bound the two remainders.}
For each fold, condition on the data used to train its nuisance estimators.
Cross-fitting makes the corresponding validation observations independent of
the fitted functions. Thus, conditional on the training data,
$\mathbb U_{\mathrm{T},k}\delta_{Y,k}$ and
$\mathbb U_{\mathrm{HL},k}H_k$ are centered empirical averages of fixed
functions. Boundedness of $Y$, $m_k$, and $w_{\mathrm{HL}}$ in
Assumption~\ref{ass:smooth} gives
\[
\|H_k\|_{\mathbb P_{\mathrm{HL}},2}
=O_p(r_{Y,2}+r_{w,\mathrm{HL},2});
\]
here the product term is controlled because $\delta_{Y,k}$ is uniformly
bounded. At this first use, Lemma~\ref{lem:conditional_empirical_mean} gives
the two bounds explicitly:
\begin{align*}
\mathbb U_{\mathrm{T},k}\delta_{Y,k}
&=O_p\!\left(
\frac{\|\delta_{Y,k}\|_{\mathbb P_{\mathrm{T}},2}}
{\sqrt{n_{\mathrm{T},k}}}\right)
=O_p\!\left(\frac{r_{Y,2}}{\sqrt{n_{\mathrm{T}}}}\right),\\
\mathbb U_{\mathrm{HL},k}H_k
&=O_p\!\left(
\frac{\|H_k\|_{\mathbb P_{\mathrm{HL}},2}}
{\sqrt{n_{\mathrm{HL},k}}}\right)\\
&=O_p\!\left(
\frac{r_{Y,2}+r_{w,\mathrm{HL},2}}
{\sqrt{n_{\mathrm{HL}}}}\right).
\end{align*}
The last equalities use balanced folds, so $n_{g,k}\asymp n_g$, and the
definitions of the nuisance rates. Averaging over the fixed number $V$ of
folds therefore yields
\[
R_{p,\mathrm{emp}}
=O_p\!\left(
\frac{r_{Y,2}}{\sqrt{n_{\mathrm{T}}}}
+\frac{r_{Y,2}+r_{w,\mathrm{HL},2}}{\sqrt{n_{\mathrm{HL}}}}
\right).
\]
Because both nuisance errors are $o_p(1)$ and
$b_{p,n}$ is at least of the order of both
$n_{\mathrm{T}}^{-1/2}$ and $n_{\mathrm{HL}}^{-1/2}$, this is
$o_p(b_{p,n})$. For the population remainder, Cauchy--Schwarz gives
\[
|R_{p,\mathrm{pop}}|
\le r_{w,\mathrm{HL},2}r_{Y,2}
=o_p(b_{p,n}),
\]
where the last equality is precisely the product-rate condition in the
theorem.

Finally, the fold-balance calculation converts the leading term in
\eqref{eq:p1_exact_decomposition} into the two full-sample empirical
averages, with an additional $o_p(b_{p,n})$ error. Specifically,
$\mathbb P_{\mathrm{T}}m_Y^\star=p_1$ and
$\mathbb P_{\mathrm{HL}}\{w_{\mathrm{HL}}(Y-m_Y^\star)\}=0$, so these
centered averages are exactly the empirical averages of
$\phi_{p_1}^{\mathrm{T}}$ and $\phi_{p_1}^{\mathrm{HL}}$, respectively.
Combining this fact with
$\hat p_1-\widetilde p_1=o_p(b_{p,n})$ gives
\[
\hat p_1-p_1
=\frac1{n_{\mathrm{HL}}}\sum_{i=1}^{n_{\mathrm{HL}}}
\phi_{p_1,i}^{\mathrm{HL}}
+\frac1{n_{\mathrm{T}}}\sum_{\ell=1}^{n_{\mathrm{T}}}
\phi_{p_1,\ell}^{\mathrm{T}}
+o_p(b_{p,n}),
\]
which is the claimed expansion.
\end{proof}

\subsection{Proof of Theorem \ref{thm:asymptotic_normality_tpr}}
\label{apendsec:proof_tpr}

\begin{proof}
Let
\[
\widehat\theta_1(c)=\widehat\theta_{\mathrm{T}}(c)
+\widehat\theta_{\mathrm{HL}}(c)
+\widehat\theta_{\mathrm{AL}}(c),\qquad
\theta_1(c)=\mathbb P_{\mathrm{T}}\{m_Y^\star \Ind_c\}.
\]
We first establish an expansion for the numerator and then account for the
estimated denominator. Replacing the exact source-specific fold weights by
$1/V$ costs $o_p(a_n)$ by the fold-weight calculation at the start of this
section; all cancellations below are therefore stated fold by fold with
common weight $1/V$.

At the true nuisance functions, define the mean-zero source terms
\begin{align*}
\xi_{\mathrm{HL}}(c)
&=w_{\mathrm{HL}}(X)
  \{Y-m_Y^\star(X)\}\Ind_c(X),\\
\xi_{\mathrm{AL},h}(c)
&=w_{\mathrm{AL}}(X)m_Y^\star(X)\\
&\quad\times K_h\{c-m_S^\star(X)\}
  \{S-m_S^\star(X)\},\\
\xi_{\mathrm{T}}(c)
&=m_Y^\star(X)\Ind_c(X)-\theta_1(c).
\end{align*}
We show that
\begin{equation}
\begin{aligned}
\widehat\theta_1(c)-\theta_1(c)
&=\mathbb P_{\mathrm{HL},n}\xi_{\mathrm{HL}}(c)
 +\mathbb P_{\mathrm{AL},n}\xi_{\mathrm{AL},h}(c)\\
&\quad+\mathbb P_{\mathrm{T},n}\xi_{\mathrm{T}}(c)+o_p(a_n).
\end{aligned}
\label{eq:numerator_expansion}
\end{equation}

We first give the target and HL remainder decomposition. Fix a fold and suppress its
index. Write $M=m_Y^\star$, $Z=m_S^\star$,
$\delta_Y=\hat m_Y-M$, $\delta_S=\hat m_S-Z$, and
$\delta_{w,d}=\hat w_d-w_d$. Also set
\[
I=\Ind\{Z\ge c\},
\qquad
\widehat I=\Ind\{Z+\delta_S\ge c\},
\]
and let $\mathbb U_{\mathrm{T}}$ and $\mathbb U_{\mathrm{HL}}$ denote the
centered empirical operators on the two validation folds. Directly expanding
the target and HL numerator terms gives
\begin{align}
&\mathbb P_{\mathrm{T},\mathcal L}\{(M+\delta_Y)\widehat I\}
\notag\\
&\quad+\mathbb P_{\mathrm{HL},\mathcal I}
[(w_{\mathrm{HL}}+\delta_{w,\mathrm{HL}})
\{Y-M-\delta_Y\}\widehat I]-\theta_1(c)
\notag\\
&\quad=\mathbb U_{\mathrm{T}}(MI)
\notag\\
&\qquad+\mathbb U_{\mathrm{HL}}[w_{\mathrm{HL}}(Y-M)I]
\notag\\
&\qquad+B_S(c)-\mathbb P_{\mathrm{HL}}
(\delta_{w,\mathrm{HL}}\delta_Y\widehat I)
\notag\\
&\qquad+R_{\mathrm{T}}+R_{\mathrm{HL}},
\label{eq:tpr_target_hl_exact}
\end{align}
where
\begin{align*}
R_{\mathrm{T}}
&=\mathbb U_{\mathrm{T}}
[\delta_Y\widehat I+M(\widehat I-I)],\\
R_{\mathrm{HL}}
&=\mathbb U_{\mathrm{HL}}
\bigl[w_{\mathrm{HL}}(Y-M)(\widehat I-I)
+\delta_{w,\mathrm{HL}}(Y-M)\widehat I\\
&\hspace{8em}-w_{\mathrm{HL}}\delta_Y\widehat I
-\delta_{w,\mathrm{HL}}\delta_Y\widehat I\bigr].
\end{align*}
To verify \eqref{eq:tpr_target_hl_exact}, first expand the four products in
the HL term. Conditional mean-zero residuals give
\[
\mathbb P_{\mathrm{HL}}
[w_{\mathrm{HL}}(Y-M)\widehat I]
=\mathbb P_{\mathrm{HL}}
[\delta_{w,\mathrm{HL}}(Y-M)\widehat I]=0.
\]
Change of measure then gives the exact cancellation
\[
\mathbb P_{\mathrm{T}}(\delta_Y\widehat I)
-\mathbb P_{\mathrm{HL}}
(w_{\mathrm{HL}}\delta_Y\widehat I)=0,
\]
leaving only the displayed product bias and
$B_S(c)=\mathbb P_{\mathrm{T}}\{M(\widehat I-I)\}$.

The product bias in \eqref{eq:tpr_target_hl_exact} is at most
$r_{w,\mathrm{HL},2}r_{Y,2}$ by Cauchy--Schwarz. Moreover,
$|\widehat I-I|$ can be nonzero only when $|Z-c|\le\bar r_S$. The bounded
target score density therefore gives
\[
\mathbb P_{\mathrm{T}}(|Z-c|\le\bar r_S)=O_p(\bar r_S).
\]
Applying Lemma~\ref{lem:conditional_empirical_mean} to the terms above yields
\begin{align*}
R_{\mathrm{T}}
&=O_p\!\left(
\frac{r_{Y,2}+\bar r_S^{1/2}}{\sqrt{n_{\mathrm{T}}}}
\right)=o_p(n_{\mathrm{T}}^{-1/2}),\\
R_{\mathrm{HL}}
&=O_p\!\left(
\frac{r_{Y,2}+r_{w,\mathrm{HL},2}+\bar r_S^{1/2}}
{\sqrt{n_{\mathrm{HL}}}}
\right)=o_p(n_{\mathrm{HL}}^{-1/2}).
\end{align*}
For the first term in $R_{\mathrm{HL}}$, use
$\mathbb E_{\mathrm{HL}}
[w_{\mathrm{HL}}^2\Ind\{|Z-c|\le\bar r_S\}]
\le C\mathbb P_{\mathrm{T}}(|Z-c|\le\bar r_S)$; the remaining terms follow
from the nuisance $L_2$ rates and boundedness. This accounts explicitly for
every target and HL remainder.

It remains to control estimation of the surrogate-derived model inside the target indicator.
We use the score-estimation path in Assumption \ref{ass:higher_smoothness}, which
avoids replacing a hard indicator evaluated at the fitted score by a smoothed
indicator. Suppress the fold index, put $Z=m_S^\star(X)$ under the target law, and for
$a=m_Y^\star$ define
\[
\nu_u^{(1)}(A)=\mathbb P_{\mathrm{T}}
\left[a(X)\delta_S(X)\Ind\{Z+u\delta_S(X)\in A\}\right],
\]
and let $\mu_u^{(1)}$ be its density on $\mathcal U$. Also define
\[
F(u)=\mathbb P_{\mathrm{T}}\!\left[
a(X)\Ind\{Z+u\delta_S(X)\ge c\}\right].
\]
The signed-density condition permits differentiation along this path and gives
\[
F'(u)=\mu_u^{(1)}(c).
\]
Consequently,
\[
B_S(c)=F(1)-F(0)=\int_0^1\mu_u^{(1)}(c)\,du.
\]

Conditional on the training data, transportability of the AL residual gives
\begin{align*}
&\mathbb P_{\mathrm{AL}}\!\left[
w_{\mathrm{AL}}m_Y^\star K_h(c-m_S^\star-\delta_S)
\{S-m_S^\star-\delta_S\}\right]\\
&\qquad=-\mathbb P_{\mathrm{T}}\!\left[
m_Y^\star K_h(c-Z-\delta_S)\delta_S\right]\\
&\qquad=-\int K_h(c-t)\,d\nu_1^{(1)}(t),
\end{align*}
where $\nu_1^{(1)}$ is the signed measure in
Assumption~\ref{ass:higher_smoothness}. Since its density is required only on
$\mathcal U$, write
\begin{align*}
-\int K_h(c-t)\,d\nu_1^{(1)}(t)
&=-\int_{\mathcal U}K_h(c-t)\mu_1^{(1)}(t)\,dt\\
&\quad+R_{\mathrm{tail}},\\
R_{\mathrm{tail}}&=o_p(h^2r_{S,2}).
\end{align*}
To verify the tail bound, let
$d=\operatorname{dist}(c,\mathcal U^c)>0$. Boundedness of $a$ and
Cauchy--Schwarz give
$\|\nu_1^{(1)}\|_{\mathrm{TV}}\le
\mathbb P_{\mathrm{T}}\{a(X)|\delta_S(X)|\}=O_p(r_{S,2})$.
For a compactly supported kernel, the tail is zero for all sufficiently small
$h$. For an exponentially decreasing kernel,
\begin{align*}
|R_{\mathrm{tail}}|
&\le \|\nu_1^{(1)}\|_{\mathrm{TV}}
\sup_{|t-c|\ge d}|K_h(c-t)|\\
&=O_p\{r_{S,2}h^{-1}\exp(-C_0d/h)\}
=o_p(h^2r_{S,2})
\end{align*}
for some $C_0>0$. Define the combined deterministic score drift by
\[
R_S(c)=B_S(c)+\mathbb P_{\mathrm{AL}}\!\left[
w_{\mathrm{AL}}m_Y^\star K_\delta(\epsilon_S-\delta_S)\right],
\]
where $K_\delta=K_h(c-Z-\delta_S)$ and $\epsilon_S=S-Z$.
The preceding identities give
\begin{align*}
R_S(c)
&=\int_0^1\{\mu_u^{(1)}(c)-\mu_1^{(1)}(c)\}\,du\\
&\quad+\mu_1^{(1)}(c)
-\int_{\mathcal U}K_h(c-t)\mu_1^{(1)}(t)\,dt
+R_{\mathrm{tail}}.
\end{align*}
By the mean-value theorem in $u$ and Assumption
\ref{ass:higher_smoothness},
\begin{align*}
&\left|\int_0^1
\{\mu_u^{(1)}(c)-\mu_1^{(1)}(c)\}\,du\right|\\
&\qquad\le \frac12\sup_{0\le u\le1}|\partial_u\mu_u^{(1)}(c)|
=O_p(r_{S,2}^2).
\end{align*}
Because $K$ is symmetric, its first moment vanishes. A second-order Taylor
expansion and Assumption~\ref{ass:higher_smoothness} therefore give
\begin{align*}
&\mu_1^{(1)}(c)
-\int_{\mathcal U}K_h(c-t)\mu_1^{(1)}(t)\,dt\\
&\qquad=-\frac{h^2\mu_2(K)}{2}
\partial_t^2\mu_1^{(1)}(c)+o_p(h^2r_{S,2}).
\end{align*}
Together with the already separated $R_{\mathrm{tail}}$, this gives
\begin{align*}
R_S(c)
&=O_p(r_{S,2}^2+h^2r_{S,2})+o_p(a_n)
=o_p(a_n),
\end{align*}
by Assumptions~\ref{ass:bandwidth} and~\ref{ass:rate}. This establishes the
required cancellation.

We next connect this population cancellation to the empirical AL correction.
Recall $\epsilon_S=S-Z$ and $K_\delta=K_h(c-Z-\delta_S)$, and set
$K_0=K_h(c-Z)$,
\[
\widehat C=(w_{\mathrm{AL}}+\delta_{w,\mathrm{AL}})
(m_Y^\star+\delta_Y),\qquad
D_C=\widehat C-w_{\mathrm{AL}}m_Y^\star.
\]
Let
\[
\widehat G=\widehat C K_\delta(\epsilon_S-\delta_S),
\qquad
G_0=w_{\mathrm{AL}}m_Y^\star K_0\epsilon_S,
\]
and let $\mathbb U_{\mathrm{AL}}=
\mathbb P_{\mathrm{AL},\mathcal J}-\mathbb P_{\mathrm{AL}}$.
The fitted AL integrand minus its oracle counterpart has the exact
decomposition
\begin{align*}
\widehat G-G_0
&=D_CK_\delta\epsilon_S
+w_{\mathrm{AL}}m_Y^\star(K_\delta-K_0)\epsilon_S\\
&\qquad-\widehat C K_\delta\delta_S.
\end{align*}
Conditional on the training data, the first two terms on the right have
$\mathbb P_{\mathrm{AL}}$-mean zero. Therefore
\begin{align*}
\mathbb P_{\mathrm{AL}}(\widehat G-G_0)
&=-\mathbb P_{\mathrm{AL}}(\widehat C K_\delta\delta_S),\\
R_S(c)-B_S(c)
&=-\mathbb P_{\mathrm{AL}}
(w_{\mathrm{AL}}m_Y^\star K_\delta\delta_S).
\end{align*}
Subtracting these two population terms leaves
$-\mathbb P_{\mathrm{AL}}(D_CK_\delta\delta_S)$, which gives the exact
foldwise identity
\begin{align*}
B_S(c)+\mathbb P_{\mathrm{AL},\mathcal J}\widehat G
&=\mathbb P_{\mathrm{AL},\mathcal J}G_0+R_S(c)\\
&\quad+\mathbb U_{\mathrm{AL}}(\widehat G-G_0)
-\mathbb P_{\mathrm{AL}}(D_CK_\delta\delta_S).
\end{align*}
This display identifies the oracle AL average, the score-drift remainder, the
centered AL remainder, and the remaining nuisance-product bias.

We now bound the last two terms. The bounded auxiliary score density and
$\bar r_S/h=o_p(1)$ imply the shifted-kernel moment bounds
\[
\mathbb P_{\mathrm{AL}}|K_\delta|=O_p(1),
\qquad
\mathbb P_{\mathrm{AL}}K_\delta^2=O_p(h^{-1}).
\]
Also, boundedness of the true functions and uniform nuisance consistency give
$\|D_C\|_\infty=
O_p(\bar r_{w,\mathrm{AL}}+\bar r_Y)$ and
$\|\widehat C\|_\infty=O_p(1)$. Hence
\begin{align*}
\|D_CK_\delta\epsilon_S\|_{\mathbb P_{\mathrm{AL}},2}
&=O_p\!\left\{
\frac{\bar r_{w,\mathrm{AL}}+\bar r_Y}{\sqrt h}\right\},\\
\|\widehat C K_\delta\delta_S\|_{\mathbb P_{\mathrm{AL}},2}
&=O_p(\bar r_S/\sqrt h),\\
\left|\mathbb P_{\mathrm{AL}}(D_CK_\delta\delta_S)\right|
&=O_p\{(\bar r_{w,\mathrm{AL}}+\bar r_Y)\bar r_S\}.
\end{align*}
For the remaining centered term, the mean-value theorem and
$\|K_h'\|_\infty=O(h^{-2})$ give
\[
\|w_{\mathrm{AL}}m_Y^\star(K_\delta-K_0)\epsilon_S
\|_{\mathbb P_{\mathrm{AL}},2}
=O_p(\bar r_S/h^2).
\]
Applying Lemma~\ref{lem:conditional_empirical_mean} to these three
conditional $L_2$ bounds yields
\begin{align*}
\mathbb U_{\mathrm{AL}}(D_CK_\delta\epsilon_S)
&=O_p\!\left\{
\frac{\bar r_{w,\mathrm{AL}}+\bar r_Y}
{\sqrt{n_{\mathrm{AL}}h}}\right\},\\
\mathbb U_{\mathrm{AL}}
\{w_{\mathrm{AL}}m_Y^\star(K_\delta-K_0)\epsilon_S\}
&=O_p\!\left\{\frac{\bar r_S}
{h^2\sqrt{n_{\mathrm{AL}}}}\right\},\\
\mathbb U_{\mathrm{AL}}(\widehat C K_\delta\delta_S)
&=O_p\!\left\{\frac{\bar r_S}
{\sqrt{n_{\mathrm{AL}}h}}\right\}.
\end{align*}
Collecting the preceding bounds, the four numerator remainder groups
and their orders are
\begin{align*}
\text{HL product:}\quad
&O_p(r_{w,\mathrm{HL},2}r_{Y,2}),\\
\text{Target/HL empirical:}\quad
&o_p(n_{\mathrm{T}}^{-1/2}+n_{\mathrm{HL}}^{-1/2}),\\
\text{Score drift:}\quad
&O_p(r_{S,2}^2+h^2r_{S,2}),\\
\text{Remaining AL:}\quad
&O_p\!\left\{(\bar r_{w,\mathrm{AL}}+\bar r_Y)\bar r_S\right\}\\
&\quad+O_p\!\left\{
\frac{\bar r_{w,\mathrm{AL}}+\bar r_Y+\bar r_S}
{\sqrt{n_{\mathrm{AL}}h}}\right\}\\
&\quad+O_p\!\left\{
\frac{\bar r_S}{h^2\sqrt{n_{\mathrm{AL}}}}\right\}.
\end{align*}
The AL term involving
$(\bar r_{w,\mathrm{AL}}+\bar r_Y+\bar r_S)/
\sqrt{n_{\mathrm{AL}}h}$ is $o_p(a_n)$ by the consequence stated after
Assumption~\ref{ass:rate}; every other line is $o_p(a_n)$ by the rate
conditions in that assumption. Apply the target--HL decomposition and the
exact AL identity fold by fold, then average over the fixed number of folds.
The three oracle averages become
$\mathbb P_{\mathrm{HL},n}\xi_{\mathrm{HL}}(c)$,
$\mathbb P_{\mathrm{AL},n}\xi_{\mathrm{AL},h}(c)$, and
$\mathbb P_{\mathrm{T},n}\xi_{\mathrm{T}}(c)$, up to the already controlled
fold-weight error. This proves \eqref{eq:numerator_expansion}.

For the denominator, repeat the decomposition in the proof of
Theorem~\ref{thm:p1_consistency} at the $a_n$ scale. Its centered empirical
and fold-weight remainders are $o_p(b_{p,n})=o_p(a_n)$ because
$b_{p,n}\le a_n$, while Assumption~\ref{ass:rate} gives
$r_{w,\mathrm{HL},2}r_{Y,2}=o_p(a_n)$. Therefore
\begin{equation}
\hat p_1-p_1
=\mathbb P_{\mathrm{HL},n}\phi_{p_1}^{\mathrm{HL}}
+\mathbb P_{\mathrm{T},n}\phi_{p_1}^{\mathrm{T}}+o_p(a_n).
\label{eq:p1_an_expansion}
\end{equation}
The variance of the three leading numerator averages is $O(a_n^2)$, so
\eqref{eq:numerator_expansion} gives
$\widehat\theta_1(c)-\theta_1(c)=O_p(a_n)$. The last display similarly gives
$\hat p_1-p_1=O_p(a_n)$, and $p_1$ is bounded away from zero.
Lemma~\ref{lem:ratio_expansion}, with
$(\hat\theta,\theta,\hat p,p)=(\widehat\theta_1(c),\theta_1(c),\hat p_1,p_1)$,
therefore gives
\begin{align*}
\hat\tau_1(c)-\tau_1(c)
&=\frac{\widehat\theta_1(c)-\theta_1(c)}{p_1}
-\frac{\tau_1(c)}{p_1}(\hat p_1-p_1)\\
&\quad+o_p(a_n).
\end{align*}
The centering in the influence functions can now be checked algebraically.
With $M=m_Y^\star$, $I=\Ind_c$, and $\theta_1(c)=p_1\tau_1(c)$,
\begin{align*}
\frac{\xi_{\mathrm{HL}}(c)-\tau_1(c)\phi_{p_1}^{\mathrm{HL}}}{p_1}
&=\frac{w_{\mathrm{HL}}(Y-M)\{I-\tau_1(c)\}}{p_1},\\
\frac{\xi_{\mathrm{T}}(c)-\tau_1(c)\phi_{p_1}^{\mathrm{T}}}{p_1}
&=\frac{MI-\theta_1(c)-\tau_1(c)(M-p_1)}{p_1}\\
&=\frac{M\{I-\tau_1(c)\}}{p_1},\\
\frac{\xi_{\mathrm{AL},h}(c)}{p_1}
&=\psi_{\mathrm{AL},1,h}(c).
\end{align*}
These are precisely the three influence terms stated in the theorem, and
therefore
\begin{align*}
\hat\tau_1(c)-\tau_1(c)
&=\mathbb P_{\mathrm{HL},n}\psi_{\mathrm{HL},1}(c)
+\mathbb P_{\mathrm{AL},n}\psi_{\mathrm{AL},1,h}(c)\\
&\quad+\mathbb P_{\mathrm{T},n}\psi_{\mathrm{T},1}(c)+o_p(a_n).
\end{align*}
If additionally $s_{1,n}(c)\asymp a_n$, the last remainder is
$o_p\{s_{1,n}(c)\}$.

Finally, the three empirical sums are independent. The HL and target
summands are bounded, while the AL summand has envelope $O(h^{-1})$ and
variance $O(h^{-1})$. Lemma~\ref{lem:three_sample_clt} applies because
$s_{1,n}(c)\asymp a_n$ and $n_{\mathrm{AL}}h\to\infty$. Hence the sum normalized
by $s_{1,n}(c)$ converges to $\mathcal N(0,1)$. No source is required to have
asymptotically negligible variance.
\end{proof}

\subsection{Proof of Theorem \ref{thm:variance_and_ci}}
\label{thm3cor1}

\begin{proof}
The argument has two stages. We first replace the oracle nuisance functions
by their fold-specific estimates while keeping the two global scalars fixed.
We then replace $p_1$ and $\tau_1(c)$ by $\hat p_1$ and $\hat\tau_1(c)$.
For an observation in validation fold $k$, let $\mathcal T_k$ denote the
training sigma-field and define $\widetilde\psi_{g,1,r}(c)$ to be the plug-in
influence value with the fold-specific nuisance estimates unchanged but with
$\hat p_1$ and $\hat\tau_1(c)$ replaced by $p_1$ and $\tau_1(c)$. Write
\[
\widetilde\Delta_{g,1,r}
=\widetilde\psi_{g,1,r}(c)-\psi_{g,1,r}(c).
\]
Set $d_{\mathrm{HL}}=d_{\mathrm{T}}=1$ and $d_{\mathrm{AL}}=h^{-1}$; these
are the natural second-moment orders for the three sources.

\paragraph{Step 1: Control the fold-specific nuisance estimates.}
Put $Z=m_S^\star(X)$, $I_c=\Ind\{Z\ge c\}$, and
$\hat I_{c,k}=\Ind\{\hat m_S^{(-k)}(X)\ge c\}$. Conditional on
$\mathcal T_k$, the indicator can change only when the true score lies near
the threshold:
\begin{align*}
\mathbb P_{\mathrm{T}}(\hat I_{c,k}\ne I_c\mid\mathcal T_k)
&\le \mathbb P_{\mathrm{T}}(|Z-c|\le\bar r_S)\\
&=O_p(\bar r_S)=o_p(1).
\end{align*}
The last equality uses the bounded target-score density in part (c) of
Assumption~\ref{ass:joint_density} and the consistency in part (a) of
Assumption~\ref{ass:rate}. For the HL sample, boundedness of
$w_{\mathrm{HL}}$ additionally
gives, for any event $A$,
\[
\mathbb E_{\mathrm{HL}}\{w_{\mathrm{HL}}^2\Ind(A)\}
\le C\mathbb E_{\mathrm{HL}}\{w_{\mathrm{HL}}\Ind(A)\}
=C\mathbb P_{\mathrm{T}}(A).
\]
Thus the same boundary probability controls the true-weight part of the HL
indicator error. The fitted-weight and outcome-regression errors are
controlled by $r_{w,\mathrm{HL},2}$ and $r_{Y,2}$, respectively. Boundedness
from part (a) of Assumption~\ref{ass:smooth} then gives
\[
\mathbb E_g(\widetilde\Delta_{g,1,r}^2\mid\mathcal T_k)=o_p(1),
\qquad g\in\{\mathrm{HL},\mathrm{T}\}.
\]

For the AL sample, first keep the kernel argument fixed at $Z$. Parts (b) and
(c) of Assumption~\ref{ass:smooth} give
$\mathbb E_{\mathrm{AL}}K_h^2(c-Z)=O(h^{-1})$. Hence, for any uniformly
consistent nuisance error $\delta$ appearing as a bounded multiplicative
factor,
\begin{align*}
&\mathbb E_{\mathrm{AL}}\{\delta^2K_h^2(c-Z)\mid\mathcal T_k\}\\
&\qquad\le \|\delta\|_\infty^2
\mathbb E_{\mathrm{AL}}K_h^2(c-Z)=o_p(h^{-1}).
\end{align*}
This explicitly handles the fitted weight, outcome factor, and score
residual when the kernel location is fixed. It remains to move the kernel
from $Z$ to $Z+\delta_S$. The mean-value theorem gives
\[
K_h(c-Z-\delta_S)-K_h(c-Z)
=-\delta_S\int_0^1K_h'(c-Z-t\delta_S)\,dt.
\]
The kernel-tail and AL score-density conditions in parts (b) and (c) of
Assumption~\ref{ass:smooth}, together with $\bar r_S/h=o_p(1)$ from
part (a) of Assumption~\ref{ass:rate}, imply
$\int_0^1\mathbb E_{\mathrm{AL}}K_h'^2(c-Z-t\delta_S)\,dt=O_p(h^{-3})$.
Jensen's inequality therefore gives
\begin{align*}
&\mathbb E_{\mathrm{AL}}\left[
\{K_h(c-Z-\delta_S)-K_h(c-Z)\}^2\mid\mathcal T_k
\right]\\
&\qquad\le C\bar r_S^2h^{-3}=o_p(h^{-1}),
\end{align*}
where the last equality uses $\bar r_S/h=o_p(1)$. Combining this bound with
the fixed-kernel bounds and the analogous
$O_p(\bar r_S^2h^{-1})=o_p(h^{-1})$ bound for replacing $S-Z$ by
$S-Z-\delta_S$ gives
\[
\mathbb E_{\mathrm{AL}}(\widetilde\Delta_{\mathrm{AL},1,r}^2
\mid\mathcal T_k)=o_p(h^{-1}).
\]

By the sampling condition in part (a) of Assumption~\ref{ass:smooth} and the
cross-fitting construction, validation observations are i.i.d. conditional
on $\mathcal T_k$. Conditional Markov's inequality therefore converts the
preceding conditional-mean bounds into empirical bounds. Summing the fixed
number of folds gives
\[
\frac1{n_g}\sum_{r=1}^{n_g}
\widetilde\Delta_{g,1,r}^2
=o_p(d_g),
\qquad g\in\mathcal G.
\]

\paragraph{Step 2: Restore $\hat p_1$ and $\hat\tau_1(c)$.}
Theorem~\ref{thm:asymptotic_normality_tpr} gives
$\hat\tau_1(c)\to_p\tau_1(c)$. The denominator expansion already established
in its proof gives $\hat p_1-p_1=O_p(a_n)=o_p(1)$; this conclusion does not
require the sharper $b_{p,n}$-scale product rate in
Theorem~\ref{thm:p1_consistency}. Since $p_1\ge\eta$ by part (a) of
Assumption~\ref{ass:smooth},
$\mathbb P(\hat p_1\ge\eta/2)\to1$. On this event, replacing the two global
scalars is a Lipschitz operation. The empirical second moments of the
remaining factors are $O_p(1)$ for HL and target and $O_p(h^{-1})$ for AL.
These orders follow from boundedness and the kernel second-moment bound used
in Step 1. Consequently,
\[
\frac1{n_g}\sum_{r=1}^{n_g}
\{\hat\psi_{g,1,r}(c)-\widetilde\psi_{g,1,r}(c)\}^2
=o_p(d_g).
\]
Combining this display with Step 1 and the empirical triangle inequality,
and defining
$\Delta_{g,1,r}=\hat\psi_{g,1,r}(c)-\psi_{g,1,r}(c)$, yields
\begin{equation}
\frac1{n_g}\sum_{r=1}^{n_g}\Delta_{g,1,r}^2=o_p(d_g),
\qquad g\in\mathcal G.
\label{eq:if_plugin_l2}
\end{equation}

\paragraph{Step 3: Establish the oracle second-moment law of large numbers.}
Each oracle influence function has mean zero: this follows from the
conditional mean-zero HL and AL residuals and from the explicit centering of
the target term. The usual law of large numbers applies to the bounded HL and
target influence values. For AL, the kernel moment calculation gives
\begin{align*}
\mathbb E_{\mathrm{AL}}\{\psi_{\mathrm{AL},1}^2(c)\}&=O(h^{-1}),\\
\mathbb E_{\mathrm{AL}}\{\psi_{\mathrm{AL},1}^4(c)\}&=O(h^{-3}).
\end{align*}
Here parts (b) and (c) of Assumption~\ref{ass:smooth} give the required
kernel and AL score-density bounds; in particular, bounded $K$ and
$\int K^2<\infty$ imply $\int K^4<\infty$. The i.i.d. condition in part (a)
of the same assumption then gives
\[
\Var\!\left\{\frac1{n_{\mathrm{AL}}}
\sum_{j=1}^{n_{\mathrm{AL}}}\psi_{\mathrm{AL},1,j}^2(c)
\right\}=O\{(n_{\mathrm{AL}}h^3)^{-1}\}=o(h^{-2}),
\]
because $n_{\mathrm{AL}}h\to\infty$. Chebyshev's inequality therefore shows
that the AL empirical second moment differs from its expectation by
$o_p(h^{-1})$. Thus, for all three sources,
\begin{equation}
\frac1{n_g}\sum_{r=1}^{n_g}\psi_{g,1,r}^2(c)
=\mathbb E_g\{\psi_{g,1}^2(c)\}+o_p(d_g).
\label{eq:if_oracle_second_moment}
\end{equation}

\paragraph{Step 4: Pass from oracle to plug-in second moments and center them.}
Set
\begin{align*}
D_{g,n}^2&=\frac1{n_g}\sum_{r=1}^{n_g}\Delta_{g,1,r}^2,\\
Q_{g,n}&=\frac1{n_g}\sum_{r=1}^{n_g}\psi_{g,1,r}^2(c),\\
\widehat Q_{g,n}&=\frac1{n_g}\sum_{r=1}^{n_g}\hat\psi_{g,1,r}^2(c).
\end{align*}
The passage from oracle to plug-in second moments follows from the elementary
inequality
\begin{align*}
&\left|\frac1{n_g}\sum_{r=1}^{n_g}
\{\hat\psi_{g,1,r}^2(c)-\psi_{g,1,r}^2(c)\}\right|\\
&\quad\le D_{g,n}\bigl(\widehat Q_{g,n}^{1/2}+Q_{g,n}^{1/2}\bigr).
\end{align*}
Moreover, the empirical triangle inequality gives
$\widehat Q_{g,n}^{1/2}\le Q_{g,n}^{1/2}+D_{g,n}$, so the right-hand side is
at most $2D_{g,n}Q_{g,n}^{1/2}+D_{g,n}^2$.
Equations~\eqref{eq:if_plugin_l2}--\eqref{eq:if_oracle_second_moment}
give $D_{g,n}=o_p(d_g^{1/2})$ and $Q_{g,n}=O_p(d_g)$, so the bound is
$o_p(d_g)$.

It remains to subtract the source-specific empirical mean. Let
$\bar\psi_{g,n}^{\circ}=n_g^{-1}\sum_r\psi_{g,1,r}(c)$ denote the oracle
mean. Since the oracle influence function is mean zero and has variance
$O(d_g)$,
\[
\bar\psi_{g,n}^{\circ}
=O_p\{(d_g/n_g)^{1/2}\},
\qquad
|\bar\psi_{g,1}(c)-\bar\psi_{g,n}^{\circ}|\le D_{g,n}.
\]
Hence $\bar\psi_{g,1}^2(c)=o_p(d_g)$, and centering changes the variance
estimator for source $g$ by only $o_p(d_g/n_g)$. We have proved the
source-specific expansion
\begin{align*}
\frac1{n_g^2}\sum_{r=1}^{n_g}
\{\hat\psi_{g,1,r}(c)-\bar\psi_{g,1}(c)\}^2
&=\frac{\Var_g\{\psi_{g,1}(c)\}}{n_g}\\
&\quad+o_p(d_g/n_g).
\end{align*}

\paragraph{Step 5: Combine the three sources.}
Because there are only three sources and
\[
\sum_{g\in\mathcal G}\frac{d_g}{n_g}
=\frac1{n_{\mathrm{HL}}}+\frac1{n_{\mathrm{AL}}h}
+\frac1{n_{\mathrm{T}}}=a_n^2,
\]
the source-specific remainders sum to $o_p(a_n^2)$. The nondegeneracy
condition $s_{1,n}(c)\asymp a_n$ converts this into
$o_p\{s_{1,n}^2(c)\}$. Therefore
\begin{align*}
&\sum_{g\in\mathcal G}\frac{1}{n_g^2}\sum_{r=1}^{n_g}
\{\hat\psi_{g,1,r}(c)-\bar\psi_{g,1}(c)\}^2\\
&\quad=\sum_{g\in\mathcal G}
\frac{\operatorname{Var}_g\{\psi_{g,1}(c)\}}{n_g}
+o_p\{s_{1,n}^2(c)\}\\
&\quad=s_{1,n}^2(c)\{1+o_p(1)\}.
\end{align*}
Taking square roots gives
$\widehat{\mathrm{SE}}_{\mathrm{TPR}}(c)/s_{1,n}(c)\to_p1$. The Wald interval
then follows from Theorem~\ref{thm:asymptotic_normality_tpr} and Slutsky's
theorem. The regularity condition in part (b) of
Assumption~\ref{ass:higher_smoothness} enters here through that theorem; the
second-moment calculations above use only the basic kernel, density, and
nuisance-consistency conditions.
\end{proof}

\subsection{Proof of Corollary \ref{cor:pointwise_fpr}}

\begin{proof}
Write $M=m_Y^\star$, and, within fold $k$, abbreviate
$\widehat M=\hat m_Y^{(-k)}$ and $\delta_Y=\widehat M-M$. Introduce the
negative-class outcome, regression, fitted regression, and regression error:
\begin{align*}
Y_0&=1-Y, & M_0&=1-M,\\
\widehat M_0&=1-\widehat M,
&\delta_{Y,0}&=\widehat M_0-M_0=-\delta_Y.
\end{align*}
The oracle and fitted negative-class residuals are, respectively,
\begin{align*}
Y_0-M_0&=M-Y,\\
Y_0-\widehat M_0&=\widehat M-Y=(M-Y)+\delta_Y.
\end{align*}
Thus the FPR numerator is exactly the positive-class numerator construction
applied to $(Y_0,M_0,\widehat M_0)$, while the score and threshold indicator
are unchanged.

We first verify that every remainder bound used for the TPR numerator is
preserved by this substitution. In the target--HL decomposition
\eqref{eq:tpr_target_hl_exact}, replace
\begin{align*}
M&\mapsto M_0=1-M,\\
Y-M&\mapsto M-Y,\\
\delta_Y&\mapsto\delta_{Y,0}=-\delta_Y.
\end{align*}
The population product bias consequently changes from
$-\mathbb P_{\mathrm{HL}}(\delta_{w,\mathrm{HL}}\delta_Y\widehat I)$
to
$+\mathbb P_{\mathrm{HL}}(\delta_{w,\mathrm{HL}}\delta_Y\widehat I)$,
but its absolute value is still bounded by
$r_{w,\mathrm{HL},2}r_{Y,2}$. The target and HL centered empirical
remainders have the same $L_2$ bounds because
$\|\delta_{Y,0}\|_2=\|\delta_Y\|_2$ and all remaining factors are unchanged.
These steps use the sampling, transportability, change-of-measure, and
boundedness conditions in part (a), and the local density condition in part
(c), of Assumption~\ref{ass:smooth}.

For the AL term, the fitted negative-class coefficient satisfies
\begin{align*}
&\hat w_{\mathrm{AL}}^{(-k)}(1-\widehat M)
-w_{\mathrm{AL}}(1-M)\\
&\quad=\delta_{w,\mathrm{AL}}(1-M)
-w_{\mathrm{AL}}\delta_Y
-\delta_{w,\mathrm{AL}}\delta_Y.
\end{align*}
Hence its absolute error has exactly the same rate bounds as the positive-class
coefficient. The fitted score residual remains
$S-m_S^\star-\delta_S$. The score-drift cancellation now uses the
$a=1-m_Y^\star$ branch of part (b) of Assumption~\ref{ass:rate}; the kernel
and shifted-kernel bounds use parts (b) and (c) of
Assumption~\ref{ass:smooth}. Part (a) of Assumption~\ref{ass:rate} then makes
the negative-class product, empirical, score-drift, and remaining AL
remainders all $o_p(a_n)$, exactly as in the TPR proof.

Let $\widehat\nu(c)$ be the FPR numerator, put
$\nu(c)=p_0\tau_0(c)$ and $I=\Ind_c$, and define the mean-zero oracle terms
\begin{align*}
\xi_{\mathrm{HL},0}(c)
&=w_{\mathrm{HL}}(M-Y)I,\\
\xi_{\mathrm{AL},0,h}(c)
&=w_{\mathrm{AL}}(1-M)
K_h(c-m_S^\star)(S-m_S^\star),\\
\xi_{\mathrm{T},0}(c)
&=(1-M)I-\nu(c).
\end{align*}
The preceding bounds and the fixed-fold weight approximation therefore give
\begin{align}
\widehat\nu(c)-\nu(c)
&=\mathbb P_{\mathrm{HL},n}\xi_{\mathrm{HL},0}(c)
+\mathbb P_{\mathrm{AL},n}\xi_{\mathrm{AL},0,h}(c)\notag\\
&\quad+\mathbb P_{\mathrm{T},n}\xi_{\mathrm{T},0}(c)+o_p(a_n).
\label{eq:fpr_numerator_expansion}
\end{align}

For the denominator, use the $a_n$-scale prevalence expansion
\eqref{eq:p1_an_expansion}, rather than the sharper standalone
$b_{p,n}$-scale conclusion. Since $\hat p_0=1-\hat p_1$,
\begin{equation}
\hat p_0-p_0
=-\mathbb P_{\mathrm{HL},n}\phi_{p_1}^{\mathrm{HL}}
-\mathbb P_{\mathrm{T},n}\phi_{p_1}^{\mathrm{T}}+o_p(a_n).
\label{eq:p0_an_expansion}
\end{equation}
The variances of the three leading terms in
\eqref{eq:fpr_numerator_expansion} are bounded by constant multiples of
$n_{\mathrm{HL}}^{-1}$, $(n_{\mathrm{AL}}h)^{-1}$, and
$n_{\mathrm{T}}^{-1}$, respectively, so
$\widehat\nu(c)-\nu(c)=O_p(a_n)$. Equation~\eqref{eq:p0_an_expansion}
similarly gives $\hat p_0-p_0=O_p(a_n)$. Finally,
$p_0\ge\eta>0$ by part (a) of Assumption~\ref{ass:smooth}. Thus all
conditions of Lemma~\ref{lem:ratio_expansion} hold, and
\begin{align*}
\hat\tau_0(c)-\tau_0(c)
&=\frac{\widehat\nu(c)-\nu(c)}{p_0}
+\frac{\tau_0(c)}{p_0}(\hat p_1-p_1)\\
&\quad+o_p(a_n).
\end{align*}

It remains to identify the centered influence terms. Using
$\phi_{p_1}^{\mathrm{HL}}=w_{\mathrm{HL}}(Y-M)$,
$\phi_{p_1}^{\mathrm{T}}=M-p_1$,
$\nu(c)=p_0\tau_0(c)$, and $p_0+p_1=1$, we obtain
\begin{align*}
&\frac{\xi_{\mathrm{HL},0}(c)
+\tau_0(c)\phi_{p_1}^{\mathrm{HL}}}{p_0}
=\frac{w_{\mathrm{HL}}(M-Y)\{I-\tau_0(c)\}}{p_0},\\
&\frac{\xi_{\mathrm{T},0}(c)
+\tau_0(c)\phi_{p_1}^{\mathrm{T}}}{p_0}
=\frac{(1-M)\{I-\tau_0(c)\}}{p_0},\\
&\frac{\xi_{\mathrm{AL},0,h}(c)}{p_0}
=\psi_{\mathrm{AL},0,h}(c).
\end{align*}
Here and in the corollary statement, the bandwidth subscript is suppressed
when the AL term is included in the source-indexed notation
$\psi_{g,0}(c)$. Substitution into the ratio expansion proves
\[
\hat\tau_0(c)-\tau_0(c)
=\sum_{g\in\mathcal G}\mathbb P_{g,n}\psi_{g,0}(c)+o_p(a_n).
\]

We next verify the central limit theorem rather than relying only on analogy.
The three influence terms are mean zero by conditional mean-zero residuals,
target centering, and change of measure. Part (a) of
Assumption~\ref{ass:smooth} gives independence of the three source averages
and fixed bounded envelopes for the HL and target terms. Parts (b) and (c)
give an $O(h^{-1})$ envelope and $O(h^{-1})$ variance for the AL term, as well
as $n_{\mathrm{AL}}h\to\infty$. Therefore Lemma~\ref{lem:three_sample_clt}
applies. Under $s_{0,n}(c)\asymp a_n$, it yields
\[
\frac{\hat\tau_0(c)-\tau_0(c)}{s_{0,n}(c)}
\Rightarrow\mathcal N(0,1).
\]

Finally, we make the plug-in variance calculation explicit. Define
$\widetilde\psi_{g,0,r}(c)$ by retaining all fold-specific fitted nuisance
functions in $\hat\psi_{g,0,r}(c)$ but replacing the global scalars
$\hat p_0$ and $\hat\tau_0(c)$ by $p_0$ and $\tau_0(c)$. Because
$\delta_{Y,0}=-\delta_Y$, the source-specific negative-class calculations leading
to \eqref{eq:if_plugin_l2} have the same orders:
\[
\frac1{n_g}\sum_{r=1}^{n_g}
\{\widetilde\psi_{g,0,r}(c)-\psi_{g,0,r}(c)\}^2=o_p(d_g),
\]
where $d_{\mathrm{HL}}=d_{\mathrm{T}}=1$ and
$d_{\mathrm{AL}}=h^{-1}$. Here the HL and target bounds use parts (a) and (c)
of Assumption~\ref{ass:smooth}. The AL kernel-shift bound additionally uses
the kernel smoothness, tail, and bounded AL score-density conditions in
parts (b) and (c) of that assumption, together with
$\bar r_S/h=o_p(1)$ and the nuisance rates in part (a) of
Assumption~\ref{ass:rate}.

Equations~\eqref{eq:p0_an_expansion} and the FPR expansion give
$\hat p_0\to_p p_0$ and $\hat\tau_0(c)\to_p\tau_0(c)$. Restoring these global
scalars therefore preserves the same empirical $L_2$ orders. The oracle HL
and target second moments satisfy the ordinary law of large numbers. For the
AL term, the fourth moment is $O(h^{-3})$; together with
$n_{\mathrm{AL}}h\to\infty$, this gives the corresponding negative-class version of
\eqref{eq:if_oracle_second_moment}. Repeating the elementary second-moment
comparison following those two equations yields
\begin{align*}
&\sum_{g\in\mathcal G}\frac1{n_g^2}\sum_{r=1}^{n_g}
\{\hat\psi_{g,0,r}(c)-\bar\psi_{g,0}(c)\}^2\\
&\quad=s_{0,n}^2(c)+o_p\{s_{0,n}^2(c)\}.
\end{align*}
Consequently,
$\widehat{\operatorname{SE}}_{\mathrm{FPR}}(c)/s_{0,n}(c)\to_p1$, and the
pointwise Wald interval follows from the preceding Gaussian limit and
Slutsky's theorem.
\end{proof}

\subsection{Proof of Theorem \ref{thm:roc_auc_consistency}}
\label{proof:roc_auc_consistency}

\begin{proof}
Write $\hat\tau_1^{\mathrm{dir}}(c)$ and
$\hat\tau_0^{\mathrm{dir}}(c)$ for the raw estimators
$\hat\tau_1(c)$ and $\hat\tau_0(c)$ before projection.
By the definition in part (a) of Assumption~\ref{ass:uniform_curve}, for
$y\in\{0,1\}$ and $c\in\mathcal C_n$,
\[
\hat\tau_y^{\mathrm{dir}}(c)-\tau_y(c)
=\sum_{g\in\mathcal G}\mathbb P_{g,n}\psi_{g,y}(c)+R_{y,n}(c).
\]
The HL and target oracle terms are bounded empirical averages. The AL oracle
term is a kernel average with variance of order $h^{-1}$ and envelope of
order $h^{-1}$. Bernstein's inequality and a union bound over the grid give
\begin{align*}
&\max_{c\in\mathcal C_n}\sum_{y=0}^1
|\hat\tau_y^{\mathrm{dir}}(c)-\tau_y(c)|\\
&\quad=O_p\!\left\{
\sqrt{\frac{\log(M_n+1)}{n_{\mathrm{HL}}}}
+\sqrt{\frac{\log(M_n+1)}{n_{\mathrm{AL}}h}}\right.\\
&\hspace{8em}\left.
+\sqrt{\frac{\log(M_n+1)}{n_{\mathrm{T}}}}
\right\}+o_p(1)\\
&\quad=o_p(1),
\end{align*}
where the final $o_p(1)$ is precisely
$\max_{y\in\{0,1\}}\max_{c\in\mathcal C_n}|R_{y,n}(c)|$ from
part (a) of Assumption~\ref{ass:uniform_curve}. Cross-fitting permits the
empirical-process argument to be applied conditionally fold by fold, and fixed
$V$ does not change the order.
Bernstein's inequality also produces the smaller linear terms
$\log(M_n+1)/n_{\mathrm{HL}}$,
$\log(M_n+1)/(n_{\mathrm{AL}}h)$, and
$\log(M_n+1)/n_{\mathrm{T}}$. Each is dominated by its displayed square-root
counterpart because the corresponding ratio tends to zero.

At the two points in $\mathcal C_n^+\setminus\mathcal C_n$, the fitted and
population endpoint values are both fixed at one or zero. Thus the same
$o_p(1)$ gridwise bound holds for the augmented TPR and FPR vectors on
$\mathcal C_n^+$.

Let $\Pi_0$ denote equal-weight least-squares projection onto the cone of
nonincreasing sequences. For completeness, it has the max--min representation
\[
(\Pi_0z)_j
=\min_{0\le a\le j}\ \max_{j\le b\le M_n}
\frac1{b-a+1}\sum_{i=a}^bz_i.
\]
This representation shows directly that isotonic projection is coordinatewise
order preserving and commutes with constant shifts. Let $\mathbf 1$ denote
the vector of ones. Hence, for
any nonincreasing vector $v$ and any vector
$z$, if $\|z-v\|_\infty\le\delta$, then
$v-\delta\mathbf 1\le z\le v+\delta\mathbf 1$ and therefore
$v-\delta\mathbf 1\le\Pi_0z\le v+\delta\mathbf 1$. Thus
$\|\Pi_0z-v\|_\infty\le\delta$. Clipping $\Pi_0z$ to $[0,1]$, which yields
the constrained projection used here, cannot increase its distance from the
true vector $v\in[0,1]^{M_n+1}$. Applying this property to the true augmented
TPR and FPR vectors shows that projection does not enlarge their gridwise
sup-norm errors. On $\mathcal C$, linear interpolation uses the grid
$\mathcal C_n$, whose endpoints are $\underline c$ and $\overline c$.
Continuity then yields
\[
\|\tilde\tau_1-\tau_1\|_{\infty,\mathcal C}
+\|\tilde\tau_0-\tau_0\|_{\infty,\mathcal C}
=o_p(1),
\]
because the additional deterministic interpolation errors are bounded by the
moduli of continuity of $\tau_1$ and $\tau_0$ at $\Delta_{c,n}$.

The artificial endpoints also remain fixed after the constrained projection.
Indeed, their raw values are $1$ and $0$. If a feasible minimizer had
$v_0<1$, replacing only $v_0$ by $1$ would preserve the nonincreasing
constraints and strictly reduce the squared-error objective. Similarly, if
$v_{M_n}>0$, replacing only $v_{M_n}$ by $0$ would preserve feasibility and
reduce the objective. Thus each projected curve is continuous after
interpolation, starts at one, and ends at zero.

Let $c_u=\inf\{c\in\mathcal C:\tau_0(c)\le u\}$. Continuity and strict
monotonicity of $\tau_0$ on $\mathcal C$ imply that $u\mapsto c_u$ is
continuous. Fix $\epsilon\in(0,1/2)$ and put
$\delta_n=\|\tilde\tau_0-\tau_0\|_{\infty,\mathcal C}=o_p(1)$.
With probability tending to one, the fitted inverse $\hat c_u$ lies in
$\mathcal C$ for every $u\in[\epsilon,1-\epsilon]$: the fitted FPR at
$\underline c$ is larger than $1-\epsilon$, whereas its value at
$\overline c$ is smaller than $\epsilon$. On this event, monotonicity and the
definition of the generalized inverse give
\[
c_{u+\delta_n}
\le \hat c_u
\le c_{u-\delta_n},
\qquad \epsilon\le u\le1-\epsilon,
\]
whenever $\delta_n<\epsilon$. For example,
$\tilde\tau_0(c)\le u$ implies
$\tau_0(c)\le u+\delta_n$, which proves the left inequality; evaluating the fitted
curve at $c_{u-\delta_n}$ proves the right one. Uniform continuity of
$u\mapsto c_u$ on $[\epsilon/2,1-\epsilon/2]$ now yields
$\sup_{u\in[\epsilon,1-\epsilon]}|\hat c_u-c_u|=o_p(1)$.
Consequently,
\begin{align*}
\sup_{u\in[\epsilon,1-\epsilon]}\left|
\tilde\tau_1(\hat c_u)-\tau_1(c_u)\right|
&\le \|\tilde\tau_1-\tau_1\|_{\infty,\mathcal C}\\
&\quad+\sup_{u\in[\epsilon,1-\epsilon]}
|\tau_1(\hat c_u)-\tau_1(c_u)|,
\end{align*}
and the right-hand side is $o_p(1)$. This proves uniform consistency of the
ROC estimator on every fixed interior interval.

It remains to include $u=0,1$ without applying the inverse-continuity
argument to the artificial thresholds. Both the true and estimated ROC curves
are nondecreasing, take values in $[0,1]$, and are anchored at
$\operatorname{ROC}(0)=\widehat{\operatorname{ROC}}^{cf}(0)=0$ and
$\operatorname{ROC}(1)=\widehat{\operatorname{ROC}}^{cf}(1)=1$. Write
$R=\operatorname{ROC}$ and $\widehat R=\widehat{\operatorname{ROC}}^{cf}$.
Hence
\begin{align*}
\sup_{0\le u\le\epsilon}
|\widehat R(u)-R(u)|
&\le \widehat R(\epsilon)+R(\epsilon),\\
\sup_{1-\epsilon\le u\le1}
|\widehat R(u)-R(u)|
&\le 2-\widehat R(1-\epsilon)-R(1-\epsilon).
\end{align*}
For fixed $\epsilon$, the preceding interior result controls the estimated
values on the right. Letting $\epsilon\downarrow0$ and using continuity of
the true ROC at zero and one proves uniform consistency on $[0,1]$.

Under the additional AUC-grid mesh condition in
Theorem~\ref{thm:roc_auc_consistency}, the trapezoidal
estimator equals the integral of the piecewise-linear
interpolant of the estimated ROC values on the $u$-grid. Its error is bounded
by
\[
\sup_{u\in[0,1]}|\widehat{\operatorname{ROC}}^{cf}(u)
-\operatorname{ROC}(u)|
+\omega_{\operatorname{ROC}}(\Delta_{u,n}),
\]
where $\omega_{\operatorname{ROC}}$ is the modulus of continuity of the true
ROC curve. The first term is $o_p(1)$ by the preceding argument, and the
second tends to zero because the ROC curve is continuous and
$\Delta_{u,n}\to0$. This proves AUC consistency.
\end{proof}

\subsection{Auxiliary Lemmas for Pointwise ROC Inference}

\begin{lemma}[Asymptotic Equivalence of Direct and Interpolated Estimators]
\label{lem:grid_interpolation}
For $y\in\{0,1\}$, let $\hat\tau_y^{\mathrm{dir}}(c)$ denote the direct
cross-fitted TPR or FPR estimator evaluated at the fixed threshold $c$, and
let $\hat\tau_y(c)$ be the linear interpolant of its values on
$\mathcal C_n$. Under Assumptions~\ref{ass:smooth}--\ref{ass:roc_local},
for the fixed population threshold $c_u$,
\[
\max_{y\in\{0,1\}}
\left|\hat\tau_y(c_u)-\hat\tau_y^{\mathrm{dir}}(c_u)\right|
=o_p(a_n).
\]
Consequently, the fixed-threshold expansions in
Theorem~\ref{thm:asymptotic_normality_tpr} and
Corollary~\ref{cor:pointwise_fpr} also hold for the interpolated raw
estimators at $c_u$.
\end{lemma}

\begin{proof}
Let $c_{L,n}\le c_u\le c_{R,n}$ be the two adjacent grid points and let
$\mathcal L_n$ denote linear interpolation between them: for a
threshold-indexed function $f$,
\begin{align*}
(\mathcal L_nf)(c_u)
&=\omega_nf(c_{L,n})+(1-\omega_n)f(c_{R,n}),\\
\omega_n&=\frac{c_{R,n}-c_u}{c_{R,n}-c_{L,n}}.
\end{align*}
If $c_u$ itself is a grid point, the interpolation difference is zero and
the result is immediate; below we consider the nontrivial case
$c_{L,n}<c_u<c_{R,n}$.
Put
$\Delta_n=c_{R,n}-c_{L,n}\le\Delta_{c,n}$ and
$q_n=\Delta_{c,n}+\bar r_S=o_p(a_n)$. The denominator of each direct
estimator does not depend on $c$ and is bounded away from zero with
probability tending to one, so it suffices to compare the three numerator
terms.

The three interpolation differences have the following orders:
\begin{align*}
\mathrm{T}:\quad
&O_p\{q_n+(q_n/n_{\mathrm{T}})^{1/2}\},\\
\mathrm{HL}:\quad
&O_p\{\bar r_Yq_n+r_{w,\mathrm{HL},2}r_{Y,2}
+(q_n/n_{\mathrm{HL}})^{1/2}\}\\
&\quad+O_p\{r_{w,\mathrm{HL},2}/\sqrt{n_{\mathrm{HL}}}\},\\
\mathrm{AL}:\quad
&O_p\{\Delta_{c,n}/\sqrt{n_{\mathrm{AL}}h^3}
+\bar r_S\Delta_{c,n}/h\}.
\end{align*}
Each row is $o_p(a_n)$ under part (b) of Assumption~\ref{ass:roc_local} and the pointwise
rate conditions. We now derive them one source at a time.

For any score value $z$, define
$d_n(z)=(\mathcal L_n[c\mapsto\Ind\{z\ge c\}])(c_u)
-\Ind\{z\ge c_u\}$. Then
\[
|d_n(z)|\le \Ind\{|z-c_u|\le\Delta_{c,n}\}.
\]
With $z=\hat m_S^{(-k)}(X)$, the event on the right is contained in
$\{|m_S^\star(X)-c_u|\le q_n\}$. The locally bounded target score density
therefore makes its target probability $O_p(q_n)$. For the target numerator,
boundedness and conditional variance calculation give interpolation error
\[
O_p\!\left(q_n+\sqrt{q_n/n_{\mathrm{T}}}\right)=o_p(a_n).
\]
For the HL numerator, decompose the fitted weight and residual around
$w_{\mathrm{HL}}$ and $Y-m_Y^\star$. The true-weight residual part has
conditional mean zero, and change of measure bounds its conditional variance
by $O_p(q_n/n_{\mathrm{HL}})$. The conditional mean involving
$\delta_Y$ is $O_p(\bar r_Yq_n)$, while the fitted-weight terms are bounded
by $O_p(r_{w,\mathrm{HL},2}r_{Y,2})$ in mean and
$O_p(r_{w,\mathrm{HL},2}/\sqrt{n_{\mathrm{HL}}})$ after centering. Hence the
HL interpolation error is also $o_p(a_n)$. The same argument applies to the
negative-class numerator, with only the residual sign changed.

It remains to control the AL kernel term. Conditional on fold-$k$ training,
write $\mathcal T_k$ for the corresponding training sigma-field and
$\hat Z=\hat m_S^{(-k)}(X)$, and define
\[
D_{k,h}(X)=
(\mathcal L_n[c\mapsto K_h(c-\hat Z)])(c_u)-K_h(c_u-\hat Z).
\]
The mean-value theorem, the bounded AL density of $Z=m_S^\star(X)$, the
kernel-tail conditions, and $\bar r_S/h=o_p(1)$ imply
\begin{align*}
\mathbb E_{\mathrm{AL}}\{|D_{k,h}(X)|\mid\mathcal T_k\}
&=O_p(\Delta_{c,n}/h),\\
\mathbb E_{\mathrm{AL}}\{D_{k,h}^2(X)\mid\mathcal T_k\}
&=O_p(\Delta_{c,n}^2/h^3).
\end{align*}
Indeed, these follow by integrating the first derivative of the translated
kernel; $\|K_h'\|_1=O(h^{-1})$ and
$\|K_h'\|_2^2=O(h^{-3})$. Split the fitted score residual as
$S-\hat Z=(S-Z)-\delta_S^{(-k)}(X)$. The first component has conditional
mean zero, whereas the second is uniformly bounded by $\bar r_S$. Because
the remaining fitted factors are uniformly bounded with probability tending
to one, the AL interpolation error is
\[
O_p\!\left\{
\frac{\Delta_{c,n}}{\sqrt{n_{\mathrm{AL}}h^3}}
+\frac{\bar r_S\Delta_{c,n}}{h}
\right\}=o_p(a_n).
\]
For the first term, use
$a_n\ge(n_{\mathrm{AL}}h)^{-1/2}$ and
$\Delta_{c,n}/h=o(1)$; for the second, use
$\Delta_{c,n}=o(a_n)$ and $\bar r_S/h=o_p(1)$.
Combining the three sources and the fixed number of folds proves the first
claim. The consequence follows by applying the two fixed-threshold
expansions at the nonrandom threshold $c_u$ and then using the proved
$o_p(a_n)$ difference.
\end{proof}

\begin{lemma}[Local Asymptotic Equivalence after Isotonic Projection]
\label{lem:isotonic_local}
Let $f$ be nonincreasing on the full grid and continuously differentiable and
strictly decreasing on a neighborhood $\mathcal U$ of an interior point $c_0$,
with $\inf_{t\in\mathcal U}|f'(t)|>0$. Let $z_n$ be a linearly interpolated raw estimator on an ordered
quasi-uniform grid in a fixed compact interval, and let $\tilde z_n$ be its equal-weight nonincreasing isotonic
projection, followed by clipping to $[0,1]$. Suppose the grid mesh is
$o(a_n)$ and, for deterministic $a_n=O(b_n)\to0$,
\begin{align*}
\sup_{t\in\mathcal U}|z_n(t)-f(t)|&=O_p(b_n),\\
\sup_{\substack{s,t\in\mathcal U\\|s-t|\le Mb_n}}
|e_n(s)-e_n(t)|&=o_p(a_n)
\end{align*}
for every fixed $M$, where $e_n=z_n-f$. If $z_n$ is uniformly consistent on
the full grid and $f(c_0)\in(0,1)$, then, for every fixed $M$,
\[
\begin{aligned}
&\sup_{|t-c_0|\le Ma_n}
\Bigl|\{\tilde z_n(t)-f(t)\}\\
&\qquad-\{z_n(c_0)-f(c_0)\}\Bigr|=o_p(a_n).
\end{aligned}
\]
\end{lemma}

\begin{proof}
The proof has three steps. First, the block characterization of isotonic
regression converts a uniform error in the unprojected curve into a bound on the width of a
pooled block. Second, strict decrease of $f$ prevents a block touching $c_0$
from extending outside the local neighborhood. Third, local equicontinuity
shrinks the relevant blocks from order $b_n$ to $o_p(a_n)$.

\paragraph{Step 1: Relate block width to error in the unprojected curve.}
Write the grid as $x_{0,n}<\cdots<x_{N_n,n}$, set
$y_i=z_n(x_{i,n})$, and first ignore clipping and interpolation. Every maximal
constant block $[L,R]$ of the equal-weight nonincreasing isotonic fit has fitted
value $\bar y_{L:R}=(R-L+1)^{-1}\sum_{i=L}^R y_i$. The KKT conditions imply that, for every
$L\le q<R$,
\begin{equation}
\bar y_{L:q}\le \bar y_{q+1:R}.
\label{eq:isotonic_block_inequality}
\end{equation}
Indeed, summing the isotonic normal equations over the prefix gives
$\bar y_{L:q}\le\bar y_{L:R}$, and summing them over the suffix gives
$\bar y_{q+1:R}\ge\bar y_{L:R}$.

We record the consequence of
\eqref{eq:isotonic_block_inequality} used twice below. Suppose a block lies in
a region on which
\[
f(s)-f(t)\ge\kappa(t-s),\qquad s<t,
\]
and write $y_i=f(x_{i,n})+r_i$ on that block, with
$\max_i|r_i|\le\epsilon_n$. Split the block into two halves by number of grid
points; for a singleton block the asserted width bound is immediate.
Quasi-uniformity supplies a constant $c_{\mathrm{q}}>0$, independent of the
block and $n$, such that the difference between the mean grid locations in the
right and left halves is at least
$c_{\mathrm{q}}(x_{R,n}-x_{L,n})$. Averaging the strict-slope inequality over all
left--right pairs and using \eqref{eq:isotonic_block_inequality} yields
\[
0\ge \bar y_{L:q}-\bar y_{q+1:R}
\ge \kappa c_{\mathrm{q}}(x_{R,n}-x_{L,n})-2\epsilon_n.
\]
Consequently,
\begin{equation}
x_{R,n}-x_{L,n}\le
\frac{2\epsilon_n}{\kappa c_{\mathrm{q}}}.
\label{eq:isotonic_block_width}
\end{equation}

\paragraph{Step 2: Localize every block that can affect $c_0$.}
We next verify that the relevant blocks lie in the strict-slope neighborhood.
Choose $\delta>0$ such that
$\mathcal U_0=[c_0-2\delta,c_0+2\delta]\subset\mathcal U$. For all large $n$,
$Ma_n<\delta/4$. Suppose a block intersecting
$\{|t-c_0|\le Ma_n\}$ exits through the right endpoint of $\mathcal U_0$, and
split it at a grid point $x_{q,n}=c_0+\delta+O(\Delta_{c,n})$. For every
prefix index $i\le q$, monotonicity gives
$f(x_{i,n})\ge f(x_{q,n})$. By quasi-uniformity and compactness of the
full grid, a fixed positive fraction, say $c_\delta$, of the suffix grid points
lies in $[c_0+3\delta/2,c_0+2\delta]$. On those points the strict-slope
inequality gives
\[
f(x_{i,n})\le f(x_{q,n})-\kappa\delta/3
\]
for all sufficiently large $n$, while monotonicity gives
$f(x_{i,n})\le f(x_{q,n})$ on the rest of the suffix. Hence
\[
\bar f_{L:q}-\bar f_{q+1:R}
\ge c_\delta\kappa\delta/3>0.
\]
The full-grid $o_p(1)$ error preserves this sign with probability tending to
one, contradicting \eqref{eq:isotonic_block_inequality}. Splitting at
$c_0-\delta+O(\Delta_{c,n})$ gives the identical contradiction for a block
exiting through the left endpoint. Thus every block intersecting the local
set is contained in $\mathcal U_0$ with probability tending to one.

Fix $\varepsilon>0$ and choose $B<\infty$ so that
$\sup_{\mathcal U}|z_n-f|\le Bb_n$ with limiting probability at least
$1-\varepsilon$. On the preceding localization event,
\eqref{eq:isotonic_block_width} with $\epsilon_n=Bb_n$ shows, uniformly over
all blocks intersecting $\{|t-c_0|\le Ma_n\}$, that their diameters are
$O(b_n)$. Because $a_n=O(b_n)$, all their grid points therefore lie in
$\{|t-c_0|\le Cb_n\}$ for a fixed $C$.

\paragraph{Step 3: Use local equicontinuity inside the localized blocks.}
Now set $v_n(t)=f(t)+e_n(c_0)$ and
\[
\epsilon_n^\circ=
\sup_{|t-c_0|\le Cb_n}|z_n(t)-v_n(t)|.
\]
The local modulus assumption gives $\epsilon_n^\circ=o_p(a_n)$. Reapply
\eqref{eq:isotonic_block_width} on each relevant block, now with the
nonincreasing sequence $v_n$ in place of $f$ and
$\epsilon_n=\epsilon_n^\circ$. The diameter of every such block is
$o_p(a_n)$, uniformly. Since its fitted value is its raw block average and
$f'$ is bounded on $\mathcal U_0$,
\begin{align*}
|\bar y_{L:R}-v_n(x_{j,n})|
&\le \epsilon_n^\circ
+\sup_{i\in[L,R]}|v_n(x_{i,n})-v_n(x_{j,n})|\\
&=o_p(a_n)
\end{align*}
for every $j\in[L,R]$ with $|x_{j,n}-c_0|\le Ma_n$, uniformly over those
grid points. Thus the desired conclusion holds at the grid points.

Linear interpolation adds $O(\Delta_{c,n})=o(a_n)$ because $f'$ is locally
bounded. Clipping is inactive on this neighborhood with probability tending
to one because $f(c_0)\in(0,1)$ and $e_n(c_0)=o_p(1)$. Since
$\varepsilon$ was arbitrary, the result follows.
\end{proof}

\subsection{Proof of Theorem \ref{thm:pointwise_roc} and Corollary \ref{cor:roc_se}}
\label{proof:pointwise_roc}

\begin{proof}
Theorem~\ref{thm:roc_auc_consistency}
and the strict monotonicity of $\tau_0$ give $\hat c_u\to_p c_u$. We first sharpen
this to the scale $a_n$. Lemma~\ref{lem:grid_interpolation} and the joint
fixed-threshold TPR/FPR expansions give, for $y\in\{0,1\}$,
\begin{equation}
\hat\tau_y(c_u)-\tau_y(c_u)
=\sum_{g\in\mathcal G}\mathbb P_{g,n}\psi_{g,y}(c_u)+o_p(a_n).
\label{eq:interpolated_pointwise_expansion}
\end{equation}
Applying Lemma~\ref{lem:isotonic_local} to the raw
TPR and FPR curves gives, for $y\in\{0,1\}$ and every fixed $M$,
\begin{equation}
\begin{aligned}
&\sup_{|t-c_u|\le Ma_n}\Bigl|
\{\tilde\tau_y(t)-\tau_y(t)\}\\
&\qquad-\{\hat\tau_y(c_u)-\tau_y(c_u)\}\Bigr|=o_p(a_n).
\end{aligned}
\label{eq:proved_local_equivalence}
\end{equation}
For either sign, \eqref{eq:proved_local_equivalence} and a Taylor expansion
give
\begin{align*}
\tilde\tau_0(c_u\pm Ma_n)-u
&=\hat\tau_0(c_u)-\tau_0(c_u)\\
&\quad\pm Ma_n\tau_0'(c_u)+o_p(a_n)\\
&\quad+o(Ma_n).
\end{align*}
Equation~\eqref{eq:interpolated_pointwise_expansion} implies
$\hat\tau_0(c_u)-\tau_0(c_u)=O_p(a_n)$. Because
$\tau_0'(c_u)<0$ and is bounded away from zero, this bracketing can be checked
directly. Given $\varepsilon>0$, choose $C_\varepsilon<\infty$ such that
\[
\mathbb P\{|
\hat\tau_0(c_u)-\tau_0(c_u)|\le C_\varepsilon a_n\}
\ge1-\varepsilon
\]
for all sufficiently large $n$. Let
$\kappa=|\tau_0'(c_u)|$ and choose $M>2C_\varepsilon/\kappa$. For this fixed
$M$, the Taylor remainder is smaller than
$M\kappa a_n/2$ for all sufficiently large $n$. Hence, on the displayed
event,
\[
\tilde\tau_0(c_u-Ma_n)>u,
\qquad
\tilde\tau_0(c_u+Ma_n)<u.
\]
Monotonicity then places $\hat c_u$ between these two thresholds and gives
\[
\lim_{M\to\infty}\limsup_n
\mathbb P(|\hat c_u-c_u|>Ma_n)=0,
\]
which is precisely $\hat c_u-c_u=O_p(a_n)$.

The projected curve is continuous after linear interpolation and runs from
one to zero, so its generalized inverse satisfies
$\tilde\tau_0(\hat c_u)=u$; allowing for a grid convention
without interpolation would change this by only $O(\Delta_{c,n})=o(a_n)$.
Since $\hat c_u-c_u=O_p(a_n)$, we may apply
\eqref{eq:proved_local_equivalence} at the random point $\hat c_u$: first
restrict to $\{|\hat c_u-c_u|\le Ma_n\}$, apply the uniform statement for
fixed $M$, and then let $M\to\infty$. Expanding $\tau_0$ at $c_u$ therefore gives
the explicit equation
\[
0=\hat\tau_0(c_u)-\tau_0(c_u)
+\tau_0'(c_u)(\hat c_u-c_u)+o_p(a_n).
\]
Solving it gives the Bahadur representation
\[
\hat c_u-c_u
=-\frac{\hat\tau_0(c_u)-\tau_0(c_u)}{\tau_0'(c_u)}
+o_p(a_n).
\]
The same proved local equivalence for TPR now gives
\begin{align*}
\widehat{\operatorname{ROC}}^{cf}(u)-\operatorname{ROC}(u)
&=\hat\tau_1(c_u)-\tau_1(c_u)\\
&\quad-\frac{\tau_1'(c_u)}{\tau_0'(c_u)}
\{\hat\tau_0(c_u)-\tau_0(c_u)\}\\
&\quad+o_p(a_n).
\end{align*}

Substituting \eqref{eq:interpolated_pointwise_expansion} into the last
display gives
\[
\sum_{g\in\mathcal G}\mathbb P_{g,n}
\{\psi_{g,1}(c_u)-\lambda_u\psi_{g,0}(c_u)\}+o_p(a_n).
\]
The three source sums are independent. The HL and target summands are
bounded, while the AL linear combination has envelope $O(h^{-1})$ and second
moment $O(h^{-1})$. Lemma~\ref{lem:three_sample_clt}, now normalized by
$v_n(u)$, applies because $v_n(u)\asymp a_n$. This proves the stated
Gaussian limit.

For the plug-in standard-error result, write
$\hat\phi_{g,r,u}-\phi_{g,r,u}$ as the sum of the two plug-in errors in
$\hat\psi_{g,1,r}$ and $\hat\psi_{g,0,r}$ and the additional term
$(\hat\lambda_u-\lambda_u)\hat\psi_{g,0,r}$. Consistency of
$\hat\lambda_u$, the source-specific $L_2$ condition in
Corollary~\ref{cor:roc_se}, and
$n_g^{-1}\sum_r\hat\psi_{g,0,r}^2=O_p(d_g)$ give
\[
\frac1{n_g}\sum_{r=1}^{n_g}
(\hat\phi_{g,r,u}-\phi_{g,r,u})^2=o_p(d_g).
\]
The same second-moment comparison used in the proof of
Theorem~\ref{thm:variance_and_ci} then gives
\begin{align*}
\sum_{g\in\mathcal G}\frac{1}{n_g^2}
\sum_{r=1}^{n_g}\{\hat\phi_{g,r,u}-\bar\phi_{g,u}\}^2
&=\sum_{g\in\mathcal G}\frac{\Var_g(\phi_{g,u})}{n_g}\\
&\quad+o_p\{v_n^2(u)\}.
\end{align*}
The standard-error consistency and Wald interval follow by Slutsky's theorem.
\end{proof}

\subsection{Proof of Proposition \ref{prop:oracle_bandwidth}}
\label{proof:oracle_bandwidth}

\begin{proof}
Consider first class $y\in\{0,1\}$ at threshold $c$. The second-order
expansion in the proof of Theorem~\ref{thm:asymptotic_normality_tpr}, applied
with $a=a_y$ and divided by $p_y$, gives for fold $k$
\[
-\frac{h^2\mu_2(K)}{2p_y}
\partial_t^2\mu_{k,1}^{(a_y)}(c)+o_p(h^2r_{S,2}).
\]
The fold weights and derivative terms satisfy
\begin{align*}
\max_k|\alpha_{\mathrm{AL},k}-V^{-1}|
&\le n_{\mathrm{AL}}^{-1},\\
\max_k|\partial_t^2\mu_{k,1}^{(a_y)}(c)|
&=O_p(r_{S,2}).
\end{align*}
Therefore, replacing the exact fold weights by $1/V$ changes this
$h^2$-weighted average by
$O_p(h^2r_{S,2}/n_{\mathrm{AL}})=o_p(h^2r_{S,2})$. Hence the weighted average is
$B_{2,y,n}(c)h^2+o_p(h^2r_{S,2})$. The error from variation along the
score-estimation path is $O_p(r_{S,2}^2)$ and does not depend on $h$.

For the variance, the AL influence term for class $y$ is
\[
\psi_{\mathrm{AL},y,h}(c)
=p_y^{-1}w_{\mathrm{AL}}(X)a_y(X)
K_h(c-Z)\{S-Z\}.
\]
Let
\[
\Gamma_y(t)=f_{Z,\mathrm{AL}}(t)
\mathbb E_{\mathrm{AL}}\!\left[
w_{\mathrm{AL}}^2(X)a_y^2(X)\sigma_S^2(X)\mid Z=t\right].
\]
Because $\mathbb E_{\mathrm{AL}}(S-Z\mid X)=0$, a change of variables gives
\begin{align*}
h\operatorname{Var}_{\mathrm{AL}}
\{\psi_{\mathrm{AL},y,h}(c)\}
&=p_y^{-2}\int K^2(v)\Gamma_y(c-hv)\,dv\\
&\longrightarrow p_y^{-2}R(K)\Gamma_y(c)
=V_{\mathrm{AL},y}(c).
\end{align*}
The convergence follows from continuity of the variance profile at $c$ and the
kernel-tail conditions.
Thus the AL contribution to the variance of the sample average is
$V_{\mathrm{AL},y}(c)/(n_{\mathrm{AL}}h)\{1+o(1)\}$.

Retaining the localization bias and the AL variance contribution just derived
gives the two-term conditional benchmark
$V_{\mathrm{AL},R}/(n_{\mathrm{AL}}h)+B_{2,R,n}^2h^4$.
This is a definition of the two-term benchmark, not an expansion of the full
conditional mean squared error. Differentiating with respect to $h$ gives the
unique interior solution
$h^5=V_{\mathrm{AL},R}/(4n_{\mathrm{AL}}B_{2,R,n}^2)$.
Finally, when the AL variance term dominates the HL and target variance terms,
the ratio of squared localization bias to AL variance is
\[
\frac{B_{2,R,n}^2h^4}
{V_{\mathrm{AL},R}/(n_{\mathrm{AL}}h)}
=\frac{n_{\mathrm{AL}}B_{2,R,n}^2h^5}
{V_{\mathrm{AL},R}}.
\]
This ratio vanishes exactly when $h/h_{\mathrm{AMSE},R}\to_p0$, completing
the proof.
\end{proof}

\subsection{Proof of Theorem \ref{thm:auc_clt} and Corollary \ref{cor:auc_se}}
\label{proof:auc_clt}

\begin{proof}
The proof has four steps. We first expand the target pairwise plug-in. We then
show that the HL and AL source corrections cancel the first-order outcome and
score directions. Third, we expand the prevalence denominator. Finally, we
combine the three source averages and apply a triangular-array CLT.

Recall that
\begin{align*}
\eta
&=\mathbb E_{\mathrm{T}\times\mathrm{T}}
[M(X)\{1-M(X')\}\kappa\{Z(X),Z(X')\}]\\
&=p_1p_0A.
\end{align*}
The last equality follows by conditioning on $(X,X')$: $M(X)$ and
$1-M(X')$ are the conditional probabilities that the first record is a case
and the second is a control. Because $Z$ is continuous, ties have probability
zero.

\paragraph{Step 1: Compute the two path derivatives.}
For an outcome direction $\delta_Y$, direct expansion of the two outcome
weights gives
\begin{align*}
&\eta(M+s\delta_Y,Z)-\eta(M,Z)\\
&\quad=s\,\mathbb E_{\mathrm{T}}
[\delta_Y\{G_0(Z)-G_1(Z)\}]\\
&\qquad-s^2\mathbb E_{\mathrm{T}\times\mathrm{T}}
[\delta_Y(X)\delta_Y(X')\kappa\{Z(X),Z(X')\}].
\end{align*}
Indeed, the first linear term is obtained by conditioning on $X$, whereas in
the second one we interchange $X$ and $X'$ and condition on the renamed first
argument. The quadratic term is bounded in absolute value by
$s^2\|\delta_Y\|_{\mathbb P_{\mathrm{T}},2}^2$. In particular,
\[
D_Y\eta[\delta_Y]
=\mathbb E_{\mathrm{T}}[\delta_Y\{G_0(Z)-G_1(Z)\}].
\]

For a score direction, set
\begin{align*}
W&=M(X)\{1-M(X')\},\\
D&=Z(X)-Z(X'),\\
\Delta&=\delta_S(X)-\delta_S(X').
\end{align*}
Part (a) of Assumption~\ref{ass:auc_path} directly assumes the resulting
score derivative. To see why it has the displayed form, consider the
sufficient signed-density condition stated after that assumption. Under that
condition, the signed measure
$B\mapsto\mathbb E(W\Delta\Ind\{D+t\Delta\in B\})$ has a density $\mu_t$
near zero. Differentiating the ordering probability gives
\[
\frac{d}{dt}\eta(M,Z+t\delta_S)=\mu_t(0).
\]
At $t=0$, integrate the $\delta_S(X)$ part of $\Delta$ over $X'$ and the
$-\delta_S(X')$ part over $X$, interchanging the two record labels in the
second calculation. This yields
\begin{align*}
\mu_0(0)
&=\mathbb E_{\mathrm{T}}[\delta_S Mq_0(Z)]
-\mathbb E_{\mathrm{T}}[\delta_S(1-M)q_1(Z)]\\
&=\mathbb E_{\mathrm{T}}[\delta_S
\{Mq_0(Z)-(1-M)q_1(Z)\}].
\end{align*}
Thus
\[
D_S\eta[\delta_S]
=\mathbb E_{\mathrm{T}}[\delta_S
\{Mq_0(Z)-(1-M)q_1(Z)\}],
\]
and the signs of the two score-density terms follow from the two components
of $\Delta$.
The second-order remainder is $o_p(\rho_n)$ by the two parts of
Assumption~\ref{ass:auc_path}.

We now keep the fold indices explicit. Put
$M_k=\hat m_Y^{(-k)}$, $Z_k=\hat m_S^{(-k)}$, and
$\eta_k=\eta(M_k,Z_k)$. Also write
$\delta_{Y,k}=M_k-M$, $\delta_{S,k}=Z_k-Z$, and
$\delta_{w,g,k}=\hat w_g^{(-k)}-w_g$ for
$g\in\{\mathrm{HL},\mathrm{AL}\}$. Define the population summaries evaluated at the fitted nuisance functions
\begin{align*}
G_{0,k}^{\circ}(z)&=\mathbb E_{\mathrm{T}}
[\{1-M_k(X)\}\kappa\{z,Z_k(X)\}],\\
G_{1,k}^{\circ}(z)&=\mathbb E_{\mathrm{T}}
[M_k(X)\kappa\{Z_k(X),z\}],
\end{align*}
and the mean-zero first projection
\begin{align*}
\gamma_k(X)
&=M_k(X)G_{0,k}^{\circ}\{Z_k(X)\}\\
&\quad+\{1-M_k(X)\}G_{1,k}^{\circ}\{Z_k(X)\}-2\eta_k.
\end{align*}
Let
\begin{align*}
h_k(x,x')
&=M_k(x)\{1-M_k(x')\}
\kappa\{Z_k(x),Z_k(x')\},\\
\bar h_k(x,x')&=\{h_k(x,x')+h_k(x',x)\}/2.
\end{align*}
Although $h_k$ is ordered and need not be symmetric, the estimator sums over
both orders of each pair. Therefore
\[
\hat\eta_{\mathrm{T},k}
=\binom{n_{\mathrm{T},k}}{2}^{-1}
\sum_{\substack{\ell,r\in\mathcal L_k\\\ell<r}}
\bar h_k(X_\ell,X_r).
\]
Let $\mathcal T_k^A$ be the sigma-field generated by all observations used to
train $M_k$ and $Z_k$; it excludes the target validation fold
$\mathcal L_k$. Conditional on $\mathcal T_k^A$, define
\[
h_{1,k}(x)=\mathbb E_{\mathrm{T}}\{\bar h_k(x,X')\}-\eta_k.
\]
Direct conditioning shows that $2h_{1,k}(x)=\gamma_k(x)$. If
\[
h_{2,k}(x,x')=\bar h_k(x,x')-\eta_k-h_{1,k}(x)-h_{1,k}(x'),
\]
then $h_{2,k}$ is degenerate, meaning that its conditional expectation given
either argument is zero. Hoeffding's identity is consequently
\begin{align*}
\hat\eta_{\mathrm{T},k}-\eta_k
&=\mathbb P_{\mathrm{T},\mathcal L_k}\gamma_k+R_{U,k},\\
R_{U,k}
&=\binom{n_{\mathrm{T},k}}{2}^{-1}
\sum_{\substack{\ell,r\in\mathcal L_k\\\ell<r}}
h_{2,k}(X_\ell,X_r).
\end{align*}
Because the kernel is uniformly bounded, degeneracy gives
\[
\mathbb E(R_{U,k}^2\mid\mathcal T_k^A)
=\frac{2\mathbb E(h_{2,k}^2\mid\mathcal T_k^A)}
{n_{\mathrm{T},k}(n_{\mathrm{T},k}-1)}
=O(n_{\mathrm{T},k}^{-2}).
\]
To verify convergence of the fitted ordering kernel, for every $\epsilon>0$,
\begin{align*}
&\mathbb P_{\mathrm{T}\times\mathrm{T}}
[\kappa\{Z_k(X),Z_k(X')\}\ne\kappa\{Z(X),Z(X')\}]\\
&\quad\le \mathbb P(|Z-Z'|\le2\epsilon)
+2\mathbb P(|Z_k-Z|>\epsilon)=o_p(1)
\end{align*}
first along the joint sample-size sequence and then as
$\epsilon\downarrow0$. The last step
uses nuisance consistency and $\mathbb P(Z=Z')=0$. Boundedness then gives
convergence in $\|\cdot\|_{\mathbb P_{\mathrm{T}}\times\mathbb P_{\mathrm{T}},2}$.
Multiplication by the bounded outcome weights and
$\|M_k-M\|_{\mathbb P_{\mathrm{T}},2}=o_p(1)$ show that the fitted pairwise
kernel converges to its population counterpart in the same $L_2$ space.
Conditional expectation is an $L_2$ contraction, so both of its first
projections converge in $\|\cdot\|_{\mathbb P_{\mathrm{T}},2}$; its expectation $\eta_k$ also
converges to $\eta$. Hence
$\|\gamma_k-\gamma\|_{\mathbb P_{\mathrm{T}},2}=o_p(1)$, where
\[
\gamma(X)=MG_0(Z)+(1-M)G_1(Z)-2\eta.
\]
Both $\gamma_k$ and $\gamma$ have target mean zero. Lemma
\ref{lem:conditional_empirical_mean} therefore gives
\[
\sum_{k=1}^V\alpha_{\mathrm{T},k}
(\mathbb P_{\mathrm{T},\mathcal L_k}-\mathbb P_{\mathrm{T}})
(\gamma_k-\gamma)=o_p(n_{\mathrm{T}}^{-1/2}).
\]
Also, the fixed number of degenerate remainders $R_{U,k}$ contributes
$O_p(n_{\mathrm{T}}^{-1})=o_p(n_{\mathrm{T}}^{-1/2})$. Cross-fitting and the
Hoeffding identity therefore yield
\begin{equation}
\begin{aligned}
\hat\eta_{\mathrm{T}}-\frac1V\sum_{k=1}^V\eta_k
&=\mathbb P_{\mathrm{T},n}\gamma+o_p(n_{\mathrm{T}}^{-1/2})\\
&=\mathbb P_{\mathrm{T},n}\gamma+o_p(\rho_n).
\end{aligned}
\label{eq:auc_target_hoeffding}
\end{equation}
The fold-weight reduction at the start of this section justifies the common
$1/V$ weights in this display and in the two source expansions below; its
error is $O_p(\max_g n_g^{-1})=o_p(\rho_n)$.

\paragraph{Step 2: Expand the two source corrections.}
We first make the conditional independence explicit. For the HL correction,
let $\mathcal A_{\mathrm{HL},k}$ contain the data used to fit
$M_k$, $Z_k$, and $\hat w_{\mathrm{HL}}^{(-k)}$, together with the target
fold $\mathcal L_k$ used to construct $\hat\Gamma_{Y,k}$. By construction,
this sigma-field contains no record in $\mathcal I_k$; observations from the
other two samples are also independent of $\mathcal I_k$. Thus the held-out
HL records are i.i.d. and independent of $\mathcal A_{\mathrm{HL},k}$. The
analogous sigma-field for the AL correction excludes $\mathcal J_k$.
Dependence among different training folds does not affect the rates because
$V$ is fixed and each remainder bound is established fold by fold before summing.

For the gold source, define
\begin{align*}
A_{\mathrm{HL},k}
&=\hat w_{\mathrm{HL}}^{(-k)}(Y-M_k)\hat\Gamma_{Y,k},\\
A_{\mathrm{HL}}^0
&=w_{\mathrm{HL}}(Y-M)\Gamma_Y.
\end{align*}
Since $\mathbb P_{\mathrm{HL}}A_{\mathrm{HL}}^0=0$, exact addition and
subtraction gives
\begin{align*}
\mathbb P_{\mathrm{HL},\mathcal I_k}A_{\mathrm{HL},k}
&=\mathbb P_{\mathrm{HL},\mathcal I_k}A_{\mathrm{HL}}^0\\
&\quad+(\mathbb P_{\mathrm{HL},\mathcal I_k}
-\mathbb P_{\mathrm{HL}})
(A_{\mathrm{HL},k}-A_{\mathrm{HL}}^0)\\
&\quad+\mathbb P_{\mathrm{HL}}A_{\mathrm{HL},k}.
\end{align*}
Boundedness and change of measure give
\[
\|A_{\mathrm{HL},k}-A_{\mathrm{HL}}^0\|_{
\mathbb P_{\mathrm{HL}},2}
=O_p(r_{w,\mathrm{HL},2}+r_{Y,2}+r_{\Gamma_Y,2})=o_p(1).
\]
Lemma~\ref{lem:conditional_empirical_mean} makes the centered middle term
$o_p(n_{\mathrm{HL}}^{-1/2})$. For the final population expectation, conditional
mean-zero residuals and change of measure give, line by line,
\begin{align*}
\mathbb P_{\mathrm{HL}}A_{\mathrm{HL},k}
&=-\mathbb P_{\mathrm{T}}(\delta_{Y,k}\hat\Gamma_{Y,k})
-\mathbb P_{\mathrm{HL}}
(\delta_{w,\mathrm{HL},k}\delta_{Y,k}\hat\Gamma_{Y,k})\\
&=-D_Y\eta[\delta_{Y,k}]
-\mathbb P_{\mathrm{T}}
[\delta_{Y,k}(\hat\Gamma_{Y,k}-\Gamma_Y)]\\
&\quad-\mathbb P_{\mathrm{HL}}
(\delta_{w,\mathrm{HL},k}\delta_{Y,k}\hat\Gamma_{Y,k}).
\end{align*}
Cauchy--Schwarz bounds the last two terms by
$O_p[\{r_{\Gamma_Y,2}+r_{w,\mathrm{HL},2}\}r_{Y,2}]$.
Averaging the fixed number of folds consequently gives
\begin{equation}
\begin{aligned}
\hat\eta_{\mathrm{HL}}
&=\mathbb P_{\mathrm{HL},n}
[w_{\mathrm{HL}}(Y-M)\Gamma_Y]\\
&\quad-\frac1V\sum_{k=1}^VD_Y\eta[\delta_{Y,k}]
+R_{\mathrm{HL},n},\\
|R_{\mathrm{HL},n}|
&=O_p[\{r_{\Gamma_Y,2}+r_{w,\mathrm{HL},2}\}r_{Y,2}]
+o_p(\rho_n)\\
&=o_p(\rho_n).
\end{aligned}
\label{eq:auc_hl_expansion}
\end{equation}

For the auxiliary source, set
\begin{align*}
A_{\mathrm{AL},k}
&=\hat w_{\mathrm{AL}}^{(-k)}(S-Z_k)\hat\Gamma_{S,k},\\
A_{\mathrm{AL}}^0
&=w_{\mathrm{AL}}(S-Z)\Gamma_S.
\end{align*}
The same add-and-subtract identity and Lemma
\ref{lem:conditional_empirical_mean} control the centered empirical
difference, because
\[
\|A_{\mathrm{AL},k}-A_{\mathrm{AL}}^0\|_{
\mathbb P_{\mathrm{AL}},2}
=O_p(r_{w,\mathrm{AL},2}+r_{S,2}+r_{\Gamma_S,2})=o_p(1).
\]
Since $\mathbb E_{\mathrm{AL}}(S-Z_k\mid X)=-\delta_{S,k}$, its population
term is
\begin{align*}
\mathbb P_{\mathrm{AL}}A_{\mathrm{AL},k}
&=-D_S\eta[\delta_{S,k}]
-\mathbb P_{\mathrm{T}}
[\delta_{S,k}(\hat\Gamma_{S,k}-\Gamma_S)]\\
&\quad-\mathbb P_{\mathrm{AL}}
(\delta_{w,\mathrm{AL},k}\delta_{S,k}\hat\Gamma_{S,k}).
\end{align*}
Cauchy--Schwarz now gives the remainder
$O_p[\{r_{\Gamma_S,2}+r_{w,\mathrm{AL},2}\}r_{S,2}]$, and therefore
\begin{equation}
\begin{aligned}
\hat\eta_{\mathrm{AL}}
&=\mathbb P_{\mathrm{AL},n}
[w_{\mathrm{AL}}(S-Z)\Gamma_S]\\
&\quad-\frac1V\sum_{k=1}^VD_S\eta[\delta_{S,k}]
+R_{\mathrm{AL},n},\\
|R_{\mathrm{AL},n}|
&=O_p[\{r_{\Gamma_S,2}+r_{w,\mathrm{AL},2}\}r_{S,2}]
+o_p(\rho_n)\\
&=o_p(\rho_n).
\end{aligned}
\label{eq:auc_al_expansion}
\end{equation}
The error measures $r_{\Gamma_Y,2}$ and $r_{\Gamma_S,2}$ include the target
pilot randomness, the evaluation of $G_y$ and $q_y$ at $Z_k$, and estimation
of $M$ within the two gradients. Thus no pilot term is left unaccounted for in
\eqref{eq:auc_hl_expansion}--\eqref{eq:auc_al_expansion}.

\paragraph{Step 3: Cancel the nuisance directions and expand the denominator.}
The foldwise path expansion in Assumption~\ref{ass:auc_path} gives
\[
\frac1V\sum_{k=1}^V(\eta_k-\eta)
=\frac1V\sum_{k=1}^V
\{D_Y\eta[\delta_{Y,k}]+D_S\eta[\delta_{S,k}]\}
+o_p(\rho_n).
\]
Write
$\hat\eta=\hat\eta_{\mathrm{T}}+\hat\eta_{\mathrm{HL}}
+\hat\eta_{\mathrm{AL}}$. Combining this display with
\eqref{eq:auc_target_hoeffding}--\eqref{eq:auc_al_expansion} cancels both
first-order nuisance directions. Recalling that
$\Gamma_Y=H_Y+A(p_0-p_1)$, we first obtain the numerator
expansion
\begin{align*}
\hat\eta-\eta
&=\mathbb P_{\mathrm{HL},n}
[w_{\mathrm{HL}}(Y-M)\\
&\hspace{7em}\times\{G_0(Z)-G_1(Z)\}]\\
&\quad+\mathbb P_{\mathrm{AL},n}
[w_{\mathrm{AL}}(S-Z)\Gamma_S]\\
&\quad+\mathbb P_{\mathrm{T},n}
[MG_0(Z)+(1-M)G_1(Z)\\
&\hspace{8em}-2\eta]\\
&\quad+o_p(\rho_n).
\end{align*}

For completeness, repeat the augmented prevalence decomposition at the
$\rho_n$ scale. Its centered cross-fitted empirical remainder has conditional
standard deviation bounded by constants times
\[
\frac{r_{Y,2}}{\sqrt{n_{\mathrm{T}}}}
+\frac{r_{Y,2}+r_{w,\mathrm{HL},2}}{\sqrt{n_{\mathrm{HL}}}}
=o_p(\rho_n),
\]
because both nuisance errors are $o_p(1)$. Its population remainder is bounded
by $r_{w,\mathrm{HL},2}r_{Y,2}=o_p(\rho_n)$ under
part (b) of Assumption~\ref{ass:auc_rates}. Therefore, without invoking the sharper
$b_{p,n}$-scale conclusion of Theorem~\ref{thm:p1_consistency},
\[
\hat p_1-p_1=\mathbb P_{\mathrm{HL},n}
[w_{\mathrm{HL}}(Y-M)]+\mathbb P_{\mathrm{T},n}(M-p_1)
+o_p(\rho_n).
\]
To make the final ratio step explicit, put
\begin{align*}
d&=p_1p_0,
&\hat d&=\hat p_1(1-\hat p_1),\\
\Delta_p&=\hat p_1-p_1,
&\Delta_\eta&=\hat\eta-\eta.
\end{align*}
Then
\[
\hat d-d=(p_0-p_1)\Delta_p-\Delta_p^2
\]
and the following identity is exact:
\begin{align*}
\frac{\hat\eta}{\hat d}-\frac{\eta}{d}
&=\frac{\Delta_\eta}{d}-\frac{\eta(\hat d-d)}{d^2}\\
&\quad-\frac{\Delta_\eta(\hat d-d)}{d\hat d}
+\frac{\eta(\hat d-d)^2}{d^2\hat d}.
\end{align*}
The numerator and prevalence expansions show that
$\Delta_\eta=O_p(\rho_n)$ and $\Delta_p=O_p(\rho_n)$. Since $d$ is bounded
away from zero, the last two displayed terms and the contribution of
$\Delta_p^2$ are $O_p(\rho_n^2)=o_p(\rho_n)$. Retaining the linear terms
gives the following three source-specific expressions:
\begin{align*}
&\frac{w_{\mathrm{HL}}(Y-M)}{p_1p_0}
\{\Gamma_Y-A(p_0-p_1)\}
=\varphi_{\mathrm{HL}}^A,\\
&\frac{w_{\mathrm{AL}}(S-Z)\Gamma_S}{p_1p_0}
=\varphi_{\mathrm{AL}}^A,\\
&\frac{\gamma-A(p_0-p_1)(M-p_1)}{p_1p_0}
=\varphi_{\mathrm{T}}^A.
\end{align*}
Therefore
\[
\widehat{\operatorname{AUC}}_{\mathrm{os}}^{cf}-A
=\sum_{g\in\mathcal G}\mathbb P_{g,n}\varphi_g^A+o_p(\rho_n).
\]

\paragraph{Step 4: Prove the CLT and plug-in variance consistency.}
The three leading averages are independent. Moreover, the boundedness and
density-ratio conditions, together with bounded $q_0$ and $q_1$, give a common
finite envelope $C$ for all three $\varphi_g^A$. For an observation from
source $g$, its normalized contribution is therefore bounded by
\[
\frac{C}{n_g\sigma_{A,n}}
\le \frac{C'}{n_g\rho_n}
\le \frac{C'}{\sqrt{n_g}}\to0,
\]
where the second inequality uses
$\sigma_{A,n}\asymp\rho_n$ and the third uses
$\rho_n\ge n_g^{-1/2}$. The normalized variances sum to one by the definition
of $\sigma_{A,n}$. Hence the Lindeberg condition holds and the
Lindeberg--Feller theorem gives the stated Gaussian limit.

For the plug-in variance result, let
$e_{g,r}=\hat\varphi_{g,r}^A-\varphi_{g,r}^A$, and write bars for
source-specific empirical means. Centering is an empirical $L_2$ contraction,
so the assumed source-specific convergence gives
$\mathbb P_{g,n}(e_g-\bar e_g)^2=o_p(1)$. Cauchy--Schwarz then yields
\begin{align*}
&\left|\mathbb P_{g,n}
(\hat\varphi_g^A-\bar{\hat\varphi}_g^A)^2
-\mathbb P_{g,n}(\varphi_g^A-\bar\varphi_g^A)^2\right|\\
&\quad\le
2\{\mathbb P_{g,n}(\varphi_g^A-\bar\varphi_g^A)^2\}^{1/2}
\{\mathbb P_{g,n}(e_g-\bar e_g)^2\}^{1/2}\\
&\qquad+\mathbb P_{g,n}(e_g-\bar e_g)^2
=o_p(1),
\end{align*}
because the oracle influence functions are bounded. The ordinary law of large
numbers gives
$\mathbb P_{g,n}(\varphi_g^A-\bar\varphi_g^A)^2
=\Var_g(\varphi_g^A)+o_p(1)$. Multiplying the preceding comparison by
$n_g^{-1}$ and summing the fixed three sources gives an error
$o_p(\rho_n^2)=o_p(\sigma_{A,n}^2)$. Therefore
\[
\sum_{g\in\mathcal G}\frac{1}{n_g^2}
\sum_{r=1}^{n_g}(\tilde\varphi_{g,r}^A)^2
=\sum_{g\in\mathcal G}\frac{\Var_g(\varphi_g^A)}{n_g}
+o_p(\sigma_{A,n}^2).
\]
Variance consistency and the Wald interval follow from Slutsky's theorem.
\end{proof}

\paragraph{Scope of the proof.}
The TPR and FPR arguments are pointwise in $c$, and the ROC argument adds
only local stochastic equicontinuity around the crossing $c_u$. It does not
establish the global process approximation needed for a uniform ROC band.
The AUC result instead uses a scalar pairwise expansion and therefore does not
imply simultaneous validity over the ROC curve. The centered Wald results for
TPR, FPR, and pointwise ROC require the remainder conditions in
Assumption~\ref{ass:rate}; the one-step AUC result instead uses
Assumption~\ref{ass:auc_path}. For the threshold-based
results, the condition $h^2r_{S,2}=o_p(a_n)$ ensures that localization bias is
negligible on the inference scale. The bandwidth sensitivity analysis
examines the corresponding finite-sample behavior.

\section{Additional Simulation Details and Results}
\label{sec:sim_setup_details}

\subsection{Data-Generating Processes}

All target covariates have dimension \(d=20\). For DGP-A,
\[
  X\sim N(0,\Sigma_A),\qquad
  (\Sigma_A)_{jk}=0.2^{|j-k|},
\]
and
\begin{align*}
  \eta_Y^A(X)
    &=0.10+0.70X_1+0.45\sin(X_2)\\
    &\quad+0.30X_3X_4+0.20(X_5^2-1),\\
  \eta_S^A(X)
    &=-0.05+1.00X_1+0.30\sin(X_2)\\
    &\quad+0.45X_3X_4+0.12(X_5^2-1).
\end{align*}

DGP-B represents a heterogeneous target population:
\begin{align*}
  X&=B\mu+\varepsilon,\qquad \varepsilon\sim N(0,\Sigma_B),\\
  \Pr(B=1)&=\Pr(B=-1)=\tfrac12,
\end{align*}
where
\[
  \mu=(0.9,-0.7,0.6,0,0,-0.5,0,\ldots,0)^\top.
\]
The diagonal entries of \(\Sigma_B\) are \(0.65\);
\((\Sigma_B)_{2r-1,2r}=(\Sigma_B)_{2r,2r-1}=-0.12\) for
\(r=1,\ldots,10\); all other off-diagonal entries are zero. Its response
surfaces are
\begin{align*}
  \eta_Y^B(X)
    &=-0.20+0.85\tanh(1.2X_1)-0.50X_2\\
    &\quad+0.40(X_3^2-1.01)+0.30\sin(X_4X_5),\\
  \eta_S^B(X)
    &=-0.50+0.55\tanh(1.2X_1)-0.15X_2\\
    &\quad+0.65(X_3^2-1.01)-0.35\sin(X_4X_5)\\
    &\quad+0.30\tanh(X_6).
\end{align*}


For \(q\in\{A,B\}\),
\(m_Y^\star(x)=\operatorname{expit}\{\eta_Y^q(x)\}\) and
\(m_S^\star(x)=\operatorname{expit}\{\eta_S^q(x)\}\). We draw
\(Y\mid X\sim\operatorname{Bernoulli}\{m_Y^\star(X)\}\) in HL and
\(S\mid X\sim\operatorname{Bernoulli}\{m_S^\star(X)\}\) in AL.
Note that while our asymptotic theory assumes compactly supported covariates for uniform rate control (Assumption \ref{ass:rate}), the use of unbounded Gaussian covariates here serves as an intentional stress test of the estimators' empirical robustness.

\paragraph{Additional DGP-D diagnostic.}
The regularized quadratic-dictionary diagnostic uses a separate
20-dimensional sparse quadratic design:
\[
  X\sim N(0,\Sigma_D),\qquad
  (\Sigma_D)_{jk}=0.3^{|j-k|}.
\]
Its outcome and surrogate mechanisms are
\begin{align*}
  \eta_Y^D(X)
    &=0.10+0.65X_1-0.45X_2+0.30X_3X_4\\
    &\quad+0.22(X_5^2-1),\\
  \eta_S^D(X)
    &=-0.10+0.95X_1-0.30X_2+0.48X_3X_4\\
    &\quad+0.18(X_5^2-1)+0.20X_6.
\end{align*}
We set
\[
m_Y^\star(x)=\operatorname{expit}\{\eta_Y^D(x)\},
\qquad
m_S^\star(x)=\operatorname{expit}\{\eta_S^D(x)\},
\]
and generate the corresponding Bernoulli gold and surrogate labels.

\subsection{Controlled and Moderate Source Shifts}

Studies~1, 2, and 4 use a controlled bounded-overlap design. Source
covariates are selected from the target law with
\begin{align*}
  \pi_{\mathrm{HL}}(x)
    &=0.35+0.30\operatorname{expit}\{a_{\mathrm{HL}}(x)\},\\
  \pi_{\mathrm{AL}}(x)
    &=0.40+0.20\operatorname{expit}\{a_{\mathrm{AL}}(x)\}.
\end{align*}
For DGP-A,
\begin{align*}
  a_{\mathrm{HL}}&=0.45X_1-0.30X_2+0.20X_6,\\
  a_{\mathrm{AL}}&=-0.40X_1+0.25X_3-0.20X_7,
\end{align*}
whereas for DGP-B,
\begin{align*}
  a_{\mathrm{HL}}&=0.35X_1+0.25X_3-0.20X_7,\\
  a_{\mathrm{AL}}&=-0.30X_2+0.25X_4+0.20X_8.
\end{align*}
Central symmetry gives \(E_{\mathrm T}\{\pi_g(X)\}=1/2\), so the exact
untruncated ratio is \(w_g(x)=\{2\pi_g(x)\}^{-1}\).

Study~3 uses a stronger bounded log-density-ratio tilt:
\begin{align*}
  \frac{p_g(x)}{p_{\mathrm T}(x)}
    &=\frac{\exp\{\gamma_g^\top z_g(x)\}}{Z_g},\\
  w_g(x)&=Z_g\exp\{-\gamma_g^\top z_g(x)\},
\end{align*}
where
\begin{align*}
  z_{\mathrm{HL}}(x)&=(\tanh X_1,\tanh X_2,\tanh X_6)^\top,\\
  z_{\mathrm{AL}}(x)&=(\tanh X_1,\tanh X_3,\tanh X_7)^\top,\\
  \gamma_{\mathrm{HL}}&=(0.810,-0.675,0.585)^\top,\\
  \gamma_{\mathrm{AL}}&=(-0.765,0.630,-0.540)^\top,
\end{align*}
and \(Z_g=E_{\mathrm T}[\exp\{\gamma_g^\top z_g(X)\}]\).
This construction is bounded because every coordinate of \(z_g\) is in
\([-1,1]\). For all non-oracle Study~3 estimators, \(w_{\mathrm{HL}}\) and
\(w_{\mathrm{AL}}\) are estimated separately in each cross-fitting training
fold. With target membership coded as one, the classifier probability
\(\hat\pi_g(x)\) is converted by
\[
  \hat w_g(x)=
  \frac{n_g^{(-v)}}{n_{\mathrm T}^{(-v)}}
  \frac{\hat\pi_g^{(-v)}(x)}
       {1-\hat\pi_g^{(-v)}(x)}.
\]
The logistic domain models use \(z_g\), for which the log density ratio is
correctly specified. A numerical probability guard at \(10^{-6}\) was
included but never activated; no realized weight was truncated. The oracle
normalizers \(Z_g\) are evaluated by the same \(10^6\)-draw population Monte
Carlo calculation used for the truths, yielding a high-accuracy numerical
oracle.

\begin{table}[t]
  \centering
  \caption{Study 3 cross-fitted density-ratio diagnostics, averaged over 500
  replications. ESS is \((\sum_iw_i)^2/\sum_iw_i^2\).}
  \label{tab:sim_weight_diagnostics}
  \scriptsize
  \begin{tabular}{llrrrr}
    \toprule
    DGP & Source & Domain AUC & Mean wt. & Mean max. & ESS \\
    \midrule
    DGP-A & HL & 0.689 & 1.002 & 7.800 & 2462 \\
    DGP-A & AL & 0.688 & 1.001 & 7.759 & 9913 \\
    DGP-B & HL & 0.713 & 1.000 & 8.814 & 2096 \\
    DGP-B & AL & 0.659 & 1.000 & 6.563 & 11447 \\
    \bottomrule
  \end{tabular}
\end{table}

The DGP-D diagnostic uses the same bounded-overlap selection form, with
\begin{align*}
  a_{\mathrm{HL}}(x)
    &=0.40X_1-0.30X_6+0.20X_7,\\
  a_{\mathrm{AL}}(x)
    &=-0.35X_1+0.25X_8-0.20X_9.
\end{align*}
Central symmetry again gives
\(E_{\mathrm T}\{\pi_g(X)\}=1/2\), so the exact untruncated
density ratio is \(w_g(x)=\{2\pi_g(x)\}^{-1}\).

\subsection{Implementation and Monte Carlo Algorithm}
\label{app:simulation_algorithm}

The correctly specified nuisance learner is logistic regression on the exact DGP-specific nonlinear basis. The sensitivity analysis instead uses histogram gradient boosting with 50 iterations, at most 10 leaves, and minimum leaf size 35. Pointwise targets use \(c\in\{0.2,0.5\}\). ROC uses \(u\in\{0.1,0.2\}\), a 199-point threshold grid, separate isotonic projections, and a common grid-based linear interpolation convention. This finite-grid convention approximates the generalized inverse and is applied identically to Full, IW-HL, and TSC. The scalar one-step AUC uses the half-tie pairwise rule and the density pilot \(b=n_{\mathrm T}^{-1/5}\). Population truths are computed from \(10^6\) independent target draws.

\paragraph{Fixed-score transported-validation comparators.}
For Studies~1, 2, and 4, let
\(\widehat Z(x)=\widehat m_S^{\mathrm{AL}}(x)\), where the same correctly
specified score learner is fit once on the complete independent AL sample.
IW-HL uses the exact controlled-design HL weight and the Hájek ratios
\begin{align*}
 \widehat{\mathrm{TPR}}_{\mathrm{IW}}(c)
 &=\frac{\mathbb P_{\mathrm{HL}}
   [w_{\mathrm{HL}}Y\Ind\{\widehat Z\ge c\}]}
   {\mathbb P_{\mathrm{HL}}(w_{\mathrm{HL}}Y)},\\
 \widehat{\mathrm{FPR}}_{\mathrm{IW}}(c)
 &=\frac{\mathbb P_{\mathrm{HL}}
   [w_{\mathrm{HL}}(1-Y)\Ind\{\widehat Z\ge c\}]}
   {\mathbb P_{\mathrm{HL}}\{w_{\mathrm{HL}}(1-Y)\}}.
\end{align*}
Its ROC is formed from the transported weighted empirical curve, and its AUC
is the corresponding half-tie weighted Mann--Whitney functional.

TSC uses the cubic Legendre score basis
\[
 b_3(z)=
 \left(1,t,\frac{3t^2-1}{2},\frac{5t^3-3t}{2}\right)^\top,
 \qquad t=2z-1.
\]
It fits \(r_\beta(z)=\operatorname{expit}\{\beta^\top b_3(z)\}\) by
HL-weighted logistic likelihood and standardizes the fitted calibration over
the target scores; for example,
\[
 \widehat{\mathrm{TPR}}_{\mathrm{TSC}}(c)
 =
 \frac{\mathbb P_{\mathrm T}
 [r_{\widehat\beta}(\widehat Z)\Ind\{\widehat Z\ge c\}]}
 {\mathbb P_{\mathrm T}[r_{\widehat\beta}(\widehat Z)]},
\]
with \(1-r_{\widehat\beta}\) replacing \(r_{\widehat\beta}\) for FPR.
ROC and AUC use the corresponding target-standardized pairwise
functionals. IW-HL uses its transported empirical influence function; TSC
combines the weighted calibration \(M\)-estimation influence function with
the target empirical influence function. Both standard errors condition on
\(\widehat Z\) and therefore omit AL score-learning variability. This is a
valid fixed-score analysis but is not unconditional inference for the
population surrogate-derived model \(m_S^\star\), which is the estimand of the proposed method.
Using the complete AL sample rather than folds gives both comparators their
strongest natural fixed-score implementation.

\begin{algorithm}[t]
  \caption{Three-sample Monte Carlo experiment}
  \label{alg:simulation_pipeline}
  \footnotesize
  \begin{algorithmic}[1]
    \REQUIRE DGP, design, \(C\), \(R=500\), and \(V=5\) folds
    \STATE Set \(h=Cn_{\mathrm{AL}}^{-1/4}\); compute population truths
    \FOR{\(r=1,\ldots,R\)}
      \STATE Draw the target, HL, and AL samples and split each into \(V\) folds
      \FOR{\(v=1,\ldots,V\)}
        \STATE Fit \(m_Y,m_S\) and exact or estimated density ratios
        \STATE Form held-out nuisance predictions and weights
      \ENDFOR
      \STATE Compute Full TPR, FPR, ROC, AUC, and influence-function SEs
      \IF{Study 1, 2, or 4}
        \STATE Fit the score on complete AL; evaluate IW-HL and TSC with their conditional SEs
      \ELSIF{Study 3}
        \STATE Evaluate plug-in, HL-only, AL-only, Full, and oracle variants
      \ENDIF
    \ENDFOR
    \STATE Report bias, empirical SD, mean SE, SE/SD, and 95\% coverage
  \end{algorithmic}
\end{algorithm}

\begin{table}[t]
  \centering
  \caption{Study 1 TPR empirical SD/mean analytical SE at \(c=0.5\) for the
  proposed Full estimator. Non-proportional designs change one P2 sample size
  at a time.}
  \label{tab:sim_scaling_sdse}
  \small
  \begin{tabular}{lcc}
    \toprule
    Design & DGP-A & DGP-B \\
    \midrule
    P1 & 0.0295/0.0281 & 0.0323/0.0292 \\
    P2 & 0.0208/0.0205 & 0.0233/0.0216 \\
    P3 & 0.0155/0.0151 & 0.0161/0.0161 \\
    HL 0.25k & 0.0354/0.0343 & 0.0333/0.0328 \\
    HL 2k & 0.0188/0.0175 & 0.0214/0.0194 \\
    AL 1k & 0.0304/0.0271 & 0.0385/0.0301 \\
    AL 8k & 0.0195/0.0186 & 0.0186/0.0188 \\
    T 1k & 0.0259/0.0251 & 0.0299/0.0279 \\
    T 5k & 0.0219/0.0213 & 0.0236/0.0226 \\
    T 40k & 0.0206/0.0204 & 0.0216/0.0214 \\
    \bottomrule
  \end{tabular}
\end{table}

\begin{figure*}[t]
  \centering
  \includegraphics[width=\textwidth]{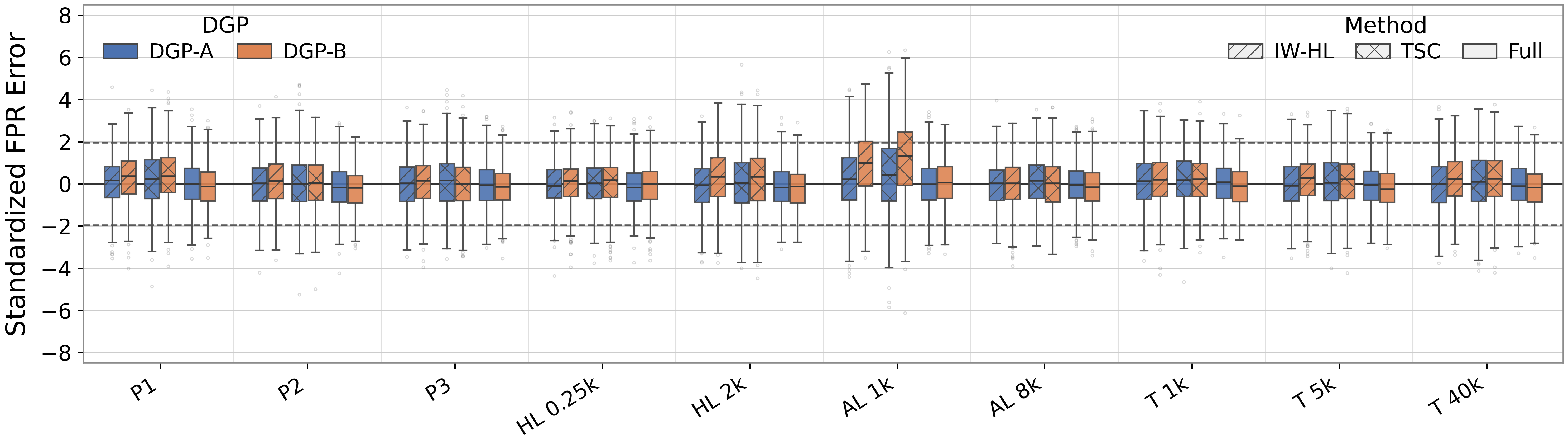}
  \caption{Study 1 standardized FPR errors at \(c=0.5\). Full is the proposed
  estimator; IW-HL is importance-weighted gold-label validation; TSC is
  transported score calibration. Colors distinguish DGPs. HL, AL, and T
  denote the gold-labeled, surrogate-labeled, and target samples; dashed lines
  mark \(\pm1.96\), and outliers are shown.}
  \label{fig:sim_scaling_fpr}
\end{figure*}

\begin{figure*}[t]
  \centering
  \begin{subfigure}{0.49\textwidth}
    \includegraphics[width=\linewidth]{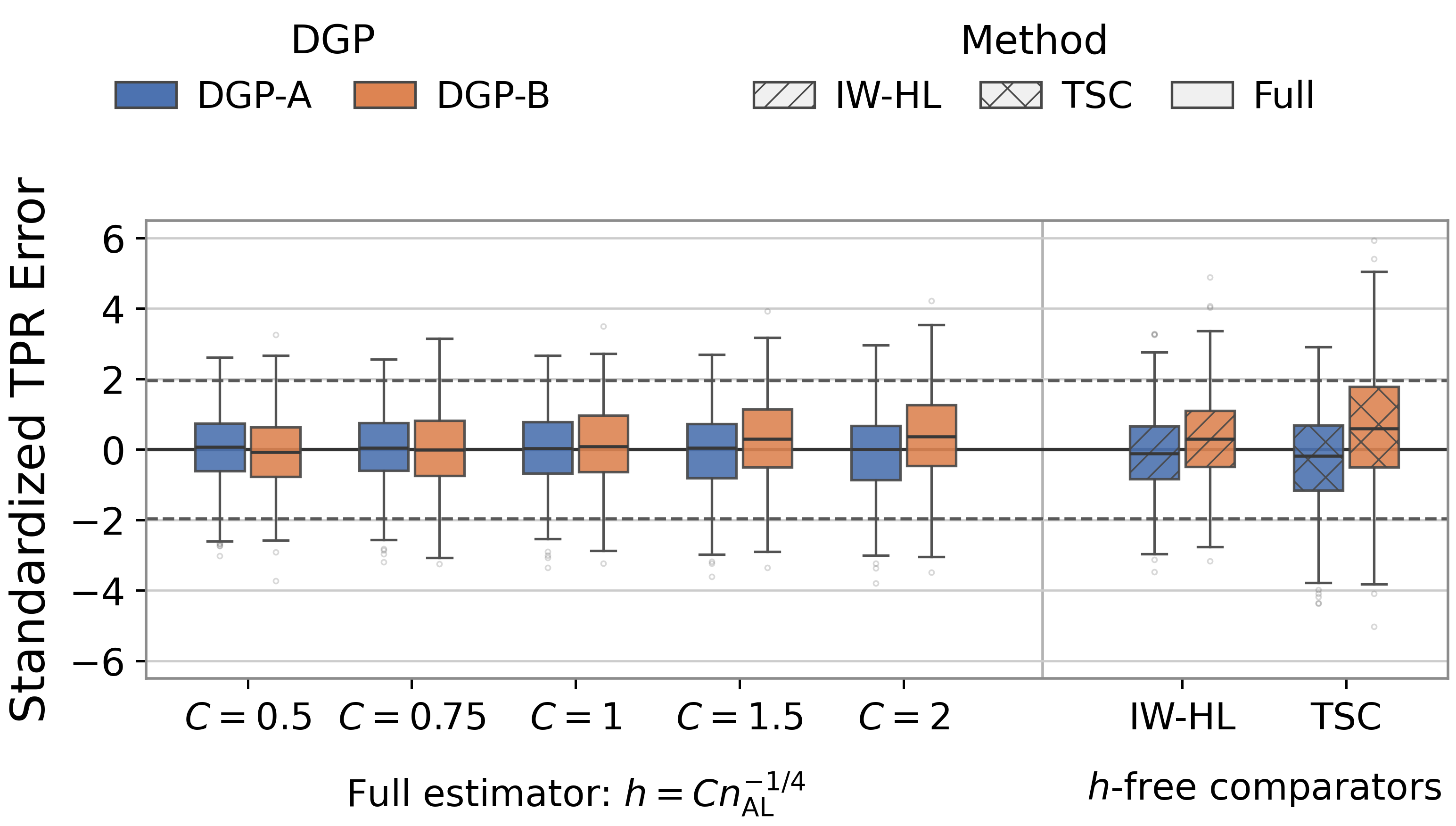}
    \caption{TPR}
    \label{fig:sim_bandwidth_tpr}
  \end{subfigure}
  \hfill
  \begin{subfigure}{0.49\textwidth}
    \includegraphics[width=\linewidth]{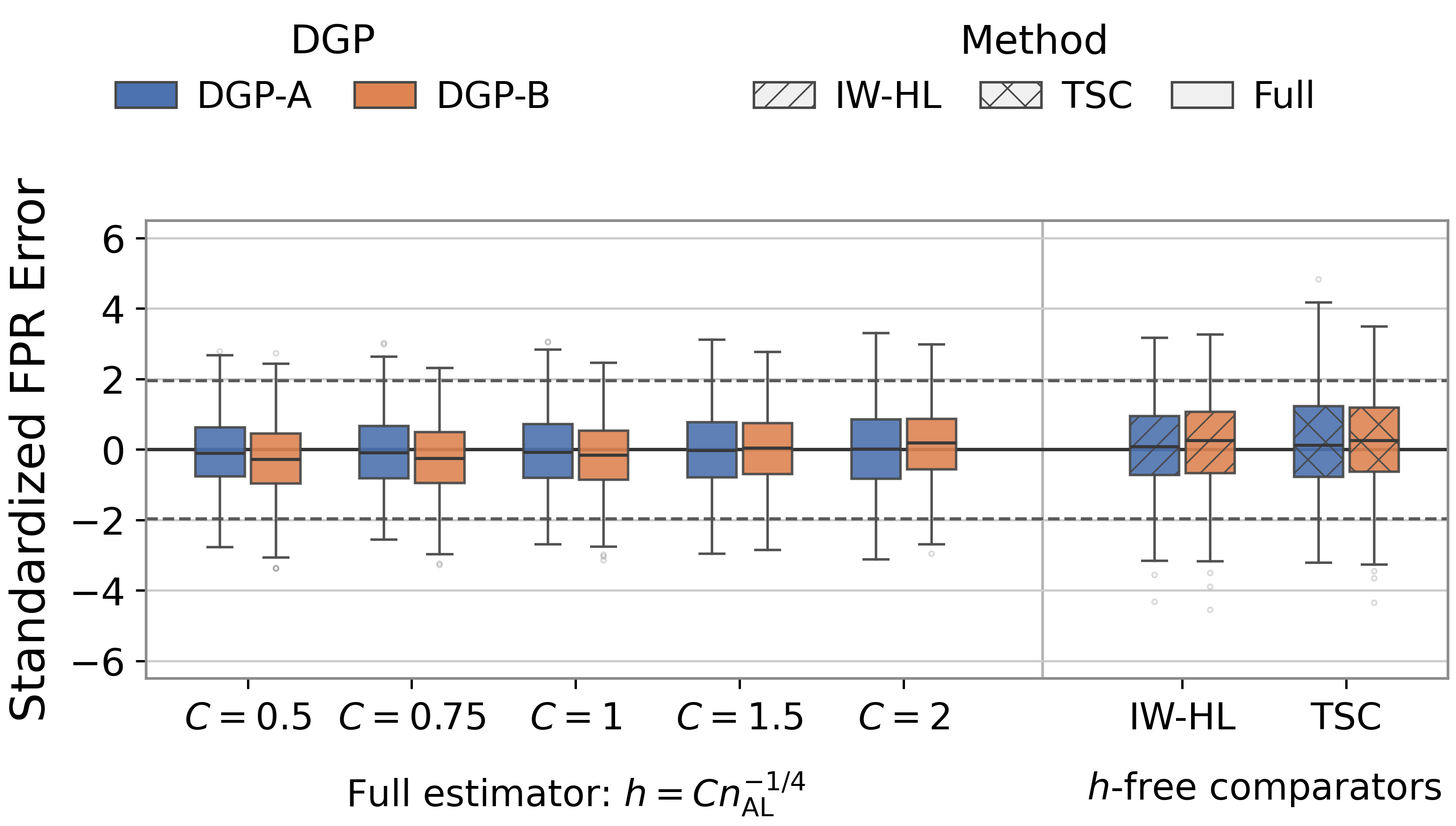}
    \caption{FPR}
    \label{fig:sim_bandwidth_fpr}
  \end{subfigure}
  \caption{Study 2 standardized errors at \(c=0.5\). The five Full groups
  vary \(C\); the \(h\)-free IW-HL and TSC references are shown once. Colors
  distinguish DGPs, hatches distinguish methods, dashed lines mark
  \(\pm1.96\), and outliers are shown.}
  \label{fig:sim_bandwidth_comparison}
\end{figure*}

\begin{figure}[t]
  \centering
  \includegraphics[width=\linewidth]{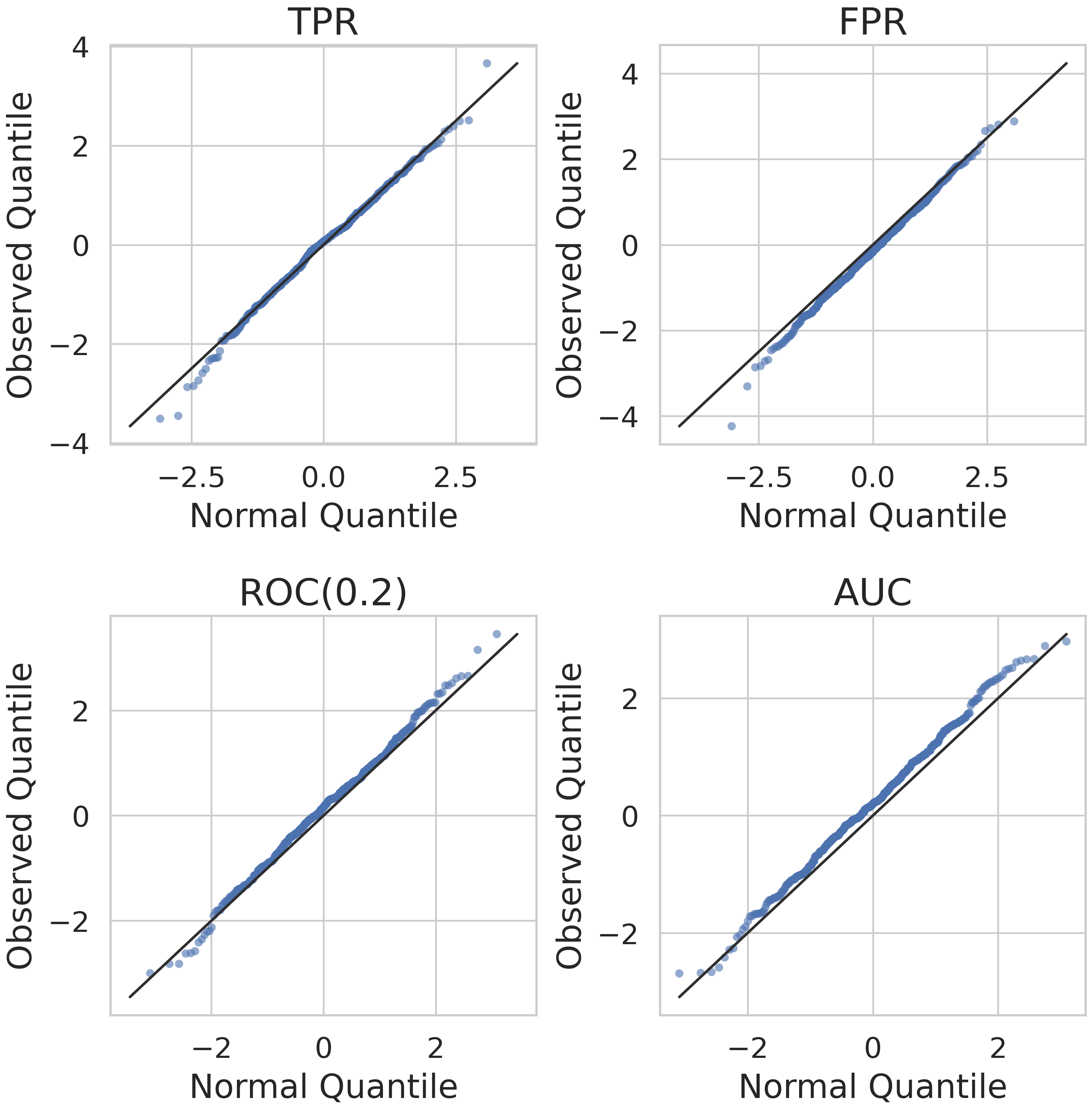}
  \caption{Normal Q--Q diagnostics for the proposed Full estimator under
  DGP-A at P2. TPR and FPR use \(c=0.5\), ROC uses \(u=0.2\), and AUC is the
  scalar one-step estimator.}
  \label{fig:sim_qq_appendix}
\end{figure}

For the additional DGP-D diagnostic, all 20 raw covariates are expanded into the complete degree-two dictionary containing 20 linear terms, 20 squared terms, and 190 pairwise products. Thus, the learner receives 230 candidate terms but is not supplied with the true active terms or their variable indices. Within each training fold, these features are standardized and separately fit to the HL outcome and AL surrogate labels using elastic-net logistic regression with fixed \(C=0.05\), \(\ell_1\) ratio \(0.5\), the SAGA solver, tolerance \(10^{-3}\), and at most 750 iterations. No internal cross-validation or data-driven hyperparameter tuning is performed. The diagnostic uses \(4\times\mathrm{P2}=(4000,16000,80000)\), exact controlled-shift density ratios, five-fold cross-fitting, and \(h=n_{\mathrm{AL}}^{-1/4}\).

\subsection{Coverage and Standard-Error Summary}

\begin{table*}[t]
  \centering
  \caption{Empirical coverage and SE/SD of the proposed Full construction over
  500 replications. Study 3 rows select the Full ablation member; HGB rows use
  the same estimator with histogram-gradient-boosting nuisances. ``Worst''
  gives the configuration attaining minimum coverage.}
  \label{tab:sim_calibration_ranges}
  \scriptsize
  \begin{tabular}{lllccc}
    \toprule
    Study & Target & Coordinate & Coverage range & SE/SD range & Worst \\
    \midrule
    Study 1 & TPR & \(c=0.2\) & .896--.944 & .826--1.029 &
      B/AL 1k (.896) \\
    Study 1 & FPR & \(c=0.2\) & .662--.930 & .496--.916 &
      B/AL 1k (.662) \\
    Study 1 & TPR & \(c=0.5\) & .872--.952 & .782--1.014 &
      B/AL 1k (.872) \\
    Study 1 & FPR & \(c=0.5\) & .914--.960 & .883--1.022 &
      A/AL 1k (.914) \\
    \addlinespace
    Study 2 & TPR & \(c=0.2\) & .904--.946 & .876--1.007 &
      B/\(C=2\) (.904) \\
    Study 2 & FPR & \(c=0.2\) & .580--.928 & .434--.952 &
      B/\(C=2\) (.580) \\
    Study 2 & TPR & \(c=0.5\) & .862--.960 & .798--.991 &
      B/\(C=2\) (.862) \\
    Study 2 & FPR & \(c=0.5\) & .906--.940 & .858--.981 &
      A/\(C=2\) (.906) \\
    \addlinespace
    Study 3 Full & TPR & \(c=0.5\) & .944--.956 & .958--1.036 &
      B (.944) \\
    Study 3 Full & FPR & \(c=0.5\) & .950--.950 & .987--1.017 &
      A/B (.950) \\
    Study 3 Full & ROC & \(u=0.1\) & .936--.938 & .951--.963 &
      B (.936) \\
    Study 3 Full & AUC & scalar & .926--.934 & .927--1.014 &
      A (.926) \\
    \addlinespace
    Study 4 & ROC & \(u=0.1,0.2\) & .882--.954 & .817--1.001 &
      B/P2/\(u=.1\) (.882) \\
    Study 4 & AUC & scalar & .932--.950 & .974--1.006 &
      A/P2 (.932) \\
    \addlinespace
    HGB & TPR/FPR & \(c=0.2\) & .288--.622 & .429--.882 &
      A/FPR (.288) \\
    HGB & TPR/FPR & \(c=0.5\) & .770--.886 & .758--.898 &
      B/FPR (.770) \\
    HGB & ROC & \(u=0.1,0.2\) & .764--.906 & .749--.878 &
      B/\(u=.2\) (.764) \\
    HGB & AUC & scalar & .788--.922 & .927--.930 &
      B (.788) \\
    \bottomrule
  \end{tabular}
\end{table*}

The table separates the main theory-focused experiments from stress tests. The histogram-gradient-boosting (HGB) learner has materially poorer finite-sample calibration: coverage is \(0.288\)--\(0.622\) at \(c=0.2\), \(0.770\)--\(0.886\) at \(c=0.5\), \(0.764\)--\(0.906\) for ROC, and \(0.788\)--\(0.922\) for AUC. This pattern is consistent with nuisance errors not yet being negligible at the inference scale. It neither validates arbitrary machine learning nor shows that flexible learners necessarily fail; it demonstrates that inference depends materially on the nuisance-rate conditions.

\begin{table*}[t]
  \centering
  \caption{Comparison with fixed-score transported validation. Entries average
  over the stated design-by-DGP cells. Coverage is mean [minimum, maximum].
  Study 2 compares the default Full choice \(C=1\) with the \(h\)-free
  comparators.}
  \label{tab:sim_comparator_summary}
  \scriptsize
  \setlength{\tabcolsep}{3.3pt}
  \begin{tabular}{llrrrrc}
    \toprule
    Target/design & Method & Cells & Mean \(|\)bias\(|\) & Mean RMSE &
      Mean SE/SD & Coverage \\
    \midrule
    S1 TPR(.5), all & Full  & 20 & .0020 & .0250 & .950 & .934 [.872,.952] \\
                       & IW-HL & 20 & .0041 & .0284 & .821 & .879 [.668,.930] \\
                       & TSC   & 20 & .0056 & .0230 & .701 & .805 [.520,.928] \\
    \addlinespace
    S1 FPR(.5), all & Full  & 20 & .0020 & .0217 & .975 & .941 [.914,.960] \\
                       & IW-HL & 20 & .0030 & .0238 & .842 & .887 [.686,.948] \\
                       & TSC   & 20 & .0031 & .0192 & .770 & .850 [.600,.920] \\
    \addlinespace
    S2 TPR(.5), P2 & Full  & 2 & .0017 & .0225 & .941 & .945 [.934,.956] \\
                     & IW-HL & 2 & .0046 & .0260 & .837 & .879 [.864,.894] \\
                     & TSC   & 2 & .0058 & .0209 & .679 & .804 [.752,.856] \\
    \addlinespace
    S2 FPR(.5), P2 & Full  & 2 & .0013 & .0199 & .951 & .938 [.936,.940] \\
                     & IW-HL & 2 & .0031 & .0224 & .823 & .890 [.876,.904] \\
                     & TSC   & 2 & .0030 & .0182 & .718 & .837 [.836,.838] \\
    \addlinespace
    S4 ROC(.1,.2), P1--P3 & Full  & 12 & .0048 & .0367 & .901 & .919 [.882,.954] \\
                            & IW-HL & 12 & .0023 & .0353 & .838 & .896 [.864,.946] \\
                            & TSC   & 12 & .0029 & .0309 & .891 & .911 [.868,.952] \\
    \addlinespace
    S4 AUC, P1--P3 & Full  & 6 & .0026 & .0181 & .986 & .941 [.932,.950] \\
                     & IW-HL & 6 & .0031 & .0178 & .964 & .940 [.926,.950] \\
                     & TSC   & 6 & .0031 & .0175 & .962 & .937 [.924,.950] \\
    \bottomrule
  \end{tabular}
\end{table*}

Table~\ref{tab:sim_comparator_summary} makes the distinction between point accuracy and reported-SE behavior explicit. Full has the most stable coverage in Study~1, while TSC can trade lower RMSE for optimistic conditional SEs. The comparison is not uniformly favorable: for FPR at \(c=0.2,C=2\), Full versus IW-HL coverage is \(0.814\) versus \(0.864\) in DGP-A and \(0.580\) versus \(0.676\) in DGP-B. This reinforces the bandwidth warning rather than being hidden by an average. In the difficult DGP-A/HL 0.25k cell at \(c=0.2\), three of 500 IW-HL replications have a degenerate zero conditional TPR SE because every empirical positive lies above the threshold; coverage uses all 500 replications, while studentized summaries use the 497 nondegenerate values.

\subsection{Unified Study 3 Numerical Results}
\label{app:ablattpr}

Table~\ref{tab:sim_ablation_full} reports the confirmed P4 Study~3 comparison at \(c=0.5\), \(u=0.1\), and for scalar AUC. Each cell reports, in order, bias in percentage points, empirical SD, mean SE, and coverage. Because the original P2 run showed finite-sample ROC undercoverage, we conducted an exploratory 100-replication screen. The first stage crossed P2, P2.5, and P3 with \(C\in\{0.75,1,1.25\}\), \(c\in\{0.4,0.5,0.6\}\), and \(u\in\{0.10,0.15,0.20\}\). A second 100-replication stage evaluated P4 at \(c=0.5\), with \(C\in\{0.5,1.5,2\}\) and \(u\in\{0.10,0.15,0.20,0.25\}\). Before inspecting a new Monte Carlo stream, we fixed P4, \(C=0.5\), \(c=0.5\), and \(u=0.1\), then ran an independent 500-replication confirmation with a disjoint base seed. This was a post-hoc screen followed by independent confirmation, not a preregistered design. All 93 screening candidates and the independent run are retained in the result files; the confirmation, rather than a screening maximum, is shown. For reference, the original P2, \(C=1\) Full coverages for TPR/ROC/AUC were 0.934/0.936/0.942 in DGP-A and 0.930/0.904/0.952 in DGP-B.

\begin{table*}[t]
  \centering
  \caption{Unified Study 3 ablation. Entries are
  bias (percentage points)/empirical SD/mean SE/95\% coverage.}
  \label{tab:sim_ablation_full}
  \scriptsize
  \setlength{\tabcolsep}{2.7pt}
  \begin{tabular}{llcccc}
    \toprule
    DGP & Estimator & TPR(.5) & FPR(.5) & ROC(.1) & AUC \\
    \midrule
    DGP-A & Plug-in
      & \(-.15/.010/.002/.260\) & \(-.02/.010/.002/.258\)
      & \(-.09/.012/.001/.112\) & \(-.08/.008/{<}.001/.100\) \\
    DGP-A & HL-only
      & \(-.18/.012/.010/.904\) & \(+.02/.012/.010/.892\)
      & \(+.02/.018/.017/.940\) & \(-.05/.010/.010/.942\) \\
    DGP-A & AL-only
      & \(-.09/.016/.014/.936\) & \(-.14/.016/.015/.926\)
      & \(+.13/.014/.006/.616\) & \(+.03/.008/.001/.298\) \\
    DGP-A & Full
      & \(-.12/.017/.017/.956\) & \(-.10/.017/.018/.950\)
      & \(+.25/.019/.018/.938\) & \(+.06/.010/.010/.926\) \\
    DGP-A & Oracle nuis.
      & \(-.15/.017/.017/.960\) & \(-.08/.017/.018/.948\)
      & \(+.16/.018/.018/.934\) & \(-.01/.010/.010/.946\) \\
    \addlinespace
    DGP-B & Plug-in
      & \(+.28/.012/.002/.296\) & \(+.17/.007/.001/.234\)
      & \(-.03/.013/.001/.138\) & \(-.13/.009/.001/.104\) \\
    DGP-B & HL-only
      & \(+.26/.013/.008/.800\) & \(+.19/.008/.007/.874\)
      & \(-.14/.019/.017/.928\) & \(-.08/.013/.011/.914\) \\
    DGP-B & AL-only
      & \(-.10/.020/.019/.936\) & \(-.10/.012/.010/.928\)
      & \(+.06/.017/.011/.780\) & \(+.05/.009/.005/.688\) \\
    DGP-B & Full
      & \(-.12/.021/.021/.944\) & \(-.09/.012/.012/.950\)
      & \(-.05/.021/.020/.936\) & \(+.10/.013/.012/.934\) \\
    DGP-B & Oracle nuis.
      & \(+.05/.022/.021/.948\) & \(+.06/.012/.012/.958\)
      & \(-.10/.021/.020/.942\) & \(+.04/.013/.012/.936\) \\
    \bottomrule
  \end{tabular}
\end{table*}

The plug-in estimator's principal failure is severe standard-error underestimation rather than large point-estimation bias. Full has SE/SD between 0.927 and 1.036 across the six main-display cells, whereas the omitted corrections produce target-specific failures. In particular, AL-only misses most uncertainty for global ROC/AUC, while HL-only is weakest for DGP-B TPR. The oracle-nuisance comparison shows that the remaining finite-sample error is similar to that obtained when nuisance functions are known.

\begin{figure*}[t]
  \centering
  \includegraphics[width=\textwidth]{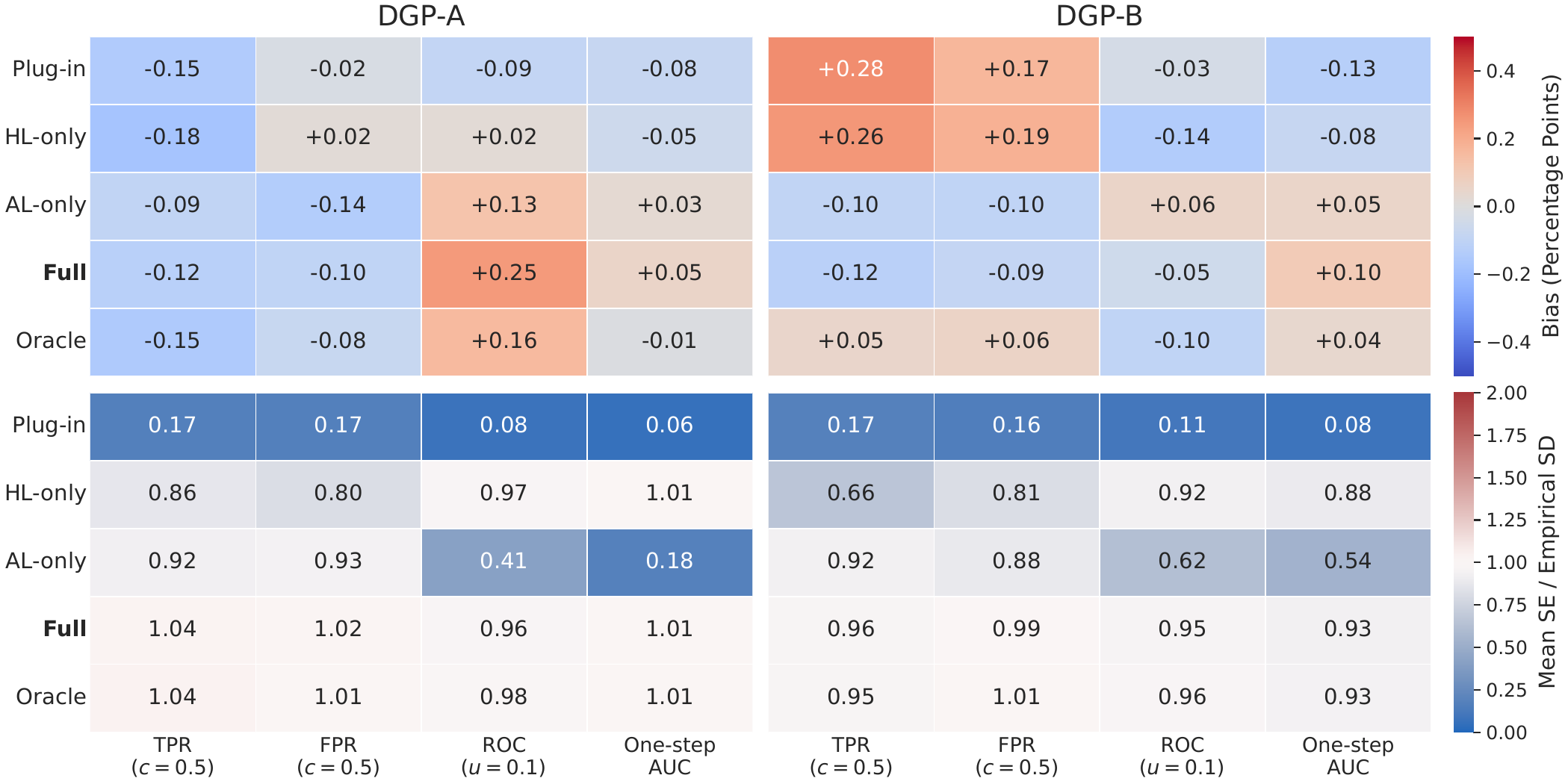}
  \caption{Confirmed P4 Study 3 diagnostics at \(c=0.5\), \(u=0.1\), and
  \(C=0.5\). The top row reports bias in percentage points; the bottom row
  reports mean analytical SE divided by empirical SD.}
  \label{fig:sim_ablation_appendix}
\end{figure*}

\begin{figure}[t]
  \centering
  \includegraphics[width=\linewidth]{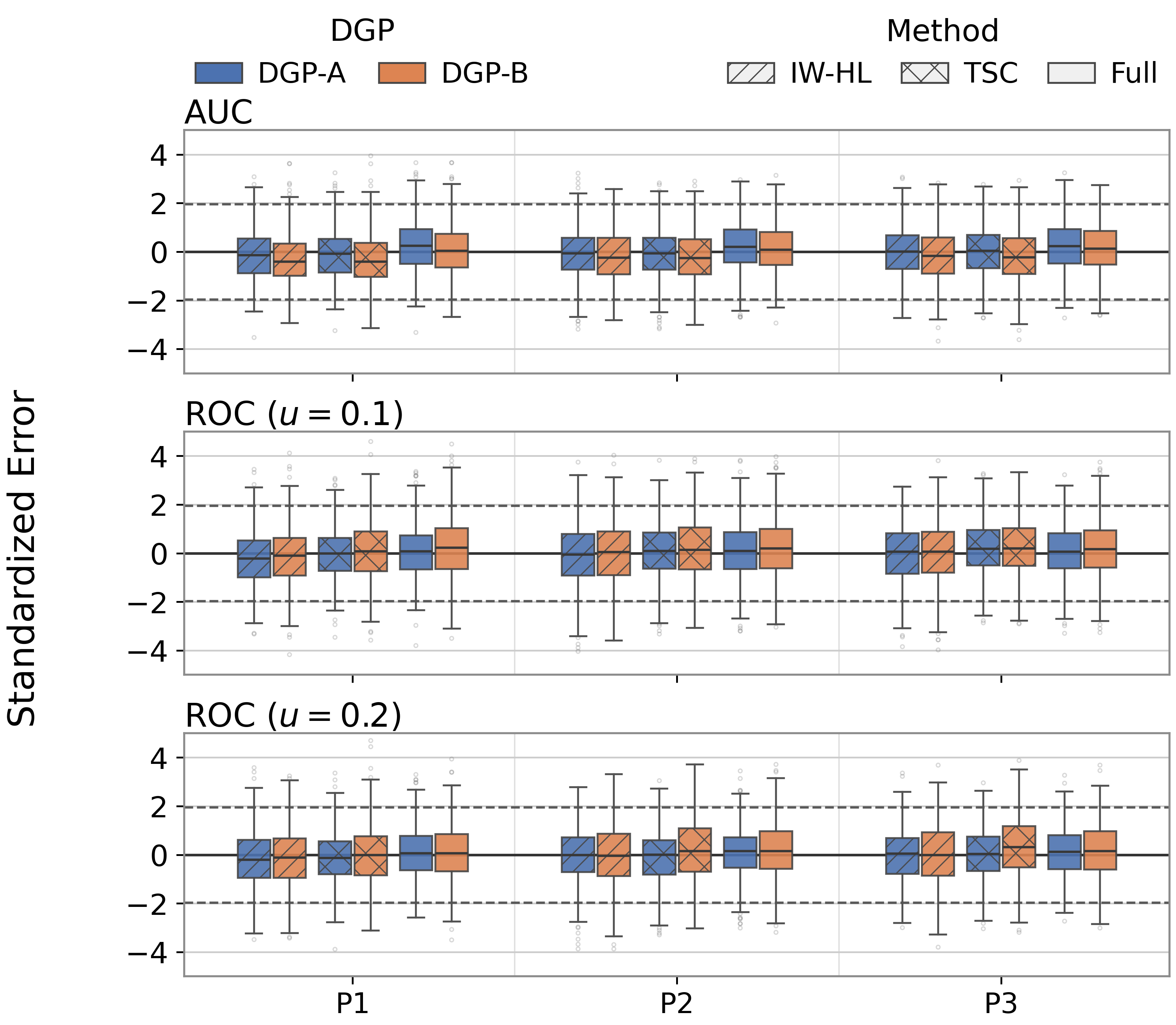}
  \caption{Study 4 standardized AUC and fixed-\(u\) ROC errors.}
  \label{fig:sim_roc_auc}
\end{figure}

\subsection{Regularized Quadratic-Dictionary Diagnostic}
\label{app:flexible_ml_diagnostic}

The additional diagnostic uses DGP-D at \(4\times\mathrm{P2}\) with the regularized quadratic-dictionary learner specified above. Because the complete degree-two dictionary contains the true DGP-D mechanisms, this is a correctly specified model-class experiment without oracle term selection, rather than evidence of robustness to arbitrary machine-learning nuisance fits. For the central and global targets TPR\((0.5)\), FPR\((0.5)\), ROC\((0.1)\), ROC\((0.2)\), and AUC, coverage is \(0.922\)--\(0.946\), SE/SD is \(0.919\)--\(0.972\), and absolute bias/SD is at most \(0.349\). This supplies a configuration without oracle term selection and with useful finite-sample coverage and SE/SD. It does not overturn the limitation: TPR\((0.2)\) and FPR\((0.2)\) cover only \(0.824\) and \(0.854\), respectively, and the HGB results for DGP-A/B remain substantially weaker. Figure~\ref{fig:sim_flexible_ml} displays every evaluated coordinate rather than only the favorable ones.

As an auxiliary diagnostic, we also proportionally enlarged the difficult DGP-B design. Larger samples did not uniformly eliminate lower-threshold FPR undercoverage by \(4\times\)P2; the remaining error was driven more by an optimistic analytical SE than by point-estimation bias. We therefore treat low-threshold FPR undercoverage as a finite-sample limitation rather than use this diagnostic as primary evidence for the method.

\begin{figure*}[t]
  \centering
  \includegraphics[width=\linewidth]{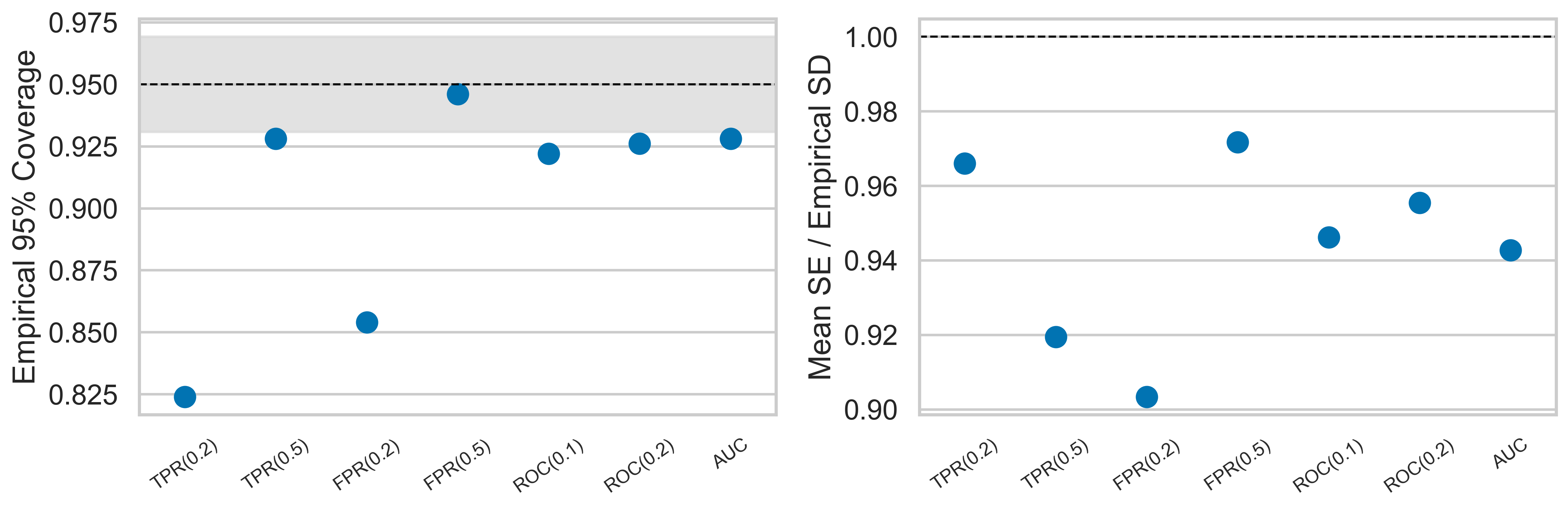}
  \caption{Regularized full quadratic-dictionary diagnostic under DGP-D over
  500 replications. The \(4\times\)P2 design is
  \((n_{\mathrm{HL}},n_{\mathrm{AL}},n_{\mathrm T})
  =(4000,16000,80000)\). Gray coverage bands show the Monte Carlo 95\%
  neighborhood of nominal 0.95 coverage; dashed lines mark the ideal value
  in each panel. All coordinates, including the weaker lower-threshold
  results, are displayed.}
  \label{fig:sim_flexible_ml}
\end{figure*}

\section{Application Details and Additional Results}
\label{app:application_details}

This appendix emphasizes reproducible construction, implementation choices, and numerical diagnostics following Section~\ref{sec:real_data}.

\subsection{Chatbot Arena Application: Additional Details}
\label{app:arena_temporal}
\suppressfloats[t]

\subsubsection{Data Linkage, Cohorts, and Masking}
\label{app:arena_temporal_data}

\paragraph{Exact linkage and analysis unit.}
We matched the original Arena-33K conversation release
\citep{lmsys2023arena33k} to a public timestamped copy at a fixed revision
\citep{agieai2023arena33kmirror} using exact question IDs, and then linked the
existing six-judge label file by question ID. A separate masking script
checks question ID, model A/B, human winner, anonymous user, language,
turn, and full-conversation hashes. The
linkage also verifies all fixed
source-file hashes, unique IDs, exact A/B conversation agreement, and the
absence of fuzzy or content-based matches.  Before assigning calendar
windows, a fixed hash keeps one comparison per anonymized user globally.
Consequently, question IDs and user proxies are disjoint across cohorts.
The common-family rule is computed from \(X\) and calendar windows only; the
declared target analysis population retains 1,499 of
3,027 later decisive user-level
comparisons (49.52\%).

\paragraph{Label visibility.}
The early HL file contains only \(X,Y\), and the middle AL file contains only
\(X,S\).  Here \(S\) is a retrospectively constructed six-judge surrogate label, not a
contemporaneous judge label.  AL eligibility never uses the human winner or
human tie status.

\paragraph{Surrogate-label construction.}
The fixed panel is
\path{Qwen/Qwen3-235B-A22B-Instruct-2507-tput},
\path{kimi-k2-0905-preview},
\path{meta-llama/Llama-4-Maverick-17B-128E-Instruct-FP8},
\path{mistralai/Mistral-7B-Instruct-v0.1},
\path{openai/gpt-oss-20b}, and
\path{zai-org/GLM-4.5-Air-FP8}.
After mapping a presented vote to canonical Arena A/B coordinates, an A vote
contributes \(+c\), a B vote contributes \(-c\), and an unknown preference
contributes zero, where \(c\) is confidence clipped to \([0,1]\); a missing or
nonfinite confidence for a directional A/B vote receives the prespecified
unit weight.  The sign of the
six-vote mean defines \(S\), while an exact-zero mean is unresolved and never
encoded as response B.  Of 32,922 X-eligible questions, 27,482 (83.48\%) have
all six records and 27,354 (83.09\% overall; 99.53\% of complete panels) have
a nonzero aggregate; 442 individual votes use the missing-confidence
fallback. We find no pair mismatch, duplicate judge key, or
within-question multiple orientation. A separate masking script uses only decisive status to define the
binary target validation population, encrypts its human outcomes in frozen ID
order, and exports only target \(X\) to the analysis directory.
Table~\ref{tab:arena_temporal_cohorts} and
Algorithm~\ref{alg:arena_temporal_workflow} summarize the resulting analysis design.

\begin{table}[t]
\centering
\caption{Fixed temporal cohorts. Intervals are UTC and follow the
left-closed, right-open convention except at the observed endpoint.}
\label{tab:arena_temporal_cohorts}
\footnotesize
\setlength{\tabcolsep}{3.0pt}
\begin{tabular}{lccc}
\toprule
Cohort & Calendar window & \(n\) & Visible labels \\
\midrule
HL & Apr.~24--May~7, 2023 &
2,816 & \(X,Y\) \\
AL & May~7--26, 2023 &
3,590 & \(X,S\) \\
Target & May~26--Jun.~26, 2023 &
1,499 & \(X\) before reveal \\
\bottomrule
\end{tabular}
\end{table}

\begin{algorithm}[t]
\caption{Temporal Chatbot Arena analysis workflow}
\label{alg:arena_temporal_workflow}
\footnotesize
\begin{algorithmic}[1]
\STATE Fix and verify conversation, timestamp, and six-judge files by SHA-256
\STATE Join exact question IDs and map every decision to standard Arena A/B
\STATE Hash-select one comparison per anonymous user before window assignment
\STATE Apply the outcome-blind common-family support rule
\STATE Export HL \((X,Y)\), AL \((X,S)\), and target \(X\) into disjoint files
\STATE Cross-fit nuisances and source-to-target ratios on five fixed folds
\STATE Freeze estimates, scores, and diagnostics and save their hashes
\STATE Verify the saved pre-reveal results, reveal target \(Y\), and evaluate the frozen scores
\end{algorithmic}
\end{algorithm}

\subsubsection{Natural Temporal Shift and Overlap}
\label{app:arena_temporal_shift}

No artificial probability-proportional-to-size (PPS) sampling or covariate tilt is used.  Calendar time determines only the early, middle, and later windows and is excluded from both the preference and domain models.  Figure~\ref{fig:arena_temporal_shift}(a) reports every retained candidate-family share.  The corresponding changes in language, prompt/response length, task/style indicators, and their Jensen--Shannon summaries are retained in the supplementary output files. Figure~\ref{fig:arena_temporal_shift}(b) and Table~\ref{tab:arena_temporal_overlap} report the out-of-fold (OOF) source-to-target domain AUC and uncapped density-ratio diagnostics; balance is summarized by standardized mean differences (SMDs).  Ratios are normalized using source-training predictions only; the held-out source fold is never forced to mean one.

\begin{figure}[t]
\centering
\IfFileExists{Figures/arena33k_temporal_shift.pdf}{%
\includegraphics[width=\columnwidth]{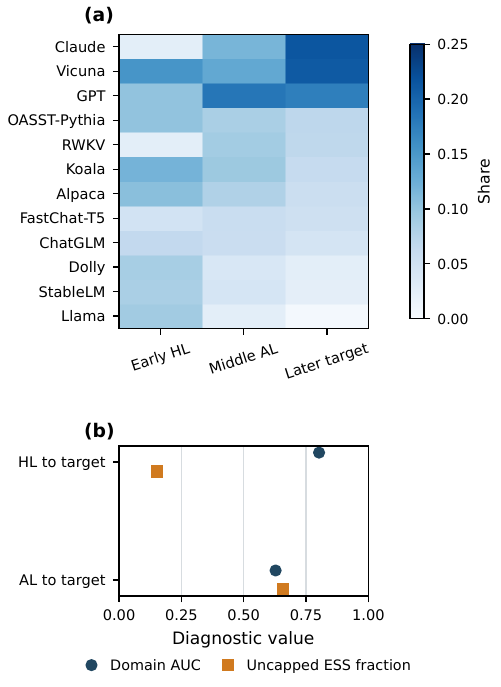}
}{%
\fbox{\parbox[c][4.8cm][c]{0.94\columnwidth}{\centering
Figure unavailable in this build\\
\texttt{arena33k\_temporal\_shift.pdf}}}
}
\caption{Outcome-blind temporal-shift diagnostics. (a) Candidate-family shares
across the fixed calendar cohorts. (b) OOF domain discrimination and
uncapped effective-sample-size fractions for each transport. Domain AUC
measures separability and is not a positivity proof. These diagnostics show
that the calendar cohorts differ in observed covariates without imposed
tilting; they do not assess estimator accuracy.}
\label{fig:arena_temporal_shift}
\end{figure}

\begin{table}[t]
\centering
\caption{Outcome-blind overlap diagnostics for the semantic-plus-metadata
density design. Primary ratios are uncapped; caps are sensitivities only.}
\label{tab:arena_temporal_overlap}
\scriptsize
\setlength{\tabcolsep}{2.0pt}
\begin{tabular}{@{}p{0.46\columnwidth}cc@{}}
\toprule
Diagnostic & HL \(\to\) T & AL \(\to\) T \\
\midrule
Domain AUC & 0.802 &
0.628 \\
Uncapped ESS & 427.9 &
2,356.7 \\
ESS/source \(n\) & 0.152 &
0.656 \\
Weight \(p_{99}\) & 9.43 &
3.69 \\
Maximum weight & 55.12 &
7.40 \\
Guard fraction & 0.000 &
0.000 \\
Weighted mean \(|\mathrm{SMD}|\) &
0.039 &
0.020 \\
\bottomrule
\end{tabular}
\end{table}

\subsubsection{Content-Aware Representation}
\label{app:arena_temporal_features}

The primary encoder is the fixed \path{intfloat/multilingual-e5-small} model (384 dimensions) \citep{wang2024multilingual,intfloat2023multilinguale5small}. Its exact model revision is \texttt{fd1525a9fd15316a2d503bf26\allowbreak ab031a61d056e98}. Prompt, response A, and response B are encoded separately with fixed query/passage prefixes, masked-mean pooling, L2 normalization, and deterministic 512-token head--tail truncation.  The observed truncation rates are 1.01\%, 6.41\%, and 6.78\%.  Preference features contain signed response contrasts and named signed metadata; density features contain only swap-invariant combinations.  Training-fold A/B augmentation makes fitted preference scores complementary up to the recorded numerical tolerance. Dense embedding coordinates are not individually interpreted: substantive interpretability comes from the named language, family, length, style, and task features and the metadata-only sensitivity.

\subsubsection{Estimation and Validation Objects}
\label{app:arena_temporal_estimation}

Five fixed folds are shared across sources.  Each held-out observation is
predicted once by nuisance models that use only the corresponding training
folds: \(m_Y\) uses HL, \(m_S\) uses AL, and each density ratio uses its source
and target training folds.  The 790-dimensional preference nuisances are
no-intercept logistic regressions with A/B augmentation and training-only
scaling without centering; principal component analysis (PCA) is not used
for \(m_Y\) or \(m_S\).
Two grouped inner folds select
\(C\in\{0.03,0.1\}\) and elastic-net mixing
\(\alpha\in\{0,0.5\}\) by source log loss inside each outer fold
(\(\alpha=0\) is L2); all ten outer-fold preference fits and the all-AL
deployment fit select \(C=0.03,\alpha=0.5\).
Each 1,224-dimensional density-ratio design uses training-only standardization,
randomized PCA capped at 128 components, and L2 logistic regression with
\(C=0.1\), posterior guard \(10^{-6}\), and a normalizer computed only from
source-training predictions.  All fits use at most 5,000 iterations with
tolerance \(10^{-5}\).  The outer-fold salt is
\path{arena33k-temporal-v4-fold-20260725}, the inner-fold salt is
\path{arena33k-temporal-v4-inner-20260725}, and the numeric seed is 20260725.
The primary endpoint is the scalar one-step AUC of the population
surrogate-derived model under the target law. We separately report
\begin{align*}
\widehat\theta_{\mathrm{CF}}
&=\operatorname{AUC}_{\mathrm T}
  \{Y_i,\widehat m_{S,-k(i)}(X_i)\},\\
\widehat\theta_{\mathrm{deploy}}
&=\operatorname{AUC}_{\mathrm T}
  \{Y_i,\widehat m_{S,\mathrm{all\text{-}AL}}(X_i)\}.
\end{align*}
The first evaluates pooled fold-specific scores and the second evaluates one
model fitted to all AL observations. Target \(S\) is not used. These quantities
describe the Target performance of actual fitted scores, but neither equals
the population-model estimand
\(\operatorname{AUC}_{\mathrm T}\{Y,m_S^\star(X)\}\).

\paragraph{Comparator definitions.}
The outcome-regression (g-computation) row plugs the cross-fitted human-outcome
regression into the target empirical law.  The importance-weighted human-score
row transports the empirical AUC of the frozen surrogate-derived model using only HL
outcomes and the HL-to-target density ratio.  The human-label augmentation
adds only the HL residual and prevalence corrections; the LLM-label
augmentation adds only the AL surrogate residual correction.  Full combines
the target plug-in with both source corrections and is the only row carrying
the theorem-covered three-source interval.  The first four rows are
prespecified analytic diagnostics or ablations, not formal method-ranking
tests.

\begin{table}[t]
\centering
\caption{Chatbot Arena scalar-AUC results. Panel A concerns the population
surrogate-score estimand. Panel B evaluates frozen finite-sample scores against
human Target outcomes revealed after estimation. The panels have different
estimands and should not be read as a method ranking.}
\label{tab:arena_temporal_all_auc}
\footnotesize
\setlength{\tabcolsep}{2.6pt}
\begin{tabular}{@{}p{0.61\columnwidth}c@{}}
\toprule
Object & AUC [95\% interval] \\
\midrule
\multicolumn{2}{@{}l}{\emph{A. Population-score estimation}}\\
Outcome regression (g-computation) &
0.796~[0.786,0.807] \\
Importance-weighted human-label score &
0.841~[0.803,0.878] \\
Human-label correction only &
0.825~[0.793,0.857] \\
LLM-label correction only &
0.819~[0.805,0.833] \\
\textbf{Full three-source (ours)} &
\textbf{0.848~
[0.814,0.881]} \\
\midrule
\multicolumn{2}{@{}l}{\emph{B. Finite-score Target-\(Y\) validation}}\\
Cross-fitted learned score &
0.810~[0.787,0.832] \\
All-AL deployment score &
0.812~[0.789,0.834] \\
\bottomrule
\end{tabular}
\end{table}

\subsubsection{Numerical Diagnostics and Sensitivities}
\label{app:arena_temporal_diagnostics}

Table~\ref{tab:arena_temporal_nuisance} gives OOF nuisance quality, score-tie checks, and the actual threshold-local AL information.  Merely observing \(n_{\mathrm{AL}}h=463.8\) is not treated as sufficient: the local count is the number of AL scores within one \(h\) of the threshold, whereas the Gaussian-kernel ESS uses its full-support weights and may exceed that count. We report pointwise inference only if the local count, kernel ESS, maximum local share, weighted kernel mass, residual variation, projection, and foldwise score diagnostics all satisfy the stated criteria.  The primary scalar analysis uses uncapped weights; Figure~\ref{fig:arena_temporal_sensitivity} reports the prespecified weight caps and AUC pilot-bandwidth perturbations; none was used to choose the primary analysis.

\begin{table}[t]
\centering
\caption{Nuisance and local-information diagnostics. Pointwise inference is
reported only when every stated criterion is satisfied.}
\label{tab:arena_temporal_nuisance}
\footnotesize
\setlength{\tabcolsep}{2.8pt}
\begin{tabular}{lcc}
\toprule
Component & Diagnostic & Value/status \\
\midrule
\(\hat m_Y\) (HL) & OOF AUC / Brier &
0.814 / 0.178 \\
\(\hat m_S\) (AL) & OOF AUC / Brier &
0.869 / 0.149 \\
\(\hat m_S\) & Swap-complement error &
\(1.5\times 10^{-16}\) \\
Target score & Unique / \(n\); max. atom &
1,499/1,499 / 0.00067 \\
ROC(.10) & Local count / kernel ESS &
734 / 1,029.0 \\
ROC(.15) & Local count / kernel ESS &
632 / 905.9 \\
ROC(.20) & Local count / kernel ESS &
612 / 821.9 \\
\bottomrule
\end{tabular}
\end{table}

The finite-score rows and their paired differences use 2,000 Target-only nonparametric bootstrap resamples with both score vectors fixed.  They include no nuisance, density-ratio, model-training, or fold/split refitting uncertainty, and are conditional validation diagnostics rather than formal three-source intervals.

\begin{table}[t]
\centering
\caption{Paired comparisons and cross-estimand discrepancies. Analytic
population-model comparisons and conditional target-only bootstrap
comparisons have different uncertainty scopes.}
\label{tab:arena_temporal_paired}
\scriptsize
\setlength{\tabcolsep}{2.0pt}
\begin{tabular}{@{}p{0.46\columnwidth}cc@{}}
\toprule
Comparison & Estimate & 95\% interval \\
\midrule
Full \(-\) outcome regression & 0.051 &
[0.020,0.083] \\
Full \(-\) human-label correction & 0.023 &
[0.014,0.031] \\
All-LLM-label deployment \(-\) cross-fitted score &
0.002 &
[-0.002,0.007] \\
Cross-fitted score \(-\) Full & \(-0.038\) &
[-0.061,-0.016] \\
All-LLM-label deployment \(-\) Full &
\(-0.036\) &
[-0.058,-0.014] \\
\bottomrule
\end{tabular}
\end{table}

\begin{figure}[t]
\centering
\IfFileExists{Figures/arena33k_temporal_roc.pdf}{%
\includegraphics[width=\columnwidth]{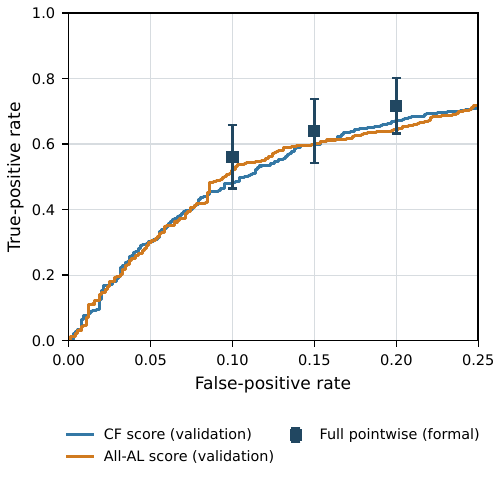}
}{%
\fbox{\parbox[c][3.4cm][c]{0.94\columnwidth}{\centering
Figure unavailable in this build\\
\texttt{arena33k\_temporal\_roc.pdf}}}
}
\caption{Low-FPR target ROC results. Solid curves are the two frozen
finite-score empirical validations; square markers and error bars are the
Full population-model estimates, for which all stated pointwise diagnostics
were satisfied. The figure checks operating-point scale and behavior, but the curves
and markers target different estimands and therefore do not rank methods or
measure bias.}
\label{fig:arena_temporal_roc}
\end{figure}

\begin{figure}[t]
\centering
\IfFileExists{Figures/arena33k_temporal_sensitivity.pdf}{%
\includegraphics[width=\columnwidth]{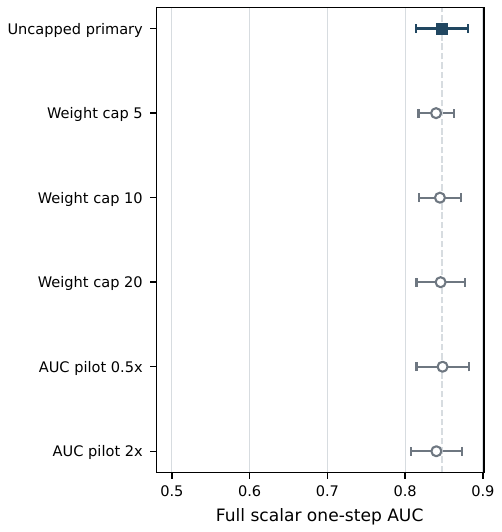}
}{%
\fbox{\parbox[c][3.4cm][c]{0.94\columnwidth}{\centering
Figure unavailable in this build\\
\texttt{arena33k\_temporal\_sensitivity.pdf}}}
}
\caption{Prespecified scalar-AUC sensitivities for the Full estimator.
The uncapped analysis is primary; the capped-weight and pilot-bandwidth
results are sensitivity analyses specified before Target reveal. Their
stability concerns numerical sensitivity only and does not establish
identification or transportability.}
\label{fig:arena_temporal_sensitivity}
\end{figure}

\subsubsection{Provenance and Scope}
\label{app:arena_temporal_limitations}

The existing judge-label release preserves the six judge identifiers, presented order, structured preference/confidence decisions, and question IDs, so the A/B mapping and panel aggregation are reproducible. It does not preserve the complete prompt template, provider and call dates, exact served model versions, raw responses, or retry/failure logs; most judge models also postdate the 2023 traffic, so training-data overlap cannot be ruled out. Each comparison was judged in one randomized orientation rather than in both A/B and B/A orders. We therefore treat \(S\) as a retrospectively constructed six-model surrogate label and do not interpret this study as validation of a contemporaneous single judge. Finally, real temporal evolution may include conditional or concept drift not captured by \(X\); the analysis assumes, rather than establishes, transportability and applies only to the declared common-family overlap population.

\subsection{Construction and Curation of the Semi-Synthetic ACS Benchmark}
\label{realdataconstruction}

\paragraph{Data source, task, and estimand.}
We use the 2019 one-year American Community Survey Public Use Microdata Sample
person file for California (\texttt{psam\_p06.csv}). We retain records satisfying
\texttt{AGEP}$>16$, \texttt{PINCP}$>100$, \texttt{WKHP}$>0$, and
\texttt{PWGTP}$\geq1$, yielding 197{,}193 eligible individuals. The binary gold
outcome is $Y=\Ind\{\mathrm{PINCP}>\$50{,}000\}$; the surrogate is the same
threshold applied to LLM-predicted annual income.
The analysis excludes the ACS survey weight \texttt{PWGTP}; accordingly, the target law is the unweighted ACS sample target distribution rather than the survey-weighted California population.

The target cohort retains the continuous \texttt{PINCP} value in a verification-only column. This column is removed before preprocessing and is unavailable to every nuisance learner, density-ratio model, cross-fitting operation, and estimating equation. It is converted to the held-out benchmark outcome only after all cross-fitted predictions have been completed.

\paragraph{Fixed candidate pools and LLM labeling.}
Before cohort construction, we fix 11{,}000 identifiers: 1{,}000 gold-pool
candidates and a 10{,}000-ID unlabeled expansion reserve. The prespecified
current expansion requested 200 gold IDs (which were used exclusively for initial prompt evaluation) and only the first 1{,}000 unlabeled IDs in the fixed permutation. All 200 gold requests and 999 unlabeled requests
produced schema-valid records; one unlabeled request produced a logged terminal
schema violation (a zero income prediction). The other 9{,}000 reserve IDs were
never requested and therefore are not missing LLM responses. Operational errors
were retained in the attempt log. Sixteen nonterminal transport or network
errors were resumed in later explicit sessions, and all affected request IDs
subsequently completed.

Calls were made through OpenRouter on July 13 and 18, 2026, using the provider-reported identifier \texttt{google/gemini-3.5-flash}, prompt version \texttt{P2-socioeconomic-role-v1}, temperature zero, and a strict three-field JSON schema. OpenRouter did not expose a finer immutable model fingerprint, so we do not claim one. The released supplement contains the frozen request manifest, every attempt and raw response, parsed outputs, the terminal-failure record, retry status, token and cost ledger, and SHA-256 hashes. Invalid model content was not automatically regenerated; recoverable transport or service errors were logged and resumed.

\paragraph{Exact pointwise prompt.}
The wording of the complete template is reproduced below, with
\texttt{\{person\_profile\}}
replaced by a readable profile that excluded \texttt{PINCP}, the gold outcome,
survey weight, and dataset identity:
\begin{quote}\small\raggedright
\texttt{You are a socioeconomic analyst estimating annual personal income from
structured survey profiles.}\\
\texttt{Task: Estimate this person's annual personal income in U.S. dollars.}\\
\texttt{Return JSON only, with exactly these fields:
predicted\_annual\_\allowbreak income\_usd,
income\_\allowbreak rank\_score, and confidence.}\\
\texttt{Definitions: predicted\_annual\_\allowbreak income\_usd: your best
estimate of this person's annual personal income in U.S. dollars.
income\_\allowbreak rank\_score: integer
from 0 to 10, where 0 means very low expected income and 10 means very high
expected income relative to U.S. adults. confidence: numeric confidence from 0
to 100, where 0 means not confident and 100 means very confident. This is your
confidence in the income prediction, not the income level.}\\
\texttt{Person profile: \{person\_profile\}}\\
\texttt{Rules: Return valid JSON only. Do not include explanation text. Do not
output ranges; output a single numeric income estimate.}
\end{quote}
The continuous prediction was thresholded at \$50{,}000 to form $S$.

\paragraph{Age-biased source sampling.}
We induce transparent opposing shifts at age 45. An older candidate enters HL
with probability 0.8 and a younger candidate with probability 0.2; AL reverses
these probabilities. All draws use seed 20260720, and the remaining eligible
individuals form the target cohort:
\begin{table}[t]
\centering
\caption{ACS analytical cohorts after opposing age-biased sampling.}
\label{tab:acs_cohorts}
\small
\setlength{\tabcolsep}{3.5pt}
\resizebox{\columnwidth}{!}{%
\begin{tabular}{lrrrr}
\toprule
Cohort & $n$ & Mean age & Age $>45$ & Label prevalence \\
\midrule
HL & 447 & 50.79 & 75.6\% & $227/447=0.508$ \\
AL & 552 & 35.78 & 15.4\% & $244/552=0.442$ \\
Target & 186,193 & 42.93 & 43.8\% & $81,674/186,193=0.4387$ \\
\bottomrule
\end{tabular}
}
\end{table}
Direct identifier comparisons give zero overlap for every pair of analytical cohorts.
The opposing shifts require the two source-specific density ratios defined in
Appendix~\ref{densityratioestimation}.

\begin{algorithm}[t]
\caption{ACS-Income label-masking workflow}
\label{alg:acs_workflow}
\small
\begin{algorithmic}[1]
\STATE Freeze candidate IDs and request only the prespecified LLM-label subset
\STATE Validate the response schema and retain the complete request/attempt log
\STATE Form disjoint age-shifted HL, AL, and target cohorts
\STATE Cross-fit outcome, surrogate, and density-ratio learners with training-fold normalization
\STATE Estimate scalar AUC and pointwise ROC diagnostics with target \texttt{PINCP} masked
\STATE Freeze estimates, reveal verification-only target outcomes, and compute the CF benchmark
\end{algorithmic}
\end{algorithm}

\subsection{Covariates and Preprocessing}
\label{sec:add_numerical}
\label{inclusionfeature}

The common feature set contains numerical variables \texttt{AGEP} and \texttt{WKHP} and categorical variables \texttt{COW}, \texttt{SCHL}, \texttt{MAR}, \texttt{OCCP}, \texttt{POBP}, \texttt{RELSHIPP}, \texttt{SEX}, and \texttt{RAC1P}. Within each outer cross-fitting fold, the numerical variables are standardized using the training observations. The categorical variables are processed using sparse \texttt{OneHotEncoder} with no reference category dropped and with unknown evaluation categories ignored. Category levels are frozen across the three cohorts. The categorical variables contain 8, 24, 5, 529, 220, 19, 2, and 9 observed levels, respectively. The common linear design used by the outcome and density-ratio models has 818 columns. For the surrogate regression, each of \texttt{AGEP} and \texttt{WKHP} is replaced by a three-knot cubic quantile-spline basis, giving 824 columns.

\subsection{Cross-Fitting and Density-Ratio Estimation}
\label{supprealmain}

We use five-fold honest cross-fitting. The outcome regression \(\hat m_Y\) uses a 100-tree Random Forest. The surrogate regression \(\hat m_S\) uses L2-penalized cubic-spline logistic regression. Within each outer AL training fold, three-fold cross-validation compares \(C\in\{0.03,0.3,3\}\) and L1 ratios \(\{0,0.5\}\) by log loss; all five folds select \(C=3\) and L1 ratio zero. Extending the grid to \(C=30\) leaves every selection and reported result unchanged. Domain learners remain calibrated 100-tree Random Forests. Every observation is evaluated using nuisance fits that exclude its corresponding held-out fold.

The source sample is the information bottleneck for each density-ratio problem. Within each outer fold, we therefore draw a seeded source-size-matched subset from the target training fold and fit the domain classifier to balanced source and target classes. The Random Forest domain learner uses 100 trees and minimum leaf size 10. Its probabilities are calibrated using internal three-fold sigmoid calibration. Under the balanced domain design, if $\hat\pi_g(x)$ denotes the fitted probability of target membership, the posterior odds estimate the density ratio: the implementation converts $\hat\pi_g(x)$ to posterior odds as in Appendix~\ref{densityratioestimation}. Domain posteriors are prespecified to be clipped to $[0.001,0.999]$. Subject to an upper bound of 10, a single multiplicative normalization factor is estimated using only the source observations in the outer training fold and then frozen before weights are evaluated on the held-out fold. Thus, an evaluation observation and its weight depend only on training-fold fits; in particular, we do not force the held-out mean weight to equal one.

In the primary analysis, neither posterior nor weight clipping is activated.
\begin{table}[t]
\centering
\caption{ACS held-out density-ratio weight diagnostics.}
\label{tab:acs_weight_diagnostics}
\begin{tabular}{lrrrr}
\toprule
Ratio & Median & Maximum & ESS & ESS fraction \\
\midrule
$\hat w_{\mathrm{HL}}$ & 0.816 & 4.730 & 332.8 & 0.745 \\
$\hat w_{\mathrm{AL}}$ & 0.864 & 3.452 & 438.2 & 0.794 \\
\bottomrule
\end{tabular}
\end{table}
The pooled held-out mean weights are 1.056 and 1.032, respectively, rather than
being mechanically set to one. Across the 818 encoded numerical and one-hot
indicator features, reweighting reduces the maximum absolute standardized mean
difference from 0.558 to 0.156 for HL and from 0.536 to 0.187 for AL; the
remaining imbalance is a limitation, not evidence of exact overlap.
Out-of-fold nuisance diagnostics are AUC/Brier
score/log loss $=0.816/0.179/0.534$ for $\hat m_Y$ and
$0.948/0.087/0.297$ for $\hat m_S$.

\subsection{Pointwise ROC Implementation and Diagnostics}
\label{app:acs_pointwise}

We prespecify $h=0.5n_{\mathrm{AL}}^{-1/4}=0.10315$ and evaluate TPR and FPR on 201 thresholds. The two sequences are separately projected by bounded least-squares isotonic regression onto $[0,1]$ to be nonincreasing, and combined by generalized inversion and linear interpolation. The slope-ratio pilot uses $b=n_{\mathrm T}^{-1/5}=0.08831$. Pointwise influence values, analytic standard errors, and nominal Wald intervals are computed exactly as in Theorem~\ref{thm:pointwise_roc} and Corollary~\ref{cor:roc_se}; repeating those formulae here would add no application-specific information.

The nominal scale $n_{\mathrm{AL}}h=56.94$ is not interpreted as an effective sample size. At target FPR $u=0.10,0.15,0.20$, transport-weighted local kernel ESS values are 150.7, 97.3, and 64.8, with 87, 49, and 37 AL observations within $\lvert\hat m_S(X)-\hat c_u\rvert\leq h$. The pooled cross-fitted Target score has 183,778 distinct values among 186,193 observations and an ordered-pair tie probability of $1.63\times10^{-7}$. The spline basis avoids the original coarse fitted score, but these diagnostics concern the fitted score \(\hat m_S\), not the population score \(m_S^\star(X)\). Because the ACS covariates are largely integer- or category-valued, they do not verify population-score continuity; we interpret the analytic interval conditional on the stated population regularity. The scalar one-step AUC remains the primary endpoint; pointwise ROC results are supplementary because they are less stable in this small-AL application.

\subsection{Scalar AUC Numerical Implementation}
\label{app:acs_scalar_auc}

The implementation follows the scalar construction in Appendix~\ref{FPRROCAUC} and the three-source influence functions in Appendix~\ref{ROCAUCCorollary}. It uses the tie-adjusted comparison $\kappa(z,z')=\Ind\{z>z'\}+\tfrac12\Ind\{z=z'\}$ and computes the target pairwise term exactly in $O(n_{\mathrm T}\log n_{\mathrm T})$ time by sorting scores, avoiding an $n_{\mathrm T}\times n_{\mathrm T}$ matrix. The weighted score-density pilots use a Gaussian kernel with $b=n_{\mathrm T}^{-1/5}$ and an FFT evaluation.

The reported estimate adds the two source corrections to the cross-fitted target pairwise term and uses the augmented prevalence denominator, exactly as specified in Corollary~\ref{cor:auc_se}. Its standard error sums the within-source empirical variances of centered influence values. This scalar procedure contains no boundary bandwidth $h$ and uses analytic three-source inference rather than a component bootstrap.

\subsection{Cross-Fitted Target Benchmark}
\label{app:acs_target_benchmark}

\begin{table}[t]
\centering
\caption{ACS-Income estimates using the cross-fitted spline-3 logistic score. Full uses the
proposed three-source analytic interval under the stated regularity
conditions. CF values evaluate a different finite-sample score against
revealed Target outcomes.}
\label{tab:acs_main_metrics}
\scriptsize
\setlength{\tabcolsep}{2.2pt}
\resizebox{\columnwidth}{!}{%
\begin{tabular}{@{}lrrrr@{}}
\toprule
Metric & Plug-in & Human correction & Full [95\% interval] & CF validation \\
\midrule
AUC & 0.707 & 0.838 & 0.866 [0.822, 0.911] & 0.855 \\
TPR at FPR .10 & 0.293 & 0.536 & 0.585 [0.458, 0.713] & 0.577 \\
TPR at FPR .15 & 0.391 & 0.631 & 0.652 [0.532, 0.772] & 0.682 \\
TPR at FPR .20 & 0.472 & 0.692 & 0.725 [0.610, 0.841] & 0.754 \\
\bottomrule
\end{tabular}%
}
\end{table}

\begin{figure}[t]
\centering
\includegraphics[width=\columnwidth]{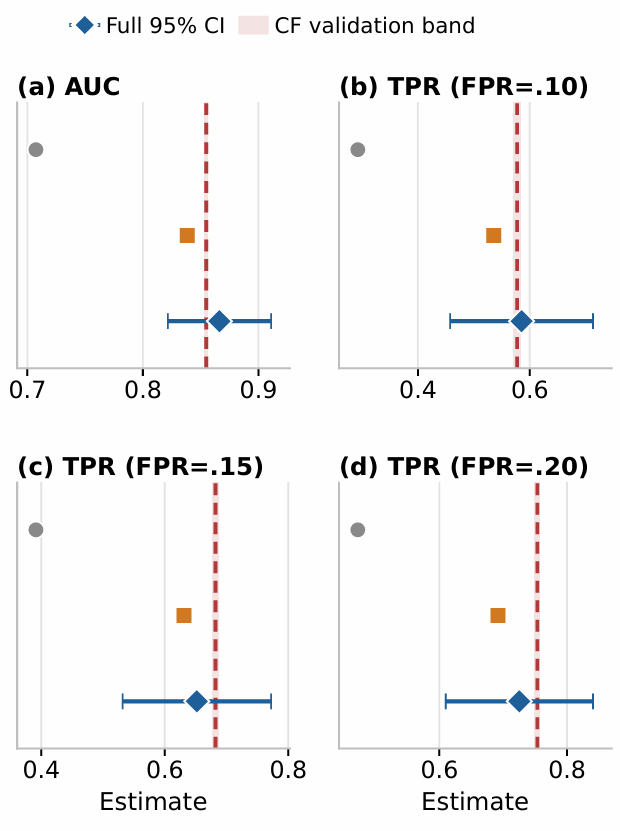}
\caption{Complete ACS comparison for the spline-3 score. Gray, orange, and blue
denote the plug-in, human-correction-only, and Full estimates. Only Full carries
the proposed three-source intervals; the ablations are shown as point estimates.
The red line and band evaluate pooled cross-fitted finite-sample scores after
Target-outcome reveal and do not represent population-score truth. Full is
descriptively closer than the human-correction-only ablation at all four
reported metrics.}
\label{fig:acs_spline_full_validation}
\end{figure}

The program first writes a pre-reveal results file containing the scalar-AUC and pointwise-ROC estimates, seed, model specification, bandwidths, and a hash of the ordered Target IDs. It then loads Target \texttt{PINCP} by ID. Target \(S\) is never loaded or used. The revealed outcomes compute the empirical Target AUC and ROC values of the pooled out-of-fold surrogate-derived models. The Target AUC uses the standard tie-adjusted empirical pairwise statistic. Its conditional standard error is computed from the positive- and negative-class empirical AUC influence values.

These two quantities are not the same estimand. The benchmark pools predictions from five fold-specific finite-sample models and is not the performance of one frozen deployable model, whereas the one-step estimator targets the population surrogate-derived model $m_S^\star$. Their difference is therefore an empirical cross-estimand discrepancy, not repeated-sampling bias and not a coverage assessment.

The full scalar one-step estimate is 0.8663 (SE 0.0228; analytic 95\% interval $[0.8216,0.9110]$ under the stated regularity conditions), whereas the cross-fitted Target benchmark AUC is 0.8549 ($\operatorname{SE}=0.0009$). Their absolute discrepancy is 0.0114. The corresponding plug-in and human-correction-only discrepancies are 0.1475 and 0.0166. The Full discrepancy is therefore 92.3\% smaller than the plug-in discrepancy and 31.5\% smaller than the human-correction-only discrepancy. More directly, the paired Full-minus-human correction is 0.0280 with 95\% interval $[0.0044,0.0515]$, showing that the auxiliary-label correction is not negligible in this application. We do not use nested subsample containment as coverage evidence and do not interpret a finite-model learning curve as evolution of this fixed population-model estimand; repeated-sampling calibration is assessed in the simulation studies.

\subsection{Exploratory Surrogate-Learner Sensitivity}
\label{app:acs_learner_sensitivity}

The six surrogate-score learners in Table~\ref{tab:acs_learner_sensitivity} were examined after the initial Target-outcome reveal, so this analysis is an exploratory sensitivity check rather than prespecified model selection. Their displayed order uses only AL out-of-fold log loss; neither Target outcomes nor Target-validation discrepancies enter the ranking. The spline-3 logistic learner has the lowest AL log loss and is therefore used for the main presentation. For every learner, the auxiliary-label correction increases AUC relative to the human-correction-only ablation, and each learner-specific paired interval excludes zero. Full is descriptively closer to its finite-score Target validation in five of six cases; calibrated quadratic-SVC is the counterexample. Thus, the sensitivity analysis supports a persistent, non-negligible auxiliary-source contribution, but it does not establish uniform estimator superiority or verify the population score-density assumption. Extending the elastic-net grid from \(C\leq3\) to \(C\leq30\) leaves the selected spline-3 fit and all reported results unchanged.

\begin{table}[t]
\centering
\caption{Post-reveal exploratory ACS surrogate-learner sensitivity. Learners are ordered by AL out-of-fold log loss. The paired intervals concern Full minus the human-correction-only ablation for the same learner; they are not simultaneous confidence intervals or a direct ranking of estimator accuracy.}
\label{tab:acs_learner_sensitivity}
\scriptsize
\setlength{\tabcolsep}{2.0pt}
\resizebox{\columnwidth}{!}{%
\begin{tabular}{@{}lrrrrr@{}}
\toprule
Learner & AL log loss & Full & Human corr. & CF validation & Full\(-\)Human [95\% interval] \\
\midrule
Spline-3 logistic & 0.2966 & 0.8663 & 0.8383 & 0.8549 & 0.0280 [0.0044, 0.0515] \\
Elastic-net logistic & 0.2992 & 0.8632 & 0.8405 & 0.8543 & 0.0226 [0.0034, 0.0419] \\
Spline-5 logistic & 0.3059 & 0.8638 & 0.8302 & 0.8517 & 0.0336 [0.0098, 0.0573] \\
Calibrated RBF-SVC & 0.3164 & 0.8618 & 0.8307 & 0.8486 & 0.0311 [0.0112, 0.0511] \\
Ridge logistic & 0.3178 & 0.8616 & 0.8354 & 0.8493 & 0.0261 [0.0058, 0.0464] \\
Calibrated quadratic-SVC & 0.3331 & 0.8745 & 0.8394 & 0.8468 & 0.0351 [0.0125, 0.0577] \\
\bottomrule
\end{tabular}%
}
\end{table}

\begin{figure}[t]
\centering
\includegraphics[width=\columnwidth]{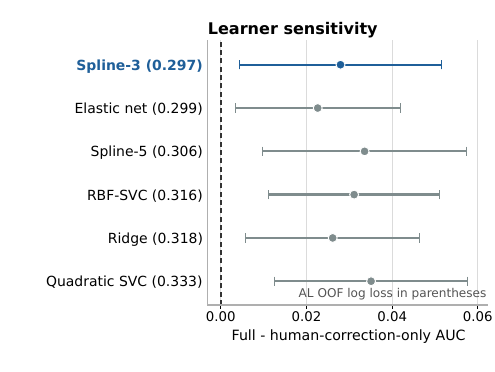}
\caption{Paired Full-minus-human-correction AUC contrasts across the six ACS surrogate learners, ordered by AL out-of-fold log loss. Bars are learner-specific 95\% intervals; all remain above zero. The intervals are learner-specific and do not provide an inferential ranking of learner accuracy.}
\label{fig:acs_learner_sensitivity}
\end{figure}

\section{Further Discussion}

The main implication of our results is that evaluating a learned surrogate-derived model is fundamentally different from evaluating a fixed prediction rule. At a hard threshold, score-estimation error contributes through observations near the threshold, producing variance of order $(n_{\mathrm{AL}}h)^{-1}$ rather than $n_{\mathrm{AL}}^{-1}$. Thus, a large surrogate-labeled sample alone does not justify treating the score as known. Treating the score as known is justified when $(n_{\mathrm{AL}}h)^{-1}=o(n_{\mathrm{HL}}^{-1}+n_{\mathrm{T}}^{-1})$, so that the auxiliary-source contribution is negligible relative to the gold-source and target-sample contributions. Otherwise, it omits a first-order source of uncertainty.

The three-source expansion also has implications for data collection. Its variance components distinguish uncertainty arising from gold-outcome estimation, score estimation near the threshold, and target-population averaging. Subject to labeling costs and nuisance-estimation quality, their relative magnitudes can indicate whether additional gold labels, surrogate labels, or target covariates would be most useful. Increasing $n_{\mathrm{AL}}$ is most valuable when the variance from estimating the score near the threshold is the largest component, but its benefit depends on the bandwidth and on the accuracy of the fitted score. More auxiliary observations cannot by themselves remove localization bias caused by an unsuitable bandwidth.

Our estimand defines the scope of the framework. We evaluate the population surrogate-derived model $m_S^\star(x)=\mathbb E_{\mathrm{AL}}(S\mid X=x)$ under the target covariate law, rather than the conditional performance of one particular finite-sample fitted model. If the inferential target is the performance of a specific deployed fit, that fitted model would instead be treated as fixed, leading to a different estimand and inferential analysis. Moreover, separate density ratios allow the two labeled sources to have different covariate distributions, but they do not protect against violations of gold-outcome transportability or inadequate source support for the target covariate distribution.

The numerical results should be interpreted within these limits. The plug-in comparison shows that the two numerator corrections jointly reduce empirical discrepancy in one controlled configuration; it neither isolates their individual effects nor establishes uniform superiority. The broader simulations show that bandwidth and sample-size configurations can materially affect finite-sample bias and coverage, reinforcing the value of reporting bandwidth sensitivity. The semi-synthetic ACS-Income design permits controlled validation because the masked target gold outcomes are revealed after estimation; in a genuinely unlabeled target population, such direct validation would not be available. Threshold-specific inference applies to prespecified thresholds or false-positive rates, while AUC inference uses the separate scalar one-step estimator. Data-driven threshold selection and simultaneous ROC bands require additional uniform-in-threshold theory.

\paragraph{Practical Considerations and Extensions}
\label{disclimappendix}

Several practical considerations remain. First, the bandwidth controls how
closely the AL correction concentrates around the threshold and how many local
observations contribute. The reference family in Section~\ref{Theory}, together
with the reported values of $h$ and $n_{\mathrm{AL}}h$ and the sensitivity
analysis, makes this choice transparent. Second, orthogonality removes the
leading nuisance-estimation effects, while the rate conditions in
Assumption~\ref{ass:rate} control the remaining higher-order terms. Third, the
threshold-specific and fixed-FPR ROC results are pointwise; simultaneous ROC
bands require a uniform Gaussian approximation. Finally, repeated
cross-fitting adds computational cost because each nuisance function is fitted
within every training fold.

One immediate extension is simultaneous confidence bands for the entire ROC curve, which would go beyond the fixed-threshold TPR/FPR, fixed-$u$ ROC, and scalar AUC inference established here. Other directions include predictive performance metrics for survival outcomes or multiclass classification.

A related extension is sequential monitoring as the surrogate-labeled sample grows. Such a procedure would require confidence bands that are uniform over the data-acquisition path, rather than repeated application of the current pointwise intervals. Developing those guarantees, potentially with a multiplicity adjustment depending on the monitoring schedule, is left for future work.

\end{appendices}

\end{document}